\documentclass{article}

\usepackage[utf8]{inputenc} % allow utf-8 input
\usepackage[T1]{fontenc}    % use 8-bit T1 fonts
\usepackage{hyperref}       % hyperlinks
\usepackage{url}            % simple URL typesetting
\usepackage{booktabs}       % professional-quality tables
\usepackage{amsfonts}       % blackboard math symbols
\usepackage{nicefrac}       % compact symbols for 1/2, etc.
\usepackage{microtype}      % microtypography
\usepackage{lipsum}
\usepackage{fancyhdr}       % header
\usepackage{graphicx}       % graphics
\graphicspath{{media/}}     % organize your images and other figures under media/ folder
\usepackage{amsmath}
\usepackage{mathtools}
\usepackage{longtable}
\usepackage{makecell}
\usepackage{tabu}

\usepackage{PRIMEarxiv}

\usepackage{tikz}
\usetikzlibrary{tikzmark}

\usepackage{amsthm}

\newtheoremstyle{myassumption}% name
  {\topsep}% Space above
  {\topsep}% Space below
  {\normalfont}% Body font
  {}% Indent amount
  {\bfseries}% Theorem head font
  {}% Punctuation after theorem head
  {.5em}% Space after theorem head
  {\thmname{#1}\thmnumber{ #2}\thmnote{ (#3)}\normalfont}% Theorem head spec

\newtheoremstyle{myremark}% name
  {\topsep}% Space above
  {\topsep}% Space below
  {\normalfont}% Body font
  {}% Indent amount
  {\bfseries}% Theorem head font
  {}% Punctuation after theorem head
  {.5em}% Space after theorem head
  {\thmname{#1}\thmnumber{ #2}\thmnote{ (#3)}\normalfont}% Theorem head spec

\theoremstyle{myassumption}
\newtheorem{myassumption}{Assumption}
\newtheorem{subassumption}{Assumption}[myassumption]

\theoremstyle{myremark}
\newtheorem{myremark}{Remark}

\theoremstyle{plain}
\newtheorem{theorem}{Theorem}
\newtheorem{lemma}{Lemma}
\newtheorem{proposition}[theorem]{Proposition}

\newtheorem{corollary}{Corollary}
\newtheorem{definition}{Definition}

\newenvironment{thma}[1]{\par\noindent{\bf Theorem #1.\ }\em}{\em}
\newenvironment{lma}[1]{\par\noindent{\bf Lemma #1.\ }\em}{\em}

\newenvironment{cora}[1]{\par\noindent{\bf Corollary #1.\ }\em}{\em}

\usepackage{enumitem}
\setlist[itemize,1]{left=1em}

\newcommand{\indep}{\perp \!\!\! \perp}

\DeclareMathOperator*{\argmin}{arg\,min}

\title{Evaluation of Active Feature Acquisition Methods for Time-varying Feature Settings
}

\author{
\name \hspace{-5pt}  Henrik von Kleist$^{1,2,3}$ \email henrik.vonkleist@helmholtz-munich.de \\
       \name Alireza Zamanian$^{2,4}$ \email alireza.zamanian@iks.fraunhofer.de \\
      \name Ilya Shpitser$^{3}$ \email ishpits1@jhu.edu
      \\
        \name Narges Ahmidi$^{1,3,4}$ \email narges.ahmidi@helmholtz-munich.de
      \\
      \addr $^1$Institute of AI for Health, Helmholtz Munich - German Research Center for Environmental Health,
Neuherberg, Germany
\\
\addr $^2$TUM School of Computation, Information and Technology, Technical University of Munich, %80333
Garching, Germany
\\
\addr $^3$Department of Computer Science, 
Johns Hopkins University
Baltimore, Baltimore, MD, % MD,% 21218,
USA
\\
\addr $^4$Fraunhofer Institute for Cognitive Systems IKS, %80686
Munich, Germany
      }

\begin{document}

\maketitle

\begin{abstract}%
Machine learning methods often assume that input features are available at no cost. 
However, in domains like healthcare, where acquiring features could be expensive or harmful, it is necessary to balance a feature's acquisition cost against its predictive value. 
The task of training an AI agent to decide which features to acquire is called active feature acquisition (AFA). 
By deploying an AFA agent, we effectively alter the acquisition strategy and trigger a distribution shift. 
To safely deploy AFA agents under this distribution shift, we present the problem of active feature acquisition performance evaluation (AFAPE).
 We examine AFAPE under i) a no direct effect (NDE) assumption, stating that acquisitions do not affect the underlying feature values; and ii) a no unobserved confounding (NUC) assumption, stating that retrospective feature acquisition decisions were only based on observed features. 
We show that one can apply missing data methods under the NDE assumption and 
offline reinforcement learning under the NUC assumption. When NUC and NDE hold, we propose a novel semi-offline reinforcement learning framework.
This framework requires a weaker positivity assumption and introduces three new estimators: A direct method (DM), an inverse probability weighting (IPW), and a double reinforcement learning (DRL) estimator.
\end{abstract}


\section{Introduction}
\label{sec_introduction}

Machine learning methods typically assume that the full set of input features will be readily available after deployment, with little to no cost. This is, however, not always the case, as acquiring features may impose a significant cost. In such situations, the predictive value of a feature should be balanced against its acquisition cost. In the medical diagnostics context, the cost of feature acquisition (e.g., for a biopsy test)  may include not only monetary cost but also the potential adverse harm to patients. This is why physicians acquire certain features, e.g., via biopsies, MRI scans, or lab tests, only when their diagnostic values outweigh their costs or risks. The challenge is exacerbated when prediction must be made regarding a large number of diverse outcomes with different sets of informative features. Going back to the medical example, a typical emergency department is able to diagnose thousands of different diseases based on a large set of possible observations. For every new emergency patient with ambiguous symptoms, clinicians must narrow down their search for a proper diagnosis via step-by-step feature acquisitions.

Active feature acquisition (AFA) addresses this problem by designing two AI systems: i) a so-called \textit{AFA agent}, deciding which features must be observed while balancing information gain vs. feature acquisition cost; ii) an ML prediction model, often a classifier, that solves the prediction task based on the acquired set of features. To elucidate the AFA process, we present a hypothetical and simplified scenario of diagnosing heart attacks.

\begin{figure}[ht]
%\begin{wrapfigure}{r}{0.50 \textwidth}
\centering
%\vspace{-15pt}
\includegraphics[width= 0.7\textwidth]{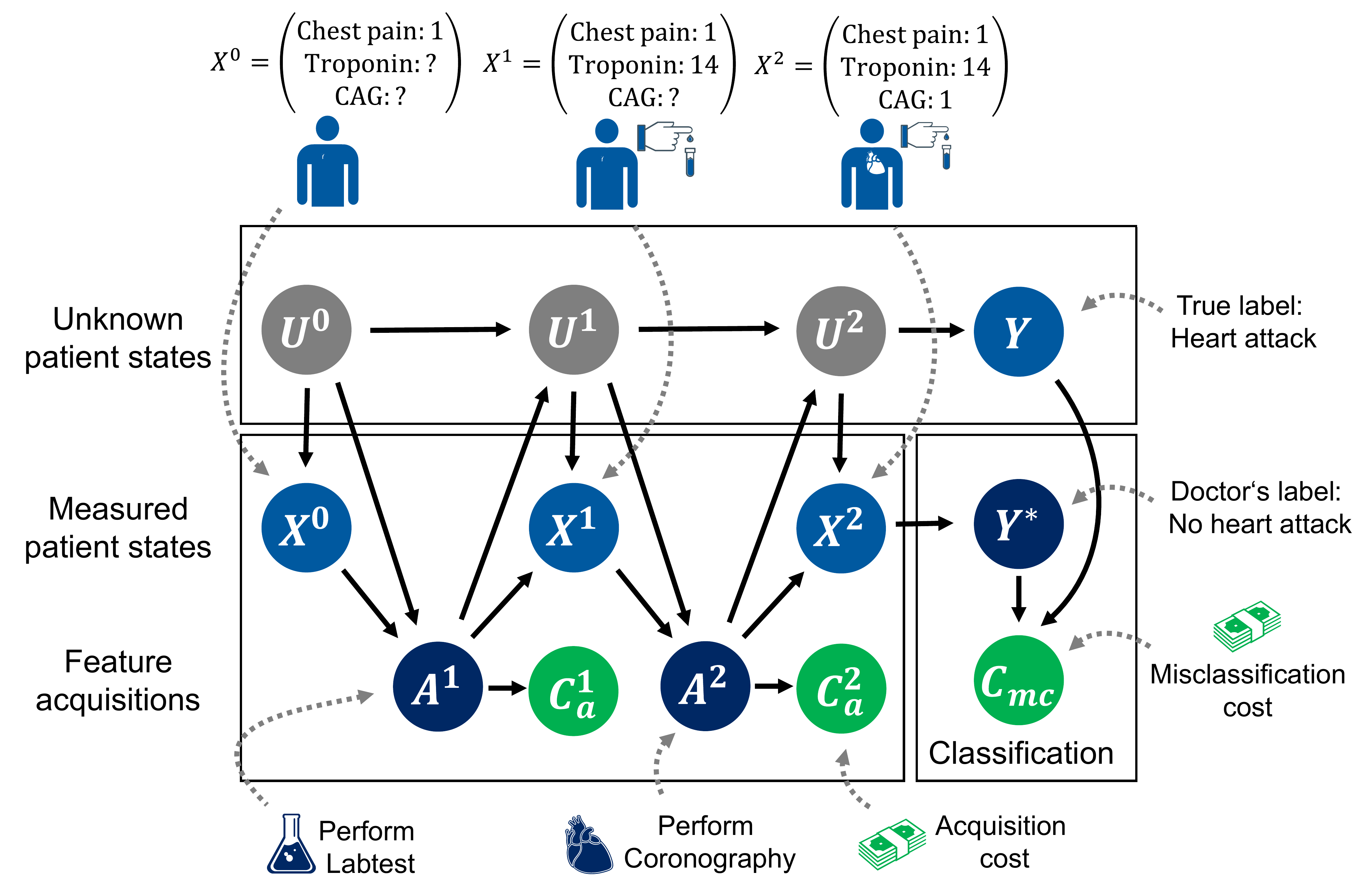}
\caption{ AFA process for a simplified hypothetical heart attack diagnosis example. A patient with chest pain ($X^0$) prompts the doctor to first order a troponin lab test ($A^1$) and, upon reviewing the result ($X^1$), to also order a coronography (CAG) ($A^2$). The feature acquisitions $A^1$ and $A^2$ produce feature acquisition costs $C_a^1$ and  $C_a^2$. After the acquisition process concludes, the doctor makes a diagnosis $Y^*$, which, if different from the true underlying condition $Y$, produces a misclassification cost $C_{mc}$.
}
\vspace{-10pt}
\label{graph_AFA_example}
\end{figure}
%\end{wrapfigure}

% \paragraph{Heart attack diagnosis example} 
\subsection{Heart Attack Diagnosis Example} 
Figure \ref{graph_AFA_example} presents the partially observable decision process that encapsulates the sequential decision-making aspect of the AFA problem for a heart attack diagnosis example. 
Upon arrival at the hospital, a patient with an unknown health state ($U^0$) exhibits the symptom of chest pain ($X^0 = $ "chest pain"). At this stage, no additional information is available. The attending doctor decides to order a troponin lab test ($A^1$ = "acquire troponin") as part of the feature acquisition process. The laboratory test incurs a feature acquisition cost ($C_a^1 = "\$100"$).
Subsequently, upon reviewing the results of the lab test ($X^2$), the doctor decides that a coronography ($A^2$), an invasive imaging procedure, is necessary. Notably, the feature acquisition cost ($C_a^2$) for this procedure may be substantially higher due to the potential harm to the patient. After the completion of the feature acquisition process, a diagnosis of whether the patient is experiencing a heart attack is performed. A (hypothetical) misclassification cost $C_{mc}$ arises if the diagnosis $Y^*$ and the true condition $Y$ differ.

In general, medical tests may also impact the patient's health (illustrated by the edges $A^t \rightarrow 
U^t$). The invasive coronography from our example may, for example, cause bleeding, infections, hypotension or other problems \cite{tavakol_risks_2012}.   
Such an effect is denoted as a \textit{direct effect}, and we refer to its absence as the \textit{no direct effect (NDE)} assumption. 

Furthermore, the decision to perform a clinical test $A^t$ may not solely rely on past observed variables $X^{\tau}$ ($\tau < t$) but can also depend on past unobserved variables $U^{\tau}$ (illustrated by the edges $U^{t-1} \rightarrow %\text{ } 
A^t$) or even other factors. For example, a doctor might base acquisition decisions on information that is partially not recorded in the medical database. We refer to the assumption that acquisitions are only determined by past observed variables as the \textit{no unobserved confounding (NUC)} assumption.

\subsection{Paper Goal} 
We investigate the evaluation of AFA agents under the distribution shift that occurs since the AFA agent makes different acquisition decisions than the doctors who were responsible for collecting the retrospective data set. 
The focus of the paper is thus not to design new AFA agents and classifiers but to estimate the performance of \textit{any} AFA agent and classifier at deployment. 
This means the doctor should be informed, for example, how many wrong diagnoses are to be expected or how much acquisition cost will be incurred on average if an AFA system is deployed.

The paper has two primary objectives:
i) Identification, which involves determining the assumptions that enable the unbiased
representation of costs from the retrospective data distribution and
ii) Estimation, which focuses on
turning the obtained identification strategy into point and interval estimators following
well-grounded statistical principles.
We specifically analyze scenarios that involve both adherence to and violation of the NDE and NUC assumptions.
We formulate this problem of \textit{active feature acquisition performance evaluation (AFAPE)} as the problem of estimating the expected counterfactual acquisition and misclassification costs using retrospective data.

\subsection{Paper Outline and Contributions} 
The remainder of this paper is organized as follows. 
After reviewing the necessary background and related methods in Section \ref{sec_related_methods}, we formulate the AFAPE problem in Section \ref{sec_afape}. The general AFAPE problem is not identified—that is, it is not possible to estimate the counterfactual acquisition and misclassification costs from retrospective data when both the NDE and NUC assumptions are violated.

Therefore, we begin Section \ref{sec_RL_view} by employing the NUC  assumption and show that this makes AFAPE amenable to an offline reinforcement learning (RL) / dynamic treatment regimes (DTR) view. This allows the application of known identification and estimation theory from the offline RL / DTR literature. 

In Section \ref{sec_missing_data}, we make instead the NDE assumption and assume the NUC assumption can be violated. We demonstrate that under the NDE assumption, the AFA decision process depicted in Figure \ref{graph_AFA_example} transforms into a missing data graph (m-graph) \cite{mohan_graphical_2013, shpitser_missing_2015}, a recognized graphical framework in the missing data literature. 
This enables us to apply established identification and estimation theory from the missing data literature. 
%Estimators include the missing data version of IPW \cite{seaman_review_2013} and the missing data version of the plug-in of the G-formula. One commonly used special version of the latter is multiple imputation (MI) \cite{sterne_multiple_2009}. 
After solving the missing data problem, the AFAPE problem is transformed into an online RL setting where one can simulate different acquisition trajectories, leading to a trivial solution for AFAPE. 

In Section \ref{sec_semi_offline_view}, we assume both the NUC and the NDE assumptions hold. 
In this setting, one can apply either offline RL or missing data methods to solve AFAPE, but both require strong positivity assumptions and do not utilize the data optimally. Therefore, we propose a new viewpoint on AFA, which we denote as \textit{semi-offline reinforcement learning}.
Under the semi-offline RL viewpoint, the AFA agent engages with the environment in an online manner, but certain actions (where the underlying feature values are missing in the retrospective data) cannot be explored. The positivity assumption required for identification is drastically reduced under the new semi-offline RL viewpoint. 
%We discuss identification and introduce three novel estimators based on this view.
We derive three novel estimators that can be denoted as semi-offline RL versions of known offline RL estimators, including the Q-function based direct method (DM) \cite{levine_offline_2020}, inverse probability weighting (IPW) \cite{levine_offline_2020}, and the double reinforcement learning (DRL) estimator \cite{kallus_double_2020}.
Notably, our DRL estimator is doubly robust, exhibiting consistency even if either the underlying Q-function or the propensity score model is misspecified. 

In Section \ref{sec_semiparametrics}, we explore the estimation of AFAPE under the NUC and NDE assumptions using semiparametric theory. We demonstrate how all three viewpoints—the offline RL view (under the NDE assumption), the missing data view (under the NUC assumption), and the semi-offline RL view—interconnect within this theoretical framework. 
Unfortunately, there is no closed-form solution for an efficient estimator. However, we can enhance the efficiency of all estimators by applying established semiparametric techniques from related fields, such as standard missing data problems \cite{tsiatis_semiparametric_2006} and dynamic testing and treatment regimes \cite{liu_efficient_2021}. These methods, though, come with significant computational costs, are challenging to implement, and require strong positivity assumptions.

In Section \ref{sec_experiments}, we present synthetic data experiments that exemplify the improved data efficiency and reduced positivity requirements of the semi-offline RL estimators. Our experiments also show that biased evaluation methods commonly used in the AFA literature can lead to detrimental conclusions regarding the performance of AFA agents. Deploying such methods without caution may pose significant risks to patients' lives.
We end the paper with a Discussion (Section \ref{sec_discussion}) and Conclusion (Section \ref{sec_conclusion}).

\section{Background and Related Methods}
\label{sec_related_methods}

%\begin{itemize}
%    \item \textcolor{blue}{Introduce AFA, basics about missing data, offline RL, and semi-parametric theory and give background in the appendices where appropriate}    
%    \item \textcolor{blue}{mention the NDE assumption of Robins} 
%    \item \textcolor{blue}{Mention research on distribution shift robust ML} 
 %   \item \{AFA, NDE and distribution shift robust ML paragraphs are more related methods, while missing data, semi-parametric theory and offline RL correspond more to background. Maybe one should split this somehow.}
%\{Maybe this is more outlook}
%\end{itemize}

In the following, we review some of the literature about AFA and provide some background on offline RL/ DTR, missing data, and semi-parametric theory. 

\subsection{Active Feature Acquisition (AFA)} 

Research on active feature acquisition (AFA) and related problem formulations have been published under various different names and in different, largely disjoint, research communities. Early research in economics and decision science literature 
addressed the problem of "Value of Information" (VoI) \cite{lavalle_cash_1968,lavalle_cash_1968-1,gould_risk_1974,hilton_determinants_1979,hess_risk_1982,keisler_value_2014}. Similar methods have also been applied in the medical field, often in terms of cost-effectiveness analysis of screening policies \cite{mushlin_is_1992,krahn_screening_1994,botteman_health_2003,force_screening_2009}.
AFA has further been studied under the name of "dynamic testing regimes" \cite{liu_efficient_2021, robins_estimation_2008} or "dynamic monitoring regimes" \cite{neugebauer_identification_2017,kreif_exploiting_2021} in the causal inference literature, often in combination with dynamic treatment regimes. In these settings, the goal of the feature acquisitions is not to enable better predictions/diagnoses but to enable better treatment decisions.  

The name "active feature acquisition" (AFA) \cite{an_active_2006,li_active_2021, li_towards_2021, chang_dynamic_2019, shim_joint_2018, yin_reinforcement_2020} is common in the machine learning literature, but other names are also frequently used. These include, but are not limited to, "active sensing" \cite{yoon_asac_2019, yoon_deep_2018, tang_adversarial_2020, jarrett_inverse_2020}, "active feature elicitation" \cite{natarajan_whom_2018,das_clustering_2021}, "dynamic feature acquisition" \cite{li_dynamic_2021}, "dynamic active feature selection" \cite{zhang_novel_2019}, "element-wise efficient information acquisition" \cite{gong_icebreaker_2019}, "classification with costly features" \cite{janisch_classification_2020} and "test-cost sensitive classification" \cite{xiaoyong_chai_test-cost_2004}.

These diverse research fields share a common characteristic, which involves designing an agent to selectively acquire a subset of features to balance acquisition cost and information gain. The approaches used to design such agents range from simple greedy acquisition strategies to more complex RL-based strategies. However, the focus of this work is not on any specific AFA method but rather on evaluating the performance of \textit{any} AFA method under the acquisition distribution shift. For a more comprehensive literature review of existing AFA methods and a distinction between AFA and other related fields, we direct interested readers to Appendix \ref{app_AFA_methods}.

\subsection{(Offline) Reinforcement Learning (RL) / Dynamic Treatment Regimes (DTR)}
We show in Section \ref{sec_RL_view} %and %\ref{sec_missing_data} 
that AFA can be analyzed from an offline RL/ DTR viewpoint. In Section \ref{sec_missing_data}, we show that AFAPE can also be analyzed from an online RL viewpoint (if NDE holds and after missingness has been resolved). Online RL allows the interaction of an agent with the environment and thus the simulation of outcomes under any desired policy, thereby leading to a trivial solution for the AFAPE problem. 
In offline RL, however, such a simulation is not possible due to missing knowledge about the environment.  The AFAPE problem then becomes equivalent to the problem of off-policy policy evaluation (OPE) \cite{dudik_doubly_2011, thomas_data-efficient_2016, kallus_double_2020}, in which the goal is to evaluate the performance of a "target" policy (here the AFA policy) from data collected under a different "behavior" policy (here the retrospective acquisition policy of, for example, a doctor).
Several estimators have been developed for the OPE problem. These include the plug-in based on the G-formula \cite{robins_new_1986} (also referred to as model-based evaluation \cite{levine_offline_2020}), inverse probability weighting (IPW) \cite{levine_offline_2020} (also known as importance sampling or the Horvitz-Thompson estimator \cite{horvitz_generalization_1952}), the direct method (DM) \cite{levine_offline_2020}, and double reinforcement learning (DRL) \cite{kallus_double_2020}.

\subsection{Missing Data}
In this paper, we show that AFAPE can be viewed as a missing data problem (+ a trivial online RL problem). Thus, known identification and estimation techniques from the missing data literature can be employed. We show that the NUC assumption described in this paper corresponds under NDE to a \textit{missing-at-random (MAR)} assumption. Violations of the NUC assumption correspond, in our setting, to a special, identified \textit{missing-not-at-random (MNAR)} scenario. 
Estimation strategies generally include inverse probability weighting (IPW)  \cite{seaman_review_2013}, and multiple imputation (MI) \cite{sterne_multiple_2009} (a special case of the plug-in of the G-formula).

\subsection{Semiparametric Theory}

The goal of AFAPE is to estimate the expected acquisition and misclassification costs that would arise when following the AFA system's decisions. In more general terms, this corresponds to estimating a target parameter $J = J(p)$ of some unknown distribution $p$ given a set of observed samples from $p$ (the retrospective data set). 
The goal in semi-parametric theory is to find suitable estimators for such a target parameter $J$ while leaving at least part of the data-generating process $p$ unrestricted/ unspecified, thereby imposing fewer assumptions, which can lead to more credible estimates. Assumptions that can be taken with a reasonable level of confidence (such as in many AFA settings the NUC or NDE assumptions), can, however, be leveraged to derive more efficient estimators. 

A key focus of semiparametric theory is the identification of \textit{influence functions}, which are used to construct estimators with desirable properties, such as
a consistently estimable asymptotic variance. The estimator associated with the influence function that has the smallest asymptotic variance is the most efficient one. Often, these influence functions include nuisance functions—unknown components of the model that must be estimated from the data. For example, a nuisance function might model the probability of a doctor acquiring a particular feature. Estimating these nuisance functions typically involves parametric assumptions (e.g., using a logistic regression model). Consequently, the resulting estimator is only  \textit{locally efficient}, meaning it achieves efficiency only if the parametric assumptions for the nuisance function hold true.

Even if the parametric assumptions are incorrect, many influence function-based estimators remain consistent due to a property known as \textit{multiple robustness}
\cite{rotnitzky_multiply_2017}. 
For instance, the DRL estimators in this paper exhibit a form of double robustness \cite{scharfstein_adjusting_1999,chernozhukov_doubledebiased_2018}. 
These estimators rely on learning two nuisance functions and remain consistent as long as the parametric assumption holds for at least one of these functions.

For a more detailed review, please see Appendix \ref{app_semiparametric_theory}, which covers both the general principles of semiparametric theory and specific insights related to missing data problems. A thorough understanding of semiparametric theory is primarily necessary for Section \ref{sec_semiparametrics}, which is intended for interested readers.

%\vspace{10pt}
%\noindent 
\subsection{Active Feature Acquisition Performance Evaluation (AFAPE)}
Although we believe to be the first to explicitly formulate and analyze the AFAPE problem, other AFA papers have reported performance metrics that can be seen as attempts to address it. The reported results, however, often lack assumption statements and justification for the chosen evaluation framework and are, in general, biased or inefficiently estimated. We categorize these results based on the viewpoints analyzed in this paper:

\textit{Offline RL view: } The offline RL view has been utilized in the AFA context \cite{chang_dynamic_2019,cheng_optimal_2018}. As we show in this paper, this approach is only valid under the NUC and strong positivity assumptions.

\textit{Missing data + online RL view:}
We show in this paper that one can apply, under the NDE assumption, a missing data + online RL viewpoint to solve AFAPE. While this viewpoint has been taken in the AFA literature before, the missing data part of it has, to our knowledge, only been solved  
using  (conditional) mean imputation \cite{an_reinforcement_2022, erion_coai_2021, janisch_classification_2020}. (Conditional) mean imputation leads, however, to biased estimation results, as we illustrate in Section \ref{sec_missing_data}.

\textit{Semi-offline RL view:}
Some AFA papers \cite{janisch_classification_2020, yoon_deep_2018} have addressed the problem of missing data during the online RL simulations by simply blocking the corresponding feature acquisitions. This approach is similar to our proposed semi-offline RL view. However, unlike our approach, these papers did not correct for the distribution shift caused by blocking feature acquisitions, resulting in biased estimation results.

\subsection{No Direct Effect (NDE) Assumption}
The only work that, to the best of our knowledge, leverages the NDE assumption in a similar way to our semi-offline RL viewpoint is a series of publications from the causal inference literature around the slightly different problem of evaluation of joint dynamic testing and treatment regimes \cite{robins_estimation_2008,caniglia_emulating_2019,liu_efficient_2021,neugebauer_identification_2017,kreif_exploiting_2021}. In this setting, the agent is not only tasked with deciding which features to acquire, but also which treatments to give to the patient. The treatment assignment replaces the need for classification/diagnosis in these settings. Robins et al. \cite{robins_estimation_2008} introduced within this setting for the first time the term "no direct effect" (NDE) assumption. NDE stated
that the feature acquisition decisions have no direct effect (or no long-term direct effect \cite{liu_efficient_2021}) on the health status of the patient, except through their effect on the treatment decisions.  

Caniglia et al. \cite{caniglia_emulating_2019} derived an IPW estimator for this context, which demonstrated a 50-fold increase in data efficiency compared to the offline RL IPW estimator, signaling the enormous benefits that can be achieved by leveraging the NDE assumption. 
We adapt this estimator to the AFA setting and show that it is equivalent to our proposed IPW estimator for a simple setting and a special positivity assumption. 
However, our IPW estimator can be applied in more general settings under weaker positivity assumptions and combined with our DM method to form the novel DRL estimator for semi-offline RL. 

Liu et al. \cite{liu_efficient_2021} developed nearly semiparametrically efficient estimators for this problem by modifying the DRL estimator from the offline RL perspective to enhance efficiency under the NDE assumption. In Section \ref{sec_semiparametrics}, we demonstrate that this approach can be adapted for AFA, where it becomes a specialized method within established semiparametric techniques for missing data problems. Our proposed semi-offline RL estimators can also achieve greater data efficiency using this framework. However, this approach has limitations: it requires strong positivity assumptions, is only applicable in very simple settings with one acquisition action per time point, and the resulting augmentation necessitates complex approximations of function spaces, making implementation challenging and resulting in estimators that are only "nearly" efficient. Additionally, it has only been tested in an extremely simplified context with one acquisition action and one time step \cite{liu_efficient_2021}.

\subsection{Distribution Shift Robust ML Models}
Lastly, this work also relates to the general literature on distribution shift-robust ML models. A common problem with the deployment of ML models occurs if the model is trained, for example, on data from hospital 1 but should be deployed to hospital 2.
The related literature aims at building robust models that retain their performances across deployment environments \cite{rockenschaub_generalisability_2023,rockenschaub_single-hospital_2023}.  One part of the distribution that might change between hospital 1 and hospital 2 is the feature acquisition policy. If this is the case, and if the acquisition policy at hospital 2 is known, one may directly apply our methods to this scenario and treat the acquisition policy at hospital 2 as the AFA policy that is to be evaluated. However, we will not go into more detail about this scenario and will focus on the AFA setting.

\section{Active Feature Acquisition Performance Evaluation (AFAPE) Problem Definition}
\label{sec_afape}

We begin the section by introducing the mathematical notation for the AFA setting and AFAPE problem. A glossary containing all the variables and important terms can be found in Appendix \ref{Appendix_glossary}.

\begin{figure}[ht]
\centering%\includegraphics[width=\textwidth]
    \includegraphics[width=0.8\textwidth]
    {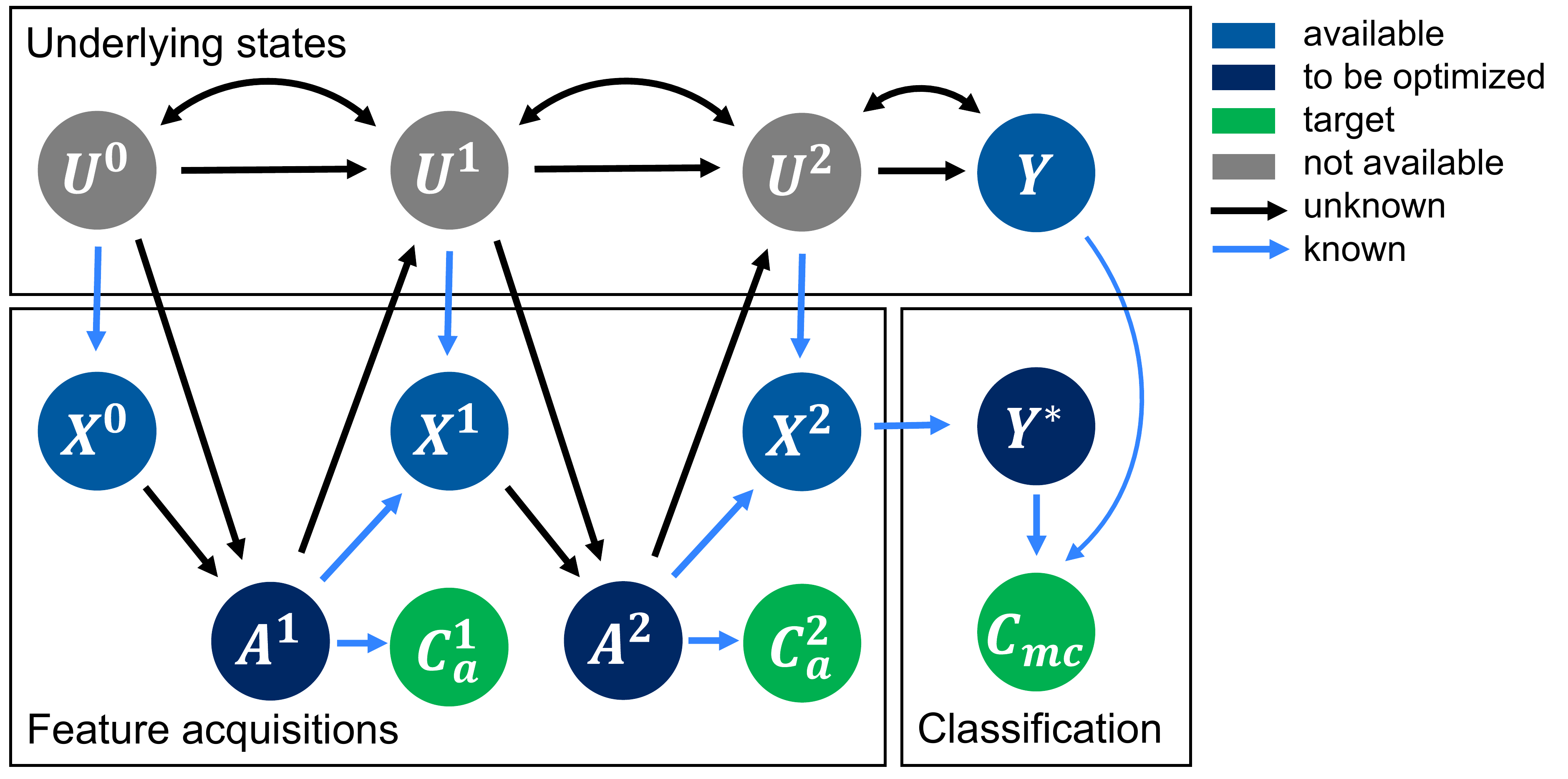}
    \caption{ 
    The causal graph depicting the AFA setting as a partially observable decision process
    consisting of unobserved underlying features $U^t$, feature acquisition actions $A^t$, feature measurements $X^t = G_{A^t}(U^t)$, and associated acquisition costs $C_a^t$. 
     After a number of acquisition steps $T$ (here $T=2$), a classification $Y^*$ is to be performed. In the case of misclassification ($Y^*$ is not equal to the true label $Y$), a misclassification cost $C_{mc}$ is produced.
    Edges showing long-term dependencies are omitted from the graph for visual clarity. These include: 
    $\underline{U}^{t-1}, \underline{X}^{t-1}, \underline{A}^{t-1} \rightarrow  A^{t}$; $\underline{X}^{T}, \underline{A}^{T} \rightarrow  Y^*$; $A^{t} \rightarrow \overline{U}^t$; $\underline{U}^{t-1} \leftrightarrow U^{t}$; $\underline{U}^{t-1} \rightarrow U^{t}$; $\underline{U}^{T} \leftrightarrow Y$
    and $\underline{U}^{T} \rightarrow Y$ (where $\leftrightarrow$ denotes unobserved confounding). }
    \label{graph_general}
\end{figure}

\subsection{Feature Acquisition Process}

The feature acquisition process, as illustrated in Figure \ref{graph_general}, is modeled using the following variables:

\begin{itemize}
    \item \textbf{Unobserved underlying features}: \(U^t \in \mathbb{R}^{d}\) for \(t \in \{0, \ldots, T\}\).
    The features are dynamic, meaning they change over time, so generally, $U^t_i \neq U^{t+1}_i$.
    \item \textbf{Feature acquisition decisions}: \(A^t \in \mathcal{A}^t = \{0,1\}^{d}\) for \(t \in \{1, \ldots, T\}\), where \(A_i^t = 1\) indicates that feature \(U_i^t\) will be acquired.
    \item \textbf{Observed feature values}: \(X^t \in (\mathbb{R} \cup \{ "?" \})^{d}\) for \(t \in \{0, \ldots, T\}\), where \( "?" \) denotes a missing feature value that was not acquired. We assume no measurement error, so \(X^t\) is deterministically determined by \(A^t\) and \(U^t\) according to:
    \begin{align}
    X_i^t = G_{A_i^t}(U_i^t) = 
    \begin{cases}
    U_i^t & \text{if } A_i^t = 1,\\
    "?" & \text{if } A_i^t = 0.
    \end{cases}
    \label{eq:measurement_process}
    \end{align}
    $G_A$ denotes the observation function for a given feature acquisition decision $A$. Further, we assume for simplicity $X^0 = U^0$.
    \item \textbf{Acquisition costs}: \(C_{a}^t \in \mathbb{R}\) represents the known feature acquisition cost associated with \(A^t\).
\end{itemize}

\noindent 
Additionally, let \(\underline{U}^t\) and \(\overline{U}^t\) denote the complete past and complete future of \(U^t\), respectively, where \(\underline{U}^t = \{U^0, \ldots, U^t\}\) and \(\overline{U}^t = \{U^t, \ldots, U^T\}\). Similarly, define \(\underline{A}^t\), \(\overline{A}^t\), \(\underline{X}^t = G_{\underline{A}^t}(\underline{U}^t)\), and \(\overline{X}^t= G_{\overline{A}^t}(\overline{U}^t)\) for the variables \(A^t\) and \(X^t\). Let \(U = \overline{U}^0\), \(A = \overline{A}^1\), and \(X = \overline{X}^0\), and denote the space of \(A\) by \(\mathcal{A}\).
The retrospective acquisition policy is given by \(\pi_\beta^t(A^t \mid \underline{X}^{t-1}, \underline{U}^{t-1}, \underline{A}^{t-1})\).

\subsection{Classification Process}

At time \(T\), the feature acquisition process concludes, and the classification of an underlying label is performed based on the acquired information. The classification process includes:

\begin{itemize}
    \item \textbf{Underlying categorical label}: \(Y\), the true label to be classified.
    \item \textbf{Predicted label}: \(Y^*\), obtained using a deterministic classifier $Y^* = f_\textit{cl}(\underline{A}^T , \underline{X}^T)$.
    %We denote its deterministic density as \(g_\textit{cl}(Y^* \mid \underline{X}^T, \underline{A}^T)\).
    \item \textbf{Misclassification cost}: \(C_{mc} \in \mathbb{R}\), the cost incurred when \(Y^*\) differs from \(Y\), defined by a predefined cost function \(C_{mc} = f_C(Y^*, Y)\). We alternatively write $C_{mc} = f_C(f_\textit{cl}( \underline{A}^T , \underline{X}^T), Y) \equiv f_C(\underline{A}^T, \underline{X}^T, Y)$ directly as a function of $\underline{A}^T, \underline{X}^T$ and $Y$. 
\end{itemize}
\noindent 
We let \(g\) represent known deterministic distributions or densities, distinct from \(p\) used for other distributions or densities. For example, we let \(g_\textit{cl}(Y^* \mid \underline{A}^T, \underline{X}^T)\) denote the deterministic distribution of the classifier $f_\textit{cl}$ and 
$g_\textit{C}(C_{mc}| Y^*, Y)$ denote the deterministic distribution of the cost function $f_\textit{C}$.  

We also assume \(Y\) is always available in the retrospective data set and allow for unobserved confounding among the unobserved underlying features and the label (represented by edges \(\underline{U}^t \leftrightarrow U^{t+1}, Y\)), but no additional confounding with \(A^t\). We denote the retrospective data set consisting of the variables $A$, $X$, and $Y$ as $\mathcal{D}$.

\subsection{Problem Definition: Active Feature Acquisition Performance Evaluation (AFAPE)}

Given a target AFA policy $\pi_{\alpha}^t(A^t|\underline{X}^{t-1},\underline{A}^{t-1})$ (which is not allowed to depend on the unobserved underlying features $\underline{U}^{t-1}$) and a target classifier $f_\textit{cl}(\underline{X}^T, \underline{A}^T)$, the goal of AFAPE is to estimate the expected acquisition and misclassification costs that would arise, had the target policy $\pi_{\alpha}$ and classifier $f_\textit{cl}$ been deployed. The estimation problem for this expected counterfactual cost can be expressed as estimating
\begin{equation}
\label{eq:AFAPE_objective}
J_a = \mathbb{E}\left[\sum_{t=1}^T C_{a, (\pi_\alpha)}^t  \right], \text{ and } J_{mc} = \mathbb{E}\left[ C_{mc, (\pi_\alpha)}  \right],
\end{equation}
\noindent
where $C_{a, (\pi_\alpha)}^t$ and $C_{mc, (\pi_\alpha)}$ denote the potential outcomes of the acquisition and misclassification costs under the AFA policy $\pi_\alpha$. 
Therefore, $J_a$ and $J_{mc}$ represent the expected acquisition and misclassification costs under a distribution induced by $\pi_{\alpha}$ rather than by the retrospective acquisition policy $\pi_\beta$.
Note that this assumes 
$\pi_\alpha$
can be followed perfectly, which may not always hold in practice, for example, if patients refuse certain medical tests or miss appointments.

The goal of this paper is to i) perform identification, i.e., to determine under which assumptions it is possible to resolve this distribution shift and to obtain an unbiased estimate of $J_a$ and $J_{mc}$; and ii) to derive such unbiased estimators.

As the AFAPE problem is similar for $J_a$ and $J_{mc}$, we will focus on $J_{mc}$ throughout the main part of the paper. We abbreviate $J_{mc} \equiv J$ and $C_{mc} \equiv C$. We provide the estimation formulas for $J_a$ and for $J_{mc}$ when a prediction is to be performed at each time step in the relevant appendices.

\subsection{Problem Definition: Optimization of Active Feature Acquisition Methods}
\label{sec_afa_optimization_problem}

While the focus of the paper is on the AFAPE problem, we provide the definition of the AFA optimization problem for completeness. The goal in AFA is to find the optimal AFA policy $\pi_{\alpha}^t(A^t|\underline{X}^{t-1},\underline{A}^{t-1};\phi_1^*)$ parameterized by $\phi_1^*$, and the optimal classifier $f_\textit{cl}(\underline{X}^T, \underline{A}^T; \phi_2^*)$ parameterized by $\phi_2^*$, such that their joint application minimizes the expected sum of counterfactual acquisition and misclassification costs: 
\begin{align*}
   \phi_1^*, \phi_2^* = \argmin_{\phi_1,\phi_2}  J_\text{total}(\phi_1, \phi_2)= \argmin_{\phi_1,\phi_2}  \mathbb{E}\left[
   \sum_{t=1}^T C_{a,(\pi_\alpha)}^t + C_{mc,(\pi_\alpha)} \Big\vert \phi_1, \phi_2  \right].
\end{align*}

\subsection{Assumptions}

Here, we provide an overview of the key assumptions in this paper. 
We start by stating the fixed assumptions that hold throughout the paper before stating assumptions that we vary within different sections.

\subsubsection{Fixed Assumptions}

\noindent
We make the following assumptions throughout the paper:

\begin{myassumption}[No measurement noise]
There is no noise in feature measurements, as expressed by Eq. \ref{eq:measurement_process}.
\label{assump:measurement_noise}
\end{myassumption}

\begin{myassumption}[Consistency]
If $\underline{A}^t = \underline{a}^t$, then $U^t_{(\underline{a}^t)} = U^t$.
\label{assump:consistency}
\end{myassumption}

\noindent
This standard consistency assumption from the causal inference literature states
that an individual's observed outcomes align with their potential outcomes under the observed acquisition decisions. Here, $U^t_{(\underline{A}^t = \underline{a}^t)}$ represents the potential outcome of $U^t$ under potential acquisition decisions $\underline{A}^t = \underline{a}^t$.

\begin{myassumption}[No interference]
The acquisition decisions for one individual do not affect other individuals. 
\label{assump:interference}
\end{myassumption}  

\noindent
One prominent example of interference in medical settings is allocation interference, which can occur when a hospital's resources or staff are overwhelmed by a high volume of medical test requests for multiple patients simultaneously, resulting in the inability to fulfill all feature acquisition requests.

\subsubsection{Investigated Assumptions}
In this paper, we analyze how the following assumptions affect identification and estimation of the target $J$ in the AFAPE problem. 

\begin{myassumption}[No direct effect (NDE)]
The unobserved underlying features are not influenced by feature acquisitions (i.e., $A^t \not \rightarrow \overline{U}^t$). 
Equivalently:  $U^t \indep \underline{A}^t \;|\; \underline{U}^{t-1}$.
\label{assump:nde}
\end{myassumption}

\noindent
This is a standard assumption in missing data problems, but it may not hold in all medical settings, as some medical tests can alter certain features of the patient. 
The NDE assumption is relaxed in Section \ref{sec_RL_view}, and made in Sections \ref{sec_missing_data}, 
\ref{sec_semi_offline_view} and \ref{sec_semiparametrics}.

\begin{myassumption}[No unobserved confounding (NUC)]
Acquisition decisions are independent of the unobserved underlying features given past acquisition decisions and measured features: $A^t \indep \underline{U}^{t-1} \;|\; \underline{X}^{t-1}, \underline{A}^{t-1}$. This is graphically expressed by the missing arrow $\underline{U}^{t-1} \not \rightarrow A^{t}$.
\label{assump:nuc}
\end{myassumption}

\noindent 
The no unobserved confounding assumption may, for example, be violated in medical settings if certain feature values are seen by the physician and influence their decision-making for further tests but are not recorded in the database. 
We assume NUC in Sections \ref{sec_RL_view}, \ref{sec_semi_offline_view} and \ref{sec_semiparametrics}
and allow certain violations in Section \ref{sec_missing_data}. Note that when referring to NUC, we only assume no unobserved confounding of the acquisition actions. Potential unobserved confounding within $U$ and between $U$ and $Y$ is allowed throughout the paper.

\begin{myassumption}[Positivity/ experimental treatment assignment/ overlap]
Certain feature sets have a positive probability of being acquired under the retrospective acquisition policy  $\pi_\beta$. 
\label{assump:positivity}
\end{myassumption}

\noindent 
The positivity assumption (also known as experimental treatment assignment assumption or overlap assumption) relates to how much exploration was done under the retrospective acquisition policy  $\pi_\beta$. Positivity requirements are crucial for identification and vary between the discussed views. Hence, we derive and discuss them separately for each view.

\section{Offline Reinforcement Learning View}
\label{sec_RL_view}

\textit{Assumptions in this section: Assumption \ref{assump:nuc} (NUC) }%, \{measuring equals recording}}

\noindent 
Firstly, we consider the scenario where the NUC assumption (Assumption \ref{assump:nuc}) holds (i.e. $\underline{U}^{t-1} \not \rightarrow A^t$), but the NDE assumption (Assumption \ref{assump:nde}) does not necessarily hold (i.e. $A^{t} \rightarrow \overline{U}^t$). This scenario can be addressed using the offline reinforcement learning (RL) view. The NUC assumption allows us to perform a latent projection \cite{verma_equivalence_1990} to project out the unknown variables $U^t$ (along with $Y$ and $Y^*$) from the causal graph in Figure \ref{graph_general} and obtain the graph in Figure \ref{graph_offline_RL} which contains only observed variables. 
The projected graph allows us to apply established identification and estimation methods from the offline RL literature.

\begin{figure}[ht]
%\begin{wrapfigure}{R}{0.50 \textwidth}
\centering
%\vspace{-15 pt}
\includegraphics[width=0.60 \textwidth]{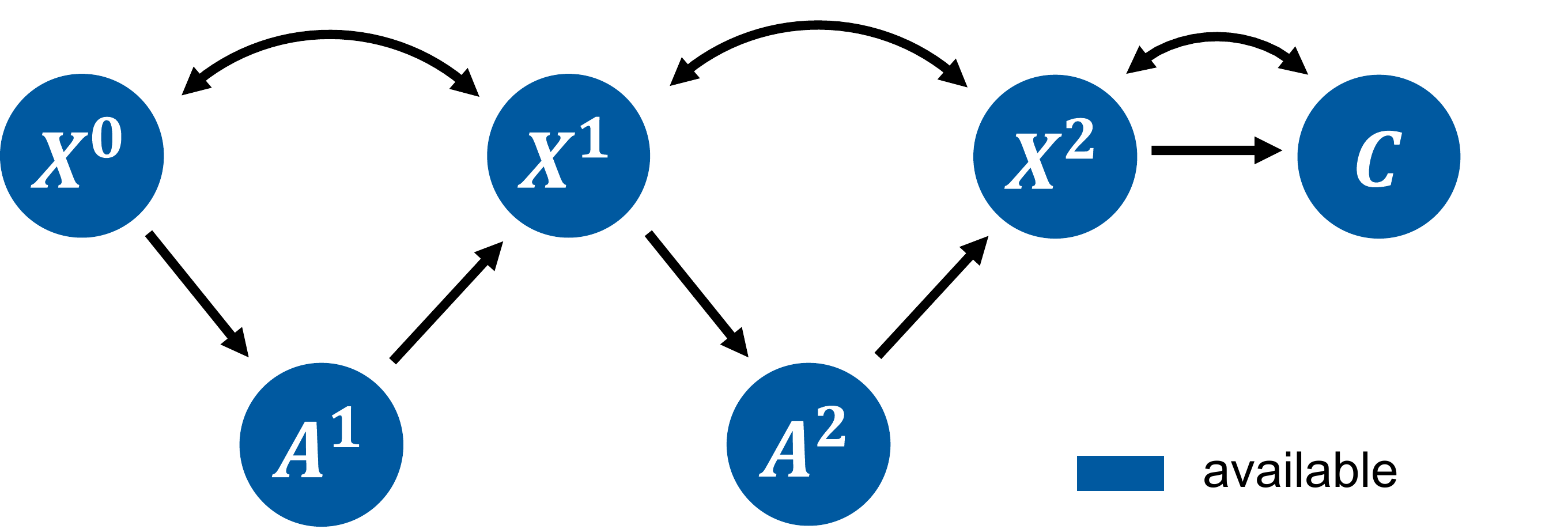}
\caption{ Updated causal graph of the AFA setting under the NUC assumption (Assumption \ref{assump:nuc}) and a latent projection. %(over the unknown states $U$, the true label $Y$ and the predicted label $Y^*$).
The graph depicts a standard, identified offline RL setting. Long-term dependencies are omitted from the graph for visual clarity. These include edges $\underline{X}^{t-1}, \underline{A}^{t-1} \rightarrow A^t$; $\underline{X}^{T}, \underline{A}^{T} \rightarrow C$; $\underline{X}^{t-1}  \leftrightarrow X^{t}$ and $\underline{X}^{T} \leftrightarrow C$. % $\forall t$.
}
\vspace{-1 pt}
\label{graph_offline_RL}
%\end{wrapfigure}
\end{figure}

\subsection{Identification}

Under the offline RL view, solving the AFAPE problem is equivalent to performing off-policy policy evaluation (OPE). Identification for OPE requires sequential exchangeability (also known as sequential ignorability), which implies that adjusting for $\underline{A}^{t-1}$ and $\underline{X}^{t-1}$ eliminates any confounding factors affecting $A^{t}$. 
The graph in Figure \ref{graph_offline_RL} satisfies this requirement. Note that sequential exchangeability would not hold under violations of the NUC assumption (Assumption \ref{assump:nuc}) because the latent projection under the presence of edges $\underline{U}^{t-1} \rightarrow A^t$ would produce confounding edges $\overline{X}^t \leftrightarrow A^{t}$ and $C \leftrightarrow A^{t}$. 

Identification of $J$ further requires Assumptions \ref{assump:measurement_noise}-\ref{assump:interference} (no measurement noise, consistency, no interference) and the following (sequential) positivity assumption: 

\vspace{5pt}
\noindent
\setcounter{myassumption}{6}
\setcounter{subassumption}{0}
\begin{subassumption}[Positivity for offline RL]
\label{assump:positivity_offline_RL}
\begin{flalign*}
%\label{eq_positivity_offline_RL}
\text{if }  \hspace{14 pt} \quad \quad \quad \quad \quad \quad
& p( 
\underline{X}^{t-1}_{(\pi_\alpha)} = \underline{x}^{t-1}, 
\underline{A}^{t-1}_{(\pi_\alpha)} = \underline{a}^{t-1}  
) \pi_\alpha^t(a^t|     
\underline{x}^{t-1}, 
\underline{a}^{t-1} )
> 0,
\nonumber
&&
\\
\text{then } \quad \quad \quad \quad \quad \quad
& p(
\underline{x}^{t-1}, 
\underline{a}^{t-1}
)
\pi_{\beta}^t(a^t| 
\underline{x}^{t-1}, 
\underline{a}^{t-1})
\geq  \mathcal{O} 
\nonumber
&&
\\ 
&
\forall  t, a^t,
\underline{x}^{t-1}, \underline{a}^{t-1}, \text{ and some constant } \mathcal{O} > 0   &&
\end{flalign*}
\end{subassumption}

%where $p( 
%\underline{X}^{t-1}_{(\pi_\alpha)}, 
%\underline{A}^{t-1}_{(\pi_\alpha)} 
%)$ denotes a marginal distribution of the counterfactual variables up to time-step $t$ under $\pi_{\alpha}$. 
% \noindent 
where we introduced the following notation:   $\pi_\alpha^t(a^t|     
\underline{x}^{t-1}, 
\underline{a}^{t-1} ) 
\equiv 
 \pi_\alpha^t(A^t=a^t| \underline{X}^{t-1} = \underline{x}^{t-1}, 
\underline{A}^{t-1} = \underline{a}^{t-1} )$. 

The positivity assumption states that, for every set of actions and observations 
$\underline{a}^{t-1}, \underline{x}^{t-1}$ reachable under $\pi_\alpha$ and a desired next action $a^{t}$ 
(i.e., an action with positive support under $\pi_\alpha$), we require also positive support for $a^{t}$ under $\pi_\beta$. 
A violation of this assumption may occur if the acquisition decisions under the AFA policy $\pi_{\alpha}$  differ significantly from the decisions made by doctors ($\pi_\beta$).

\subsection{Estimation}

Estimation can be performed using well-known techniques from the offline RL / DTR literature. The following are common estimators:

\vspace{10pt}
\noindent
\textit{1) Plug-in of the G-formula: } 

\noindent
The model-based estimator, also known in the causal inference literature as the noniterative conditional expectation (NICE) estimator of the G-formula \cite{rein_deep_2024,wen_parametric_2021}, estimates the target cost 
$J$ as follows:
% The target cost $J$ that is estimated by the plug-in of the G-formula \cite{hernan_causal_2020} is 
\begin{align}
    \hat{J}_{\textit{MB-Off}} = 
    %\frac{1}{n_\text{s}}\sum_j^{n_\text{s}} 
    \sum_{X, A} 
    \hat{\mathbb{E}}[C
    |\underline{X}^T, \underline{A}^T] 
    \prod_{t=1}^T\hat{p}(X^{t}| \underline{X}^{t-1}, \underline{A}^{t}) 
    \pi_{\alpha}^t(A^{t}| \underline{X}^{t-1}, \underline{A}^{t-1}) 
\end{align}
where the integration over $X$ and $A$ can be solved using Monte Carlo integration. Note that we use sums to denote the integration over $X$. All results in this paper do, however, also hold for continuous $X$, by replacing the sums with proper integrals. 

This estimator %(also known in RL literature as model-based evaluation)
requires learning the state transition function $p(X^{t}| \underline{X}^{t-1}, \underline{A}^{t})$ and the expected cost $\mathbb{E}[C|\underline{X}^T, \underline{A}^T]$. We denote the learned nuisance functions as $\hat{p}(X^{t}| \underline{X}^{t-1}, \underline{A}^{t})$ and $\hat{\mathbb{E}}[C|\underline{X}^T, \underline{A}^T]$. 

\vspace{10pt}
\noindent
\textit{2) Inverse probability weighting (IPW):}
\nopagebreak

\noindent
The target cost that is estimated by IPW \cite{levine_offline_2020} is 
\begin{align}
    \hat{J}_{\textit{IPW-Off}} = %\frac{1}{n_D} \sum_{i}^{n_D} 
    \hat{\mathbb{E}}_{n}\left[
    \hat{\rho}_{\textit{Off}}^T \text{ }C\right], \text{where  } \hat{\rho}_{\textit{Off}}^T = \prod_{t=1}^T \frac{\pi_{\alpha}^t(A^{t}| \underline{X}^{t-1}, \underline{A}^{t-1}) }{\hat{\pi}_{\beta}^t(A^{t}| \underline{X}^{t-1}, \underline{A}^{t-1})}. 
\end{align}
where $\hat{\mathbb{E}}_{n}\left[.\right]$ denotes the empirical average over the data set $\mathcal{D}$.
This estimator requires learning the retrospective acquisition policy/propensity score model $\pi_{\beta}^t(A^{t}| \underline{X}^{t-1}, \underline{A}^{t-1})$. 
\begin{myremark}[Cross-fitting estimators]
    When nuisance functions are estimated using flexible machine learning methods, the training of these functions and the evaluation of the estimator must be conducted on separate data splits to avoid introducing bias. To enhance efficiency, this process can be performed by alternating the training and evaluation splits, a method known as cross-fitting \cite{chernozhukov_doubledebiased_2018}. Throughout this paper, we assume that cross-fitting is used, though this is not explicitly noted in the proposed estimators' notation, leading to a slight abuse of notation.
\end{myremark}

\vspace{5pt}
\noindent
\textit{3) Direct method (DM):}

\noindent
The DM estimator, also known as the iterative conditional expectation (ICE) estimator of the G-formula \cite{wen_parametric_2021}, estimates the target cost $J$ as:
% The target cost that is estimated by the DM \cite{levine_offline_2020} is 
\begin{align}
    \hat{J}_{\textit{DM-Off}} = \hat{\mathbb{E}}_n[\hat{V}_{\textit{Off}}^0]. 
\end{align}
This estimator relies on learning a state-action value function $Q_{\textit{Off}}^{t}$ or state value function
$V_{\textit{Off}}^{t}$:
\begin{align*}
  Q_{\textit{Off}}^{t}  
  & \equiv Q_{\textit{Off}}(\underline{X}^{t-1},\underline{A}^{t}) 
  \equiv \mathbb{E}[C_{(\overline{\pi}_\alpha^{t+1})}|\underline{X}^{t-1},\underline{A}^{t}],  \\
  V_{\textit{Off}}^{t} & \equiv V_{\textit{Off}}(
  \underline{X}^t,
  \underline{A}^t) \equiv \mathbb{E}[
  C_{(\overline{\pi}_\alpha^{t+1})}|
  \underline{X}^t,
  \underline{A}^t].
\end{align*}
where $C_{(\overline{\pi}_\alpha^{t+1})}$ denotes the potential outcome of $C$ under a policy intervention $\pi_\alpha$ applied only from time step $t+1$ onwards. 
 $Q_{\textit{Off}}$ and $V_{\textit{Off}}$ can be learned using, for example, the dynamic programming (DP) algorithm, which is based on the recursive property of the Bellman equation \cite{bertsekas_dynamic_2012}:
\begin{align}
Q_{\textit{Off}}(\underline{X}^{t-1} , \underline{A}^{t} ) &  = 
\sum_{X^t}                               V_{\textit{Off}}( \underline{X}^{t},    \underline{A}^{t})            
p(X^t|       \underline{X}^{t-1},                 \underline{A}^{t} ) 
\\
    V_{\textit{Off}}( 
    \underline{X}^{t},    
    \underline{A}^{t})           
    & = 
    \sum_{A^{t+1}}
    Q_{\textit{Off}}(            \underline{X}^{t} , \underline{A}^{t+1} )
    \pi_{\alpha}^{t+1}(A^{t+1}|\underline{X}^{t}, \underline{A}^{t})
    %\nonumber
\end{align}
In practice, one only needs to learn
$Q_{\textit{Off}}$ such that $V_{\textit{Off}}$ can be simply computed as 
$V_{\textit{Off}}^t = \mathbb{E}_{\pi_{\alpha}}[Q_{\textit{Off}}^{t+1}] 
$ using, for example, Monte Carlo integration over the known AFA policy $\pi_{\alpha}$. 

\vspace{10pt}
\noindent
\textit{4) Double reinforcement learning (DRL):}

\noindent
The target cost that is estimated by DRL \cite{kallus_double_2020} is 
\begin{align}
    \hat{J}_{\textit{DRL-Off}} = \hat{\mathbb{E}}_n\left[\hat{\rho}_{\textit{Off}}^T C + 
 \sum_{t=1}^{T} \left(-\hat{\rho}_{\textit{Off }}^t  \hat{Q}_{\textit{Off}}^t +\hat{\rho}_{\textit{Off}}^{t-1}  \hat{V}_{\textit{Off}}^{t-1} \right) \right].
\end{align}
The DRL estimator combines approaches 2) and 3) by using both the learned propensity score $\hat{\pi}_\beta$ and the state action value function $\hat{Q}_{\textit{Off}}$ (and the derived $\hat{V}_{\textit{Off}}$). This estimator is (locally) efficient and doubly robust, in the sense that it is consistent if either the propensity score model $\hat{\pi}_\beta$, or the state action value function $\hat{Q}_{\textit{Off}}$ is correctly specified \cite{kallus_double_2020}.

\section{Missing Data (+ Online Reinforcement Learning) View}
\label{sec_missing_data}
\textit{Assumptions in this section: Assumption \ref{assump:nde} (NDE)}

\noindent
In this section, we assume that the NDE assumption holds  (i.e., $A^{t} \not \rightarrow \overline{U}^t$), but do not require the NUC assumption (Assumption \ref{assump:nuc}). 
We observe that the general AFA graph from Figure  \ref{graph_general} transforms under NDE into the graph shown in Figure \ref{graph_missing_data}A). This new graph represents a temporal missing data graph (m-graph) \cite{mohan_graphical_2013,shpitser_missing_2015} from the missing data literature. 
The unobserved underlying feature values $U^t$ are replaced by counterfactuals of the measured feature values $X_{(1)}^t$ since 
$U^t_{(\underline{a}^t)} = U^t_{(a^t)} = U^t_{(1)} = X^t_{(1)}$ for all potential acquisitions $\underline{a}^t$ and thus also $a^t = \vec{1}$. % For simplicity we assume $X^0_{(1)} = X^0$ to be fully observed.
Due to the temporal restrictions $\overline{X}_{(1)}^{t} \not \rightarrow A^t$, the shown graph can be more precisely specified as the known block-conditional missing data model 
\cite{zhou_block-conditional_2010}. The graph depicting the counterfactual distribution 
%under $do(\pi = \pi_{\alpha})$
is shown in Figure \ref{graph_missing_data}B).

\begin{figure}[ht]
\centering \includegraphics[width=1 \textwidth]{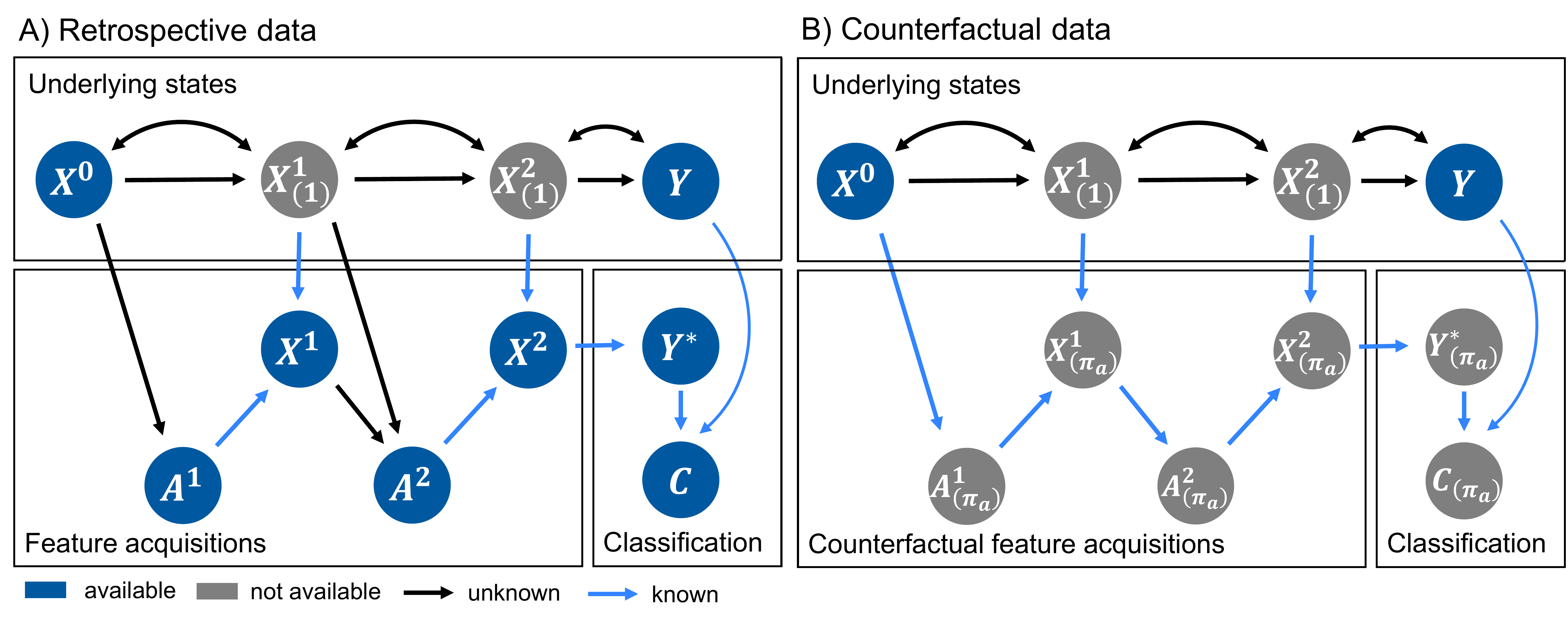}
    \caption{ 
    A) Updated causal graph of the AFA process under the NDE assumption (Assumption \ref{assump:nde}). Unknown state variables $U^t$ are replaced with the counterfactual feature values $X_{(1)}^t$, which represent the values $X^t$ would have taken if $A^t$ was $\vec{1}$ (i.e., the decision to observe all feature values). This graph describing the feature acquisition process is known as a missing data graph (m-graph). B) Graph showing the counterfactual distribution under $\pi_\alpha$.  %interventional distribution under the intervention $do(\pi = \pi_{\alpha})$. 
    Edges showing long-term dependencies are omitted for visual clarity. These include for both graphs $\underline{X}_{(1)}^{t-1} \leftrightarrow X_{(1)}^{t}$ and $\underline{X}_{(1)}^{T} \leftrightarrow Y$; for A) $\underline{X}^{t-1}, 
    \underline{X}_{(1)}^{t-1},
    \underline{A}^{t-1} \rightarrow  A^{t}$, and $\underline{X}^{T}, \underline{A}^{T} \rightarrow  Y^*$; and for B) 
    $X^0,\underline{X}_{(\pi_{\alpha})}^{t-1}, 
    \underline{A}_{(\pi_{\alpha})}^{t-1} \rightarrow  A_{(\pi_{\alpha})}^{t}$ and $X^0,\underline{X}_{(\pi_{\alpha})}^{T}, \underline{A}_{(\pi_{\alpha})}^{T} \rightarrow  Y_{(\pi_{\alpha})}^*$.
    }
    \label{graph_missing_data}
\end{figure}

\subsection{Problem Reformulation}

Now, we can establish the following theorem. Under the NDE assumption, the AFAPE problem becomes equivalent to a standard missing data problem, for which one can apply known identification and estimation theory.

\begin{theorem}
\label{theorem_problem_missing} 
(AFAPE problem reformulation and identification under the missing data view).
The AFAPE problem of estimating $J$ (Equation \ref{eq:AFAPE_objective})
is under Assumption \ref{assump:measurement_noise} (no measurement noise),  Assumption \ref{assump:consistency} (consistency),  Assumption \ref{assump:interference} (no interference), and Assumption \ref{assump:nde} (NDE) 
equivalent to estimating 

\begin{align}
\label{eq:AFAPE_objective_miss}
    J  
    = 
    \sum_{X_{(1)}, Y} 
    \underbrace{
    \mathbb{E}
    %\mathbb{E}_q
    \left[
    C_{(\pi_{\alpha})}|X_{(1)},Y
    \right]}_{\text{online RL}}
    \underbrace{p(X_{(1)}, Y)}_{\text{missing data}}.
\end{align}
Furthermore, $J$ is identified if $p(X_{(1)}, Y)$ is identified. 
\end{theorem}

\begin{proof}
The decomposition of $J$ into the two expected values follows from the law of iterated expectations and the independence of $X_{(1)}, Y$ from a policy intervention $\pi_\alpha$. %(i.e. $p(X_{(1)}, Y| do(\pi = \pi_{\alpha}))= p(X_{(1)}, Y)$). 
%Lines 2 and 3 of equation \ref{eq:AFAPE_objective_miss} are simply reformulations using different notation.
The fact that the inner expected value
%, which we refer to as the online RL problem,  
is identified can be easily verified by examining the graph representing the counterfactual distribution shown in Figure \ref{graph_missing_data}B). The graph shows that all functional relationships between variables that are part of the feature acquisition and classification processes are known (represented as blue edges). In particular, this implies the following factorization: 
\begin{align}
\label{eq_identification_online_RL}
 \mathbb{E}
    [
    C_{(\pi_{\alpha})}|X_{(1)},Y
    ]
     =  \sum_{
     \underline{X}^T_{(\pi_{\alpha})}
     ,\underline{A}^T_{(\pi_{\alpha})}
     ,Y^*_{(\pi_{\alpha})}
     ,C_{(\pi_{\alpha})}
     }
    C_{(\pi_{\alpha})}
    q(C_{(\pi_{\alpha})}
    , Y^*_{(\pi_{\alpha})}
    , X_{(\pi_{\alpha})}
    , A_{(\pi_{\alpha})}|
    X_{(1)},Y) 
\end{align}
with the identifying distribution
\begin{align*}
    q&(C_{(\pi_{\alpha})},  Y^*_{(\pi_{\alpha})}, X_{(\pi_{\alpha})}, A_{(\pi_{\alpha})}|
    X_{(1)},Y) = 
    \\ &  = 
    \prod_{t=0}^T
    \underbrace{
    g(
    X_{(\pi_{\alpha})}^t
    |
    A_{(\pi_{\alpha})}^{t},
    X_{(1)}^t
    ) 
    }_{\text{feature revelations}}
    \prod_{t=1}^T
    \underbrace{
    \pi_{\alpha}^t(
    A_{(\pi_{\alpha})}^t|
    \underline{X}_{(\pi_{\alpha})}^{t-1},
    \underline{A}_{(\pi_{\alpha})}^{t-1}) 
    }_{\text{acquisition decisions}}
    % \\ & \cdot  
    \underbrace{
    g_\textit{cl}(Y_{(\pi_{\alpha})}^*|
    \underline{X}_{(\pi_{\alpha})}^T, \underline{A}_{(\pi_{\alpha})}^T 
    )}
    _{\text{label prediction}}
    \underbrace{
    g_C(C_{(\pi_{\alpha})}|
    Y, 
    Y_{(\pi_{\alpha})}^* 
    )}
    _{\text{cost computation}}
\end{align*}

\noindent
which is identified since all (deterministic) distributions $g(.)$ and $\pi_{\alpha}$ are known functions. 
%\qedsymbol
\end{proof}

\vspace{5pt}
\noindent 
Since all densities $g()$ are deterministic and known, we can further simplify Eq. \ref{eq_identification_online_RL} using simpler notation. We use $X = G_A(X_{(1)})$, $\pi_{\alpha}(
    a|
    G_{a}(X_{(1)})) \equiv \prod_{t=0}^T
     \pi_{\alpha}^t(
    a^{t}|
    G_{\underline{a}^{t-1}}(X_{(1)}),
    \underline{a}^{t-1})$  and $C = f_C(A, G_{A}(X_{(1)}), Y)$ to obtain: 
\begin{align}
\label{eq_full_data_if}
 \mathbb{E}
    [
    C_{(\pi_{\alpha})}|X_{(1)},Y
    ]
    &  =  
     \sum_{a \in \mathcal{A}}
     f_C( \underline{a}^T, G_{\underline{a}^T}(X_{(1)}), Y ) 
     \prod_{t=0}^T
     \pi_{\alpha}^t(
    a^{t}|
    G_{\underline{a}^{t-1}}(X_{(1)}),
    \underline{a}^{t-1})
     \\
 & =
      \sum_{a \in \mathcal{A}}
     f_C(a, G_{a}(X_{(1)}), Y ) 
     \pi_{\alpha}(
    a|
    G_{a}(X_{(1)})).  
         \nonumber
\end{align}
\hfill  
%\ensuremath{\square} # 

\vspace{5pt}
\noindent 
We denote $\mathbb{E}
[
C_{(\pi_{\alpha})}|X_{(1)},Y
]$ as the online RL part of the problem because it involves the evaluation of a policy in a known environment. 
We refer to the outer expected value of $p(X_{(1)}, Y)$ as the missing data problem, as it requires the identification of the counterfactual feature distribution.  

\subsection{Identification}

As established in Theorem \ref{theorem_problem_missing}, the AFAPE problem is identified if the missing data problem (i.e., $p(X_{(1)}, Y)$) is identified. 
The following positivity assumption is required to allow identification of $p(X_{(1)}, Y)$ and therefore for the target parameter $J$ from Eq. \ref{eq:AFAPE_objective_miss}:

\vspace{10pt}
\noindent
\setcounter{myassumption}{6}
\setcounter{subassumption}{1}
\begin{subassumption}[Positivity for missing data]
\label{assump:positivity_missing_data}
\begin{flalign*}
%\label{eq_positivity_missing_data}
\text{if }  \hspace{14 pt} \quad  \quad  \quad  \quad  \quad  \quad  \quad \quad  \quad  \quad  \quad
& p( 
\underline{X}^{t-1}_{(1)}= \underline{x}^{t-1}, 
\underline{A}^{t-1} = \vec{1}  
) 
> 0,
\nonumber
&&
\\
\text{then } \quad \quad  \quad  \quad  \quad  \quad  \quad \quad  \quad  \quad  \quad
& 
\pi_{\beta}^t(A^t=\vec{1}| 
\underline{X}_{(1)}^{t-1} = \underline{x}^{t-1}, 
\underline{A}^{t-1} = \vec{1})
\geq  \mathcal{O}  
\nonumber
&&
\\ 
&
\forall  t, 
\underline{x}^{t-1}, \text{ and some constant } \mathcal{O} > 0  &&
\end{flalign*}
\end{subassumption}

\noindent
This positivity assumption is very different from the positivity assumption assumed under the offline RL view (Assumption \ref{assump:positivity_offline_RL}). It requires the "acquire everything" action trajectory $A^t=\vec{1}$ $\forall t$ to have positive support for all possible feature values. In other words, this is a requirement for complete cases among all subpopulations.

Given the positivity assumption, the block-conditional model describing $p(X_{(1)}, Y)$ is identified. We show how identification can be achieved in Appendix \ref{app_mnar}. %For readers interested in missing data identification methods for various other settings, we recommend work by \citet{nabi_full_2020} or \citet{mohan_graphical_2021}. 

Note that Theorem \ref{theorem_problem_missing} holds even in the more general case without the temporal restriction $\overline{X}_{(1)}^{t} \not \rightarrow A^t$. In this case, $p(X_{(1)}, Y)$ may or may not be identified, depending on what assumptions can be made. There exists a vast literature on identification theory for missing data problems \cite{bhattacharya_identification_2020,nabi_full_2020, mohan_graphical_2021} that can be applied. 

Here, we briefly discuss how the common missing data scenarios of missing-completely-at-random (MCAR), missing-at-random (MAR), and missing-not-at-random (MNAR) apply within the AFA framework.
MNAR scenarios arise when the missingness of a feature is influenced by the value of another feature which may itself be missing. This situation occurs when the NUC assumption is violated, indicated by the presence of arrows $X_{(1)} \rightarrow A$. On the other hand, MAR scenarios occur when the missingness of a feature is dependent only on observed features. This scenario aligns with our setting if the NUC assumption (Assumption \ref{assump:nuc}) holds. Lastly, MCAR represents the simplest case, where the missingness of a feature is independent of all other feature values. In our framework, this corresponds to the absence of any edges $X \rightarrow A$.

\subsection{Estimation}

An estimate of \(\mathbb{E}\left[C_{(\pi_{\alpha})}|X_{(1)},Y\right]\), denoted as \(\hat{\mathbb{E}}\left[C_{(\pi_{\alpha})}|X_{(1)},Y\right]\), can be readily computed from Eq. \ref{eq_full_data_if}
%\ref{eq_identification_online_RL} 
using Monte Carlo integration: 
\begin{flalign}
\label{eq:online_rl_mc}
     \hat{\mathbb{E}}[C|X_{(1)},Y] 
    = 
    \sum_{i}^{n_{\textit{MC}}} 
     f_C(a_i, G_{a_i}(X_{(1)}), Y )
    % f_c(Y, G_{\underline{a}_i^{*{T}}}(X_{(1)}))
\end{flalign}
with $n_{\textit{MC}}$ samples 
$
% \underline{a}_i^{*T} 
a_i
\sim \prod_{t=1}^T \pi_\alpha(A^{t}|G_{\underline{A}^ {{t-1}}}(X_{(1)}),\underline{A}^ {{t-1}})$.

 This approach is common in online RL settings, where the agent interacts with the environment.
Monte Carlo integration is performed in the following way: First, one takes a sample $x,y$ from \(p(X_{(1)},Y)\) (obtained through either of the methods shown below) and reveals the initial features \(x^0\) to the agent. The agent's first action, denoted by \(a'^1\), is then sampled from \(\pi_\alpha(A^1|x^0)\). Depending on \(a'^1\), the corresponding feature set amongst \(x^1\) is revealed to the agent and the next action \(a'^2\) is sampled. This continues until step \(T\), when the classifier is applied (using the acquired subset of features from the simulation as input) and a resulting misclassification cost \(c'\) is computed. This process is repeated multiple times, and the costs are averaged to obtain \(\hat{\mathbb{E}}\left[C_{(\pi_{\alpha})}|X_{(1)},Y\right]\).

This online RL process requires samples from \(p(X_{(1)},Y)\). These samples can be obtained using standard missing data estimation methods, which result in the following estimators: 

\vspace{10pt}
\noindent
\textit{1) Inverse probability weighting (IPW):}

\noindent
The target cost that is estimated by IPW \cite{seaman_review_2013} is
\begin{align}
    \hat{J}_{\textit{IPW-Miss}} = \hat{\mathbb{E}}_n \left[ \hat{\rho}_{\textit{Miss}}\text{ } \hat{\mathbb{E}}\left[C_{(\pi_{\alpha})}|X_{(1)},Y\right]\right], \text{  where  } \hat{\rho}_{\textit{Miss}} =  
    \prod_{t=1}^T
    \frac{\mathbb{I}(A^t= \vec{1}) }{
    \hat{\pi}_{\beta}^t(
    A^t=\vec{1}| 
    \underline{X}_{(1)}^{t-1},
    \underline{A}^{t-1}= \vec{1})}, 
\end{align}
and where $\mathbb{I}(.)$ denotes the indicator function.
This IPW approach requires learning the propensity score $\pi_{\beta}$, but only for the scenario of full data acquisition (where $A = \vec{1}$). Because of the indicator function $\mathbb{I}(A^t = \vec{1})$, only the complete cases are selected for reweighting. 
%for the simulation of $C'$.

\vspace{10pt}
\noindent
\textit{2) Multiple Imputation (MI): }
% Plug-in of the G-formula:} 

\noindent
The target cost that is estimated by the multiple imputation (MI)  \cite{sterne_multiple_2009} estimator is: 
% plug-in of the G-formula is
% \begin{align} \label{eq:MI}
%    \hat{J}_{\textit{G-Miss}} = \sum_{X_{(1)},Y} \hat{\mathbb{E}}\left[C_{(\pi_{\alpha})}|X_{(1)},Y\right]  \hat{p}(X_{(1)},Y).  
%\end{align}
% \begin{align} \label{eq:MI}
%     \hat{J}_{\textit{MI-Miss}} = \hat{\mathbb{E}}_n \left[
%     \sum_{X_{m}} \hat{\mathbb{E}}\left[C_{(\pi_{\alpha})}|X_{(1)},Y\right]  \hat{p}(X_{m}| X_{o}, Y) \right].  
% \end{align}
\vspace{6pt}

\begin{align} \label{eq:MI}
    \hat{J}_{\textit{MI-Miss}} = \hat{\mathbb{E}}_n \left[
    \sum_{\tikzmarknode{Xm}{X_{m}}} \hat{\mathbb{E}}\left[C_{(\pi_{\alpha})}|X_{(1)},Y\right]  \hat{p}(\tikzmarknode{Xm2}{X_{m}}| \tikzmarknode{Xo}{X_{o}}, Y) \right].
\end{align}

\begin{tikzpicture}[remember picture, overlay]
    % Arrow pointing down to X_m (from above with y-offset)
    \draw[<-, thick] (Xm2) ++(0cm, 0.3cm) -- ++(0cm, 0.4cm) node[above] {\footnotesize Missing part of \(X_{(1)}\)};
    
    % Arrow pointing up to X_o (from below with y-offset)
    \draw[<-, thick] (Xo) ++(0cm, -0.3cm) -- ++(0cm, -0.4cm) node[below] {\footnotesize Observed part of \(X_{(1)}\)};
\end{tikzpicture}

\vspace{0.5cm}

\noindent 
This estimator is based on the G-formula, which requires the counterfactual data distribution $p(X_{(1)},Y)$. In the MI estimator, this density is not modeled fully, but the empirical distribution of the available data is augmented with samples (i.e., imputations) from a model for the missing data. %This approach is known as multiple imputation (MI) \cite{sterne_multiple_2009}. 
The MI estimator is based on the decomposition $\hat{p}(X_{(1)},Y) = \hat{p}(X_{m}|X_{o},Y)p(X_{o},Y)$, where $X_{o}$ denotes the observed part and $X_{m}$ the missing part of $X_{(1)}$.
The sampling of the missing part is then usually repeated multiple times to increase the precision of the estimate, hence the name "multiple imputation". 

In certain scenarios, MI can outperform the estimators from the semi-offline RL view, which will be described next. Appendix \ref{app_missing_data_estimators} discusses the advantages and disadvantages of the MI estimator in more detail. The appendix also highlights why using (conditional) mean imputation, which has been previously employed in AFA settings \cite{an_reinforcement_2022, erion_coai_2021, janisch_classification_2020} generally leads to biased estimation results.

\section{Semi-offline Reinforcement Learning View}
\label{sec_semi_offline_view}

\textit{Assumptions in this section: Assumption \ref{assump:nde} (NDE), Assumption \ref{assump:nuc} (NUC)}%,  \{measuring equals recording}}

\noindent 
In this section, we assume both the NDE and NUC assumptions hold. In this context, either the offline RL or the missing data view can be applied to solve AFAPE, but both have limitations, as illustrated in the following scenarios.

We assume an available data point without $X_{(1)}^0$, but with two time-steps and univariate $X_{(1)}^1 = 0.6$ and $X_{(1)}^2 = 0.9$. Initially, we assume the retrospective acquisition decisions were $A^1 = 1$ and $A^2 = 1$ (denoted as scenario 1). The data point in scenario 1 corresponds to a complete case where both feature values $X_{(1)}^1$ and $X_{(1)}^2$ are known.

\begin{figure}[ht]
\centering 
\hspace*{-0.5cm}
\includegraphics[width=1.05 \textwidth]{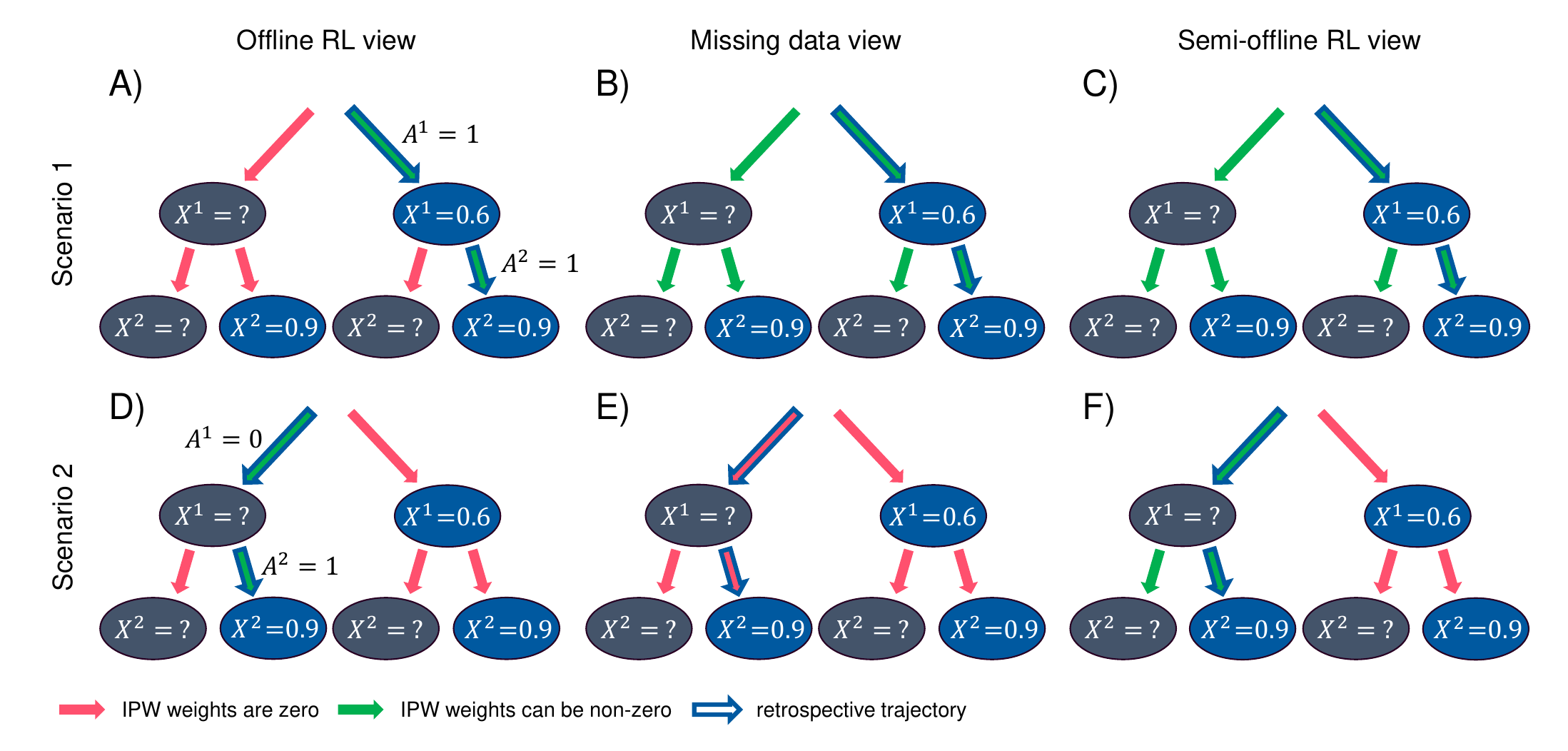}
    \caption{ 
    Visualization of data utilization by IPW estimators under different views. Each graph shows the four possible target acquisition trajectories for two exemplary retrospective acquisition scenarios and highlights which target trajectories can receive non-zero IPW weights under the respective views.
    A), D) The IPW estimator from the offline RL viewpoint: only target trajectories that match the retrospective trajectory can be evaluated. B), E) The IPW estimator from the missing data viewpoint: all trajectories can be evaluated if the datapoint is a complete case; otherwise, no evaluation is possible. C), F) The IPW estimator from the semi-offline RL viewpoint: all trajectories with equal or fewer acquisitions than the retrospective trajectory can be evaluated.
    }
    \label{figure_trajectories}
\end{figure}

We analyze how this data point is utilized by the IPW estimators within different views. Figures \ref{figure_trajectories}A) and \ref{figure_trajectories}B) show the four possible target trajectories: $a^1 = 0, a^2 = 0$; $a^1 = 1, a^2 = 0$; $a^1 = 0, a^2 = 1$; and $a^1 = 1, a^2 = 1$ under the offline RL and missing data views, respectively. In the offline RL view, this data point can only be used to evaluate the target trajectory $a^1 = 1$ and $a^2 = 1$, matching the retrospective trajectory. The missing data IPW estimator, however, can evaluate any of the four target trajectories using this data point since it is a complete case.

Now consider scenario 2 with hypothetical retrospective acquisition decisions $A^1 = 0$ and $A^2 = 1$. Figures \ref{figure_trajectories}D) and \ref{figure_trajectories}E) depict the corresponding trajectories for both views. The offline RL IPW estimator can use this datapoint to evaluate the matching trajectory $a^1 = 0$ and $a^2 = 1$ but none others. The missing data IPW estimator cannot evaluate any target trajectories as this is not a complete case.

In scenario 2, we only know the value of $X_{(1)}^2$, not $X_{(1)}^1$, since $A^1 = 0$. However, the target trajectories $a^1 = 0$, $a^2 = 0$ and 
$a^1 = 0$, $a^2 = 1$
do not require the value of $X_{(1)}^1$ and can still be simulated. This motivates our novel semi-offline RL viewpoint, where we consider all target trajectories with the same or fewer acquisitions, compared with the retrospective data point, as simulatable.

Figures \ref{figure_trajectories}C) and \ref{figure_trajectories}F) show the simulatable trajectories under the semi-offline RL view for both scenarios. Similar to online RL, a policy can sample different trajectories, but under the semi-offline RL view only among the simulatable ones, while trajectories involving non-simulatable actions (like $a^1 = 1$ in scenario 2) should not be sampled.

Since not all trajectories are simulatable, we restrict the simulation policy to block actions that would result in non-simulatable trajectories. A simple Monte Carlo estimator averaging the costs within the simulated trajectories would be biased due to the blocking of actions, necessitating post-simulation bias correction.

The remainder of this section formalizes the semi-offline RL viewpoint and is organized as follows. First, we introduce the simulation policy, referred to as the semi-offline sampling policy, and illustrate that simulations using this policy do not require information about $X_{(1)}$ not already contained in $X = G_A(X_{(1)})$. Next, we explain how the AFAPE target $J$ can be equivalently formulated based on the semi-offline sampling distribution. We then prove that this reformulated $J$ is identified under a new, weaker positivity assumption, and finally, we derive novel estimators for $J$.

\subsection{The Semi-offline Sampling Policy}
To introduce the semi-offline RL view, we first revisit the problem formulation from the missing data view (Eqs. \ref{eq:AFAPE_objective_miss} and \ref{eq_full_data_if}):
%\ref{eq:AFAPE_objective_miss}): 
\begin{align}
%\label{eq_full_data_if}
\label{eq:AFAPE_objective_miss_repeated}
    J  
    & = 
    \sum_{X_{(1)}, Y} 
    \underbrace{
     \sum_{a \in \mathcal{A}}
     f_C(a, G_{a}(X_{(1)}), Y) 
     \pi_{\alpha}(
     a|
     G_{a}(X_{(1)}))
     }_{\text{online RL}}
    \underbrace{p(X_{(1)}, Y)}_{\text{missing data}}.
    %\\
    %& = 
    %\sum_{X_{(1)}, Y, A_{(\pi_\alpha)}}
    %f_C(A_{(\pi_\alpha)}, G_{A_{(\pi_\alpha)}}(X_{(1)}), Y) 
    %p(A_{(\pi_\alpha)}, Y, X_{(1)}).
    %\nonumber
\end{align}
The online RL part involves integrating over all possible trajectories $a \in \mathcal{A}$, representing subsets of the features.
Notably, many terms in this inner integral do not require complete knowledge of $X_{(1)}$. Specifically, any summand with $a \leq A$ (where we let $\leq$ denote an element-wise comparison) does not need information from $X_{(1)}$ beyond what is already in $X = G_A(X_{(1)})$.

This missing data + online RL approach effectively says: "Solve all missingness, then integrate over all possible feature subsets, including many that didn't require addressing the missingness to begin with." This approach is clearly suboptimal. Instead, we propose a novel semi-offline RL viewpoint: "Integrate over all subsets of the observed data and adjust for the bias introduced by excluding subsets where data was missing."
We call this the semi-offline RL viewpoint because some subsets/trajectories can be simulated (the online part), while others cannot (the offline part).

We define the semi-offline sampling distribution to include the observed data trajectories. To avoid integrating over unobserved feature subsets, we replace $\pi_\alpha$ in the distribution of Eq. \ref{eq:AFAPE_objective_miss_repeated} with a policy $\pi_{\textit{sim}}^\prime$ that has no support for trajectories where the corresponding $X_{(1)}$ are missing. A policy $\pi^\prime$ that enforces this exclusion is defined formally as a blocked policy:

\begin{definition}
\label{def_blocked_policy}
(Blocked Policy)
A policy $\pi^{\prime t}(A'^t|G_{\underline{A}'^{t-1}}(X_{(1)}),\underline{A}'^{t-1}, A^t)$ is called a 'blocked policy' of the policy $\pi^t(A'^t|G_{\underline{A}'^{t-1}}(X_{(1)}),\underline{A}'^{t-1})$
if it satisfies the following conditions: 

\noindent 
1) Blocking of acquisitions of non-available features:
\begin{flalign*}
% \label{eq_def_1_1}
\text{if }  
& a'^t \not \leq a^t, %\quad  \text{for any } i,
%\nonumber
%&&
%\\
\quad \quad \quad \quad \quad  \quad     \quad\text{then } 
\pi^{\prime t}(
%A'^t = 
a'^t|
%\underline{X}'^{t-1} = 
G_{\underline{a}'^{t-1}}(X_{(1)}) = \underline{x}^{\prime t-1},
%\underline{A}'^{t-1} = 
\underline{a}'^{t-1},
%A^t = 
a^t) = 0
\nonumber
%&&
%\\ 
\quad 
&&
\forall  t,a'^t,a^t, \underline{x}'^{t-1}, \underline{a}'^{t-1}  %&&
\end{flalign*}

\noindent 
2) No blocking of acquisitions of available features:
\begin{flalign*}
%\label{eq_def_1_2}
\text{if }  
& a'^t \leq a^t  
\text{ and }
\pi^t(
%A'^t = 
a'^t|
%\underline{X}'^{t-1} = 
\underline{x}'^{t-1},
%\underline{A}'^{t-1} = 
\underline{a}'^{t-1}) > 0,  
\quad 
\text{then } 
\pi^{\prime t}(
%A'^t = 
a'^t|
%\underline{X}'^{t-1} = 
\underline{x}'^{t-1},
%\underline{A}'^{t-1} = 
\underline{a}'^{t-1}, 
%A^t = 
a^t) > 0
%\quad 
&&  
\forall  t,a'^t,a^t, \underline{x}'^{t-1}, \underline{a}'^{t-1} % &&
\end{flalign*}
\end{definition}
\noindent 
Condition 1 ensures that sampling does not depend on values of $X_{(1)}$ that are not contained in $X = G_A(X_{(1)})$.
Condition 2 ensures that the online exploration part is utilized, by forcing the blocked policy $\pi'$ to have positive support whenever $\pi$ has positive support and the desired features are available. A practical choice for $\pi'$ just sets $\pi'(A'^t=a'^t|G_{\underline{A}'^{t-1}}(X_{(1)}),\underline{A}'^{t-1}, A^t = a^t) = 0$ if $a'^t \not \leq a^t$ and rescales the other probabilities accordingly.

Having defined the blocked policy, we can now introduce the semi-offline sampling distribution: 
\begin{flalign}
\label{eq:semi_offline_sampling_distribution}
    p'(A', G_A(X_{(1)}), Y, A) & =
    \prod_{t=1}^T \pi_\textit{sim}^{\prime t}(A'^t| G_{\underline{A}^{\prime t-1}}(X_{(1)}), \underline{A}^{\prime t-1}, A^t) 
    p(A,G_A(X_{(1)}), Y)
    \\ & \equiv
    \underbrace{\pi_\textit{sim}^{\prime}(A'|G_{A'}(X_{(1)}),A)}_{\text{semi-offline RL}} \underbrace{p(A,G_A(X_{(1)}), Y)}_{\text{observed data}}
     \nonumber
\end{flalign} 
 which only involves observed data because $A' \leq A$ due to the blocking. 

Similar to online RL, we can sample from this distribution to construct a simulated dataset $\mathcal{D}^{\prime}$, consisting of $ X = G_{A}(X_{(1)}), Y, A, A'$ and $C' = f_C(A',G_{A'}(X_{(1)}), Y)$. In Figure \ref{graph_semi_offline_RL}, we show the full causal graph of how such sampling is performed.

\begin{figure}
\centering
\vspace{-1 pt}
\includegraphics[width=0.60\textwidth]
    {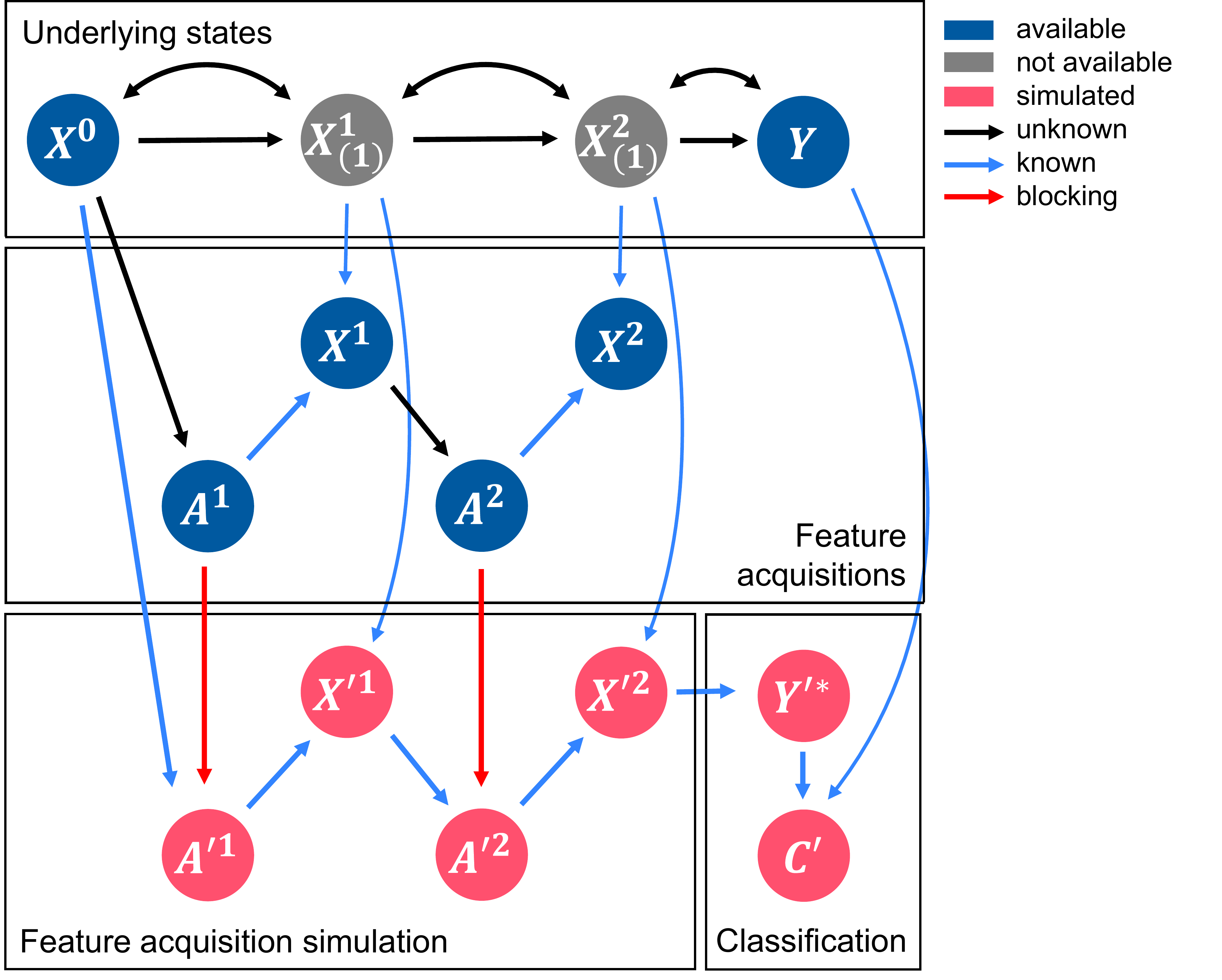}
\vspace{-1 pt}
    \caption{ 
    Causal graph for the semi-offline sampling distribution $p'$. % which constitutes the  the semi-offline generation of a simulated dataset $\mathcal{D}'$. 
    Simulated acquisition actions $A'^t$ and observations $X'^t = G_{A^{\prime t}}(X_{(1)}^t)$  follow a blocked simulation policy $\pi'_{sim}$. The simulation policy is restricted by $A^t$ such that actions $A'^t$ are blocked if $A'^t_i > A_i^t$ for any $i$.
    The simulated cost $C'$ can be computed from $A', X'$ and $Y$: $C' = f_C( Y^{\prime*},Y)$ with $Y^{\prime *} = f_\textit{cl}(X',A')$ being the predicted label under the simulated acquisitions. 
    Edges showing long-term dependencies are omitted from the graph for visual clarity. These include $ \underline{X}^{t-1}, \underline{A}^{t-1} \rightarrow  A^{t}$; $ \underline{X}'^{t-1}, \underline{A}'^{t-1} \rightarrow  A'^{t}$; $\underline{X}'^{T}, \underline{A}'^{T} 
    \rightarrow  Y^*$; $\underline{X}_{(1)}^{t-1} 
    \leftrightarrow X_{(1)}^{t}$; and $\underline{X}_{(1)}^{T}
    \leftrightarrow Y$. 
    }
    \label{graph_semi_offline_RL}
\vspace{-1 pt}
%\end{wrapfigure}
\end{figure}

Note that although the number of observed data samples is limited by the dataset size $n$, the semi-offline RL variables $A'$ and $C'$ can be sampled multiple times beyond this constraint.
While $\pi_\textit{sim}^{\prime}$ can be chosen to be a blocked AFA policy $\pi_\alpha^\prime$, this is not required; different choices of $\pi_\textit{sim}^{\prime}$ introduce an off-policy aspect. The only requirement is that $\pi_{sim}$ (the unblocked version of $\pi_\textit{sim}^{\prime}$) meets the positivity assumption from the offline RL view (Assumption \ref{assump:positivity_offline_RL}).

Since we replace $\pi_\alpha$ with $\pi_\textit{sim}^\prime$, the resulting cost samples $C'$ cannot simply be averaged to estimate $J$. Instead, we reformulate the AFAPE problem as a causal inference problem by intervening on the semi-offline sampling distribution $p^\prime$, reversing $\pi_\textit{sim}^\prime$ back to $\pi_\alpha$. This reformulation is formalized as follows.

\subsection{Problem Reformulation}

The AFAPE problem can be reformulated under the semi-offline RL view (i.e. under the proposed distribution $p'$) as the following theorem states.

\begin{theorem}
\label{theorem_problem_semi_offline_RL} 
(AFAPE problem reformulation under the semi-offline RL view).
The AFAPE problem of estimating $J$ (Eq. \ref{eq:AFAPE_objective} or Eq. \ref{eq:AFAPE_objective_miss})
is under Assumption \ref{assump:measurement_noise} (no measurement noise),  Assumption \ref{assump:consistency} (consistency),  Assumption \ref{assump:interference} (no interference), Assumption \ref{assump:nde} (NDE) and Assumption \ref{assump:nuc} (NUC) 
equivalent to estimating 
\begin{align}
\label{eq:AFAPE_objective_semi_offline_RL}
    J  
    = 
    \mathbb{E}_{p'}[C'_{(\pi_\alpha)}].
\end{align}
\end{theorem}

\noindent 
$C'_{(\pi_\alpha)}$ denotes the potential outcome of $C'$, had, instead of the blocked simulation policy $\pi'_{sim}$, the AFA policy $\pi_\alpha$ been employed. 

\vspace{10pt}
\begin{proof}
Starting from Eq. \ref{eq:AFAPE_objective_miss_repeated}, we find: 
\begin{align*}
    J  
    & = 
    \sum_{X_{(1)}, Y} 
     \sum_{a \in \mathcal{A}}
     f_C(a, G_{a}(X_{(1)}), Y) 
     \pi_{\alpha}(
     a|
     G_{a}(X_{(1)}))
    p(X_{(1)}, Y)
    \\ & = 
    \sum_{ X_{(1)}, Y, A_{(\pi_\alpha)}}
    f_C(A_{(\pi_\alpha)}, G_{A_{(\pi_\alpha)}}(X_{(1)}), Y) 
    p(A_{(\pi_\alpha)}, X_{(1)},Y) 
    \\ & = 
    \sum_{X_{(1)}, Y,A_{(\pi_\alpha)}, A}
    f_C(A_{(\pi_\alpha)}, G_{A_{(\pi_\alpha)}}(X_{(1)}), Y) 
    p(A_{(\pi_\alpha)}, X_{(1)},Y, A) 
    \\ & = 
    \sum_{X_{(1)}, Y,A_{(\pi_\alpha)}, A}
    f_C(A^\prime_{(\pi_\alpha)}, G_{A^\prime_{(\pi_\alpha)}}(X_{(1)}), Y) 
    p^\prime(A^\prime_{(\pi_\alpha)}, X_{(1)},Y, A) 
    = 
    \mathbb{E}_{p'}[C'_{(\pi_\alpha)}].
\end{align*}
%\hfill %\qedsymbol
\end{proof}

\begin{myremark}[Comparison of AFAPE under offline RL and semi-offline RL] 
The AFAPE problem formulation under the semi-offline RL view (Eq. \ref{eq:AFAPE_objective_semi_offline_RL}) closely resembles the original AFAPE formulation under the offline RL view (Eq. \ref{eq:AFAPE_objective}): $J = \mathbb{E}[C_{(\pi_\alpha)}]$. This similarity might raise the question of what is gained by this reformulation if one still needs to adjust for the intervention $\pi_\alpha$. The key difference is that, in the offline RL view, one must account for a significant distribution shift from $\pi_\beta$ to $\pi_\alpha$. In contrast, under the semi-offline RL view, the distribution shift is much smaller, from a blocked $\pi_\textit{sim}^\prime$ to $\pi_\alpha$. 
\end{myremark}

\subsection{Identification}

\noindent 
In the following section, we focus on identifying the reformulated target $J$ from Eq. \ref{eq:AFAPE_objective_semi_offline_RL}. Similar to the other views, identification under the semi-offline RL view requires a positivity assumption. However, under the semi-offline RL view, this assumption is considerably less stringent compared to the positivity requirements in the offline RL (Assumption \ref{assump:positivity_offline_RL}) and missing data views (Assumption \ref{assump:positivity_missing_data}).

In the offline RL view, the positivity assumption requires support for retrospective trajectories (i.e., action sequences under $\pi_\beta$) that match any target trajectory (i.e., action sequences under $\pi_\alpha$). In the missing data view, it demands support for "acquire all" retrospective trajectories, meaning complete cases. In contrast, the semi-offline RL view only necessitates that for any target trajectory, there is positive support for at least one retrospective trajectory with equal or more acquisitions. We now formalize this positivity assumption rigorously. Due to its inherent complexity, we divide the formalization into multiple definitions:
% Nevertheless, this positivity assumption is more complex than the previous ones and is formalized as follows:

\begin{definition}
\label{def_local_positivity}
(Local positivity assumption and local admissible set  $\mathcal{A}_\textit{adm}$ for semi-offline RL).
Let the local admissible set
$\mathcal{A}_\textit{adm}^t(
\underline{x}^{t-1}, \underline{a}^{t-1}, 
a'^t)$, %\neq \emptyset$, 
defined for all $a'^t$ and all  $\underline{x}^{t-1}, \underline{a}^{t-1}$ s.t. $p(G_{\underline{a}^{t-1}}(X_{(1)}) \equiv \underline{X}^{t-1} = \underline{x}^{t-1}, \underline{A}^{t-1} = \underline{a}^{t-1})>0$, 
be the non-empty set containing all values of $a^t$ for which 
%\begin{align*}
%\pi_{\beta}(
%A^1 \geq a'^1 | x^0) > 0 
%\end{flalign*}
\begin{flalign*}
&(1)  \quad \quad \quad   a^t \geq  a'^t
\\ &(2) \quad \quad \quad  \pi_{\beta}^t(
 a^t | \underline{x}^{t-1},\underline{a}^{t-1})   
\geq  \mathcal{O}
\end{flalign*}
for some constant $\mathcal{O} > 0$.
%where $a^t \geq a'^t$ denotes an element-wise comparison. 
We further say that the local positivity assumption holds at $\underline{x}^{t-1}, \underline{a}^{t-1}, a'^t$ if $\mathcal{A}_\textit{adm}^t(
\underline{x}^{t-1}, \underline{a}^{t-1}, 
a'^t)$ exists. 
\end{definition}

\noindent 
The local positivity assumption states that the observed data allows the simulation of a desired action $a'^t$ (i.e. there is positive support for at least one value $a^t$ s.t.  $a^t\geq a'^t$). The local positivity assumption is, however, not enough, which leads to the following definition of regional positivity.

\begin{definition}
\label{def_regional_positivity}
(Regional positivity assumption and regional admissible set $\tilde{\mathcal{A}}_\textit{adm}$ for semi-offline RL).
Let the regional admissible set $\mathcal{\tilde{A}}_\textit{adm}^t(
\underline{x}^{t-1},
\underline{a}^{t-1},
\underline{a}'^{t})
\subseteq 
\mathcal{A}_\textit{adm}^t(
\underline{x}^{t-1},
\underline{a}^{t-1},
a'^{t})$,
%\neq \emptyset $, 
defined for all 
$
\underline{x}^{t-1},
\underline{a}^{t-1},
\underline{a}'^{t}
$
s.t. 
$p'(
G_{\underline{a}^{t-1}}{(X_{(1)}}) = \underline{x}^{t-1},
\underline{A}^{t-1} = \underline{a}^{t-1}, 
\underline{A}'^{t} = \underline{a}'^{t}) > 0$, 
be the non-empty set containing all values of $a^t$ such that $\mathcal{\tilde{A}}_\textit{adm}^{t+1}
(
\underline{x}^{t},
\underline{a}^{t},
\underline{a}'^{t+1})$ exists %for all $x'^t, 
%a'^{t+1}$, and 
%$x^t$ for which the %following holds: 
for all $
\underline{x}^{t},
\underline{a}^{t},
\underline{a}'^{t+1}$ such that %$a^{t} \in \mathcal{\tilde{A}}_\textit{adm}^t(\underline{x}'^{t},
%\underline{a}'^{t+1},
%\underline{x}^{t},
%\underline{a}^{t})$ and 
the following holds for $ 
a'^{t+1}$, and $x^t$:
\begin{align}
p(
%G_{a^{t}}{(X_{(1)}}) = 
x^{t}|
%G_{\underline{a}^{t-1}}{(X_{(1)}}) = 
\underline{x}^{t-1},
%\underline{A}^{t} = 
\underline{a}^{t}) 
%g(x'^t|
%x^t,
%a'^{t}) 
\pi_{\alpha}^t(
a'^{t+1} | 
\underline{x}'^{t-1},
% G_{\underline{a}'^{t-1}}{(X_{(1)}}) 
% = \underline{x}'^{t}
\underline{a}'^{t
})   
> 0
\end{align}
where $\underline{x}'^{t-1}$ denotes the acquired features under $\underline{a}'^{t-1}$ and thus a subset of $\underline{x}^{t-1}$.
We further say that the regional positivity assumption holds at $
\underline{x}^{t-1},
\underline{a}^{t-1},
\underline{a}'^{t}$ if \newline $\mathcal{\tilde{A}}_\textit{adm}^t(
\underline{x}^{t-1},
\underline{a}^{t-1},
\underline{a}'^{t})$ exists. 
\end{definition}

\noindent 
Regional positivity states that there is not only a value $a^t$ with positive support in the observed data that allows the simulation of a desired action $a'^t$ (i.e. local positivity), but it also ensures for such an $a^t$, that the simulations of all possible future desired actions are also possible. 
As regional positivity is still limited to a given datapoint, we also make the following global positivity assumption.

%\vspace{10pt}
\noindent
\setcounter{myassumption}{6}
\setcounter{subassumption}{2}
\begin{subassumption}[Global positivity assumption for semi-offline RL]
\label{assump:positivity_global_semi_offline_RL}
We say that the global positivity assumption holds if the regional positivity assumption holds for all datapoints $a'^1, x^0$ s.t. 
$p(x^0) \pi_\alpha(a'^{1}|
x^0)>0$.
\end{subassumption}

\noindent
The following lemma (proved in Appendix \ref{app_proof_lemma_comparison_positivity}) establishes that the global positivity assumption for semi-offline RL is weaker than both the positivity assumptions required by the offline RL and the missing data views:

\begin{lemma}
\label{lemma_comparison_positivity}
    (Sufficiency conditions for global positivity). The global positivity assumption for semi-offline RL (Assumption \ref{assump:positivity_global_semi_offline_RL}) holds if the positivity assumption from offline RL (Assumption \ref{assump:positivity_offline_RL}) or from missing data (Assumption \ref{assump:positivity_missing_data}) holds. 
\end{lemma}

\noindent
After having defined positivity for semi-offline RL, we can now perform identification for $J$: 

\begin{theorem}
\label{theorem_identification_semi_offline_RL} 
(Identification of $J$ for the semi-offline RL view).
The reformulated AFAPE problem of estimating $J$ under the semi-offline RL view  (Eq. \ref{eq:AFAPE_objective_semi_offline_RL}) is under Assumption \ref{assump:measurement_noise} (no measurement noise),  Assumption \ref{assump:consistency} (consistency),  Assumption \ref{assump:interference} (no interference), Assumption \ref{assump:nde} (NDE), Assumption \ref{assump:nuc} (NUC) and Assumption \ref{assump:positivity_global_semi_offline_RL} (global positivity)
identified by
\begin{align}
\label{eq_identificiation_semi_offline_RL_1}
    J  
     = 
    \mathbb{E}_{p'}[C'_{(\pi_\alpha)}]
     = 
    \sum_{A',A,G_A(X_{(1)}),Y} 
    f_C(             
A', 
X',        
Y) 
q'(            
A',         
A, 
X, 
%G_A(X_{(1)}), 
Y) 
\end{align}
\noindent
with the distribution
\begin{align}
\label{eq_identificiation_semi_offline_RL_2}
q'(
A',A ,   
X, 
Y)
& = 
\prod_{t=1}^{T}  
\underbrace{\pi_{\text{id}}^t(
A^{t}| 
\underline{A}'^{t},
\underline{X}^{t-1}, 
\underline{A}^{t-1})}_{\text{distr. subject to constraints}}
\underbrace{
\pi_{\alpha}^t(A'^t|            
\underline{X}'^{t-1},            
\underline{A}'^{t-1})
}_{\text{target policy}}  
%\nonumber
% \\ & \cdot 
\prod_{t=0}^{T}
p(
X^{t}| 
\underline{X}^{t-1}, 
\underline{A}^{t}, Y) 
p(Y)
\end{align}

\noindent 
where 
\begin{flalign}
\label{eq_pi_id}
\pi_{id}^t(
A^t|  
\underline{A}'^t,&
\underline{X}^{t-1},
% G_{\underline{A}^{t-1}}(X_{(1)}),
\underline{A}^{t-1}) 
%\nonumber
=
% \\ & = 
\underbrace{\mathbb{I}(A^t \in 
\mathcal{
\tilde{A}}_\textit{adm}^t( 
\underline{X}^{t-1},
%G_{\underline{A}^{t-1}}(X_{(1)}),
\underline{A}^{t-1}, 
\underline{A}'^t
))}
_{\text{support restriction}}
f_{id}^t(
\underline{A}'^t,
\underline{X}^{t-1},
%G_{\underline{A}^{t-1}}(X_{(1)}),
\underline{A}^{t-1}
) 
\end{flalign}
for any function $f_{id}^t$
%(
%\underline{X}'^{t-1},
%\underline{A}'^t,
%\underline{X}^{t-1},
%\underline{A}^{t-1}
%) $ 
s.t. $\pi_{id}^t$ is a valid density.
\end{theorem}

\noindent 
Here, we used again $X \equiv G_A(X_{(1)})$ and $X' \equiv G_{A'}(X_{(1)})$ to simplify notation. In fact, since $A' \leq A$, we can also write $X' =  G_{A'}(X)$.
Additionally, we can define semi-offline RL versions of the Bellman equation:

\begin{theorem}
\label{theorem_Bellman_equation} 
(Bellman equation for semi-offline RL). 
The semi-offline RL view admits under 
Assumption \ref{assump:measurement_noise} (no measurement noise),  Assumption \ref{assump:consistency} (consistency),  Assumption \ref{assump:interference} (no interference), Assumption \ref{assump:nde} (NDE), Assumption \ref{assump:nuc} (NUC) and the local positivity assumption at datapoint $\underline{x}^{t-1},\underline{a}^{t-1},a'^t$ (from Definition \ref{def_local_positivity}), the following semi-offline RL version of the Bellman equation:   
\begin{align}
\label{eq_bellman_1}
Q_{\textit{Semi}}( &
\underline{A}'^{t},         
\underline{X}^{t-1},
%G_{\underline{A}^{t-1}}(X_{(1)}), 
\underline{A}^{t-1}, 
\Xi) 
 = 
 %\\ & =  
\sum_{\mathclap{X^{t}}}                
V_{\textit{Semi}}
(    
\underline{A}'^{t}, 
\underline{X}^{t},
% G_{\underline{A}^{t}}(X_{(1)}),
\underline{A}^{t-1}, 
A^{t}=a^t, 
\Xi) 
%\nonumber
%\\ & \cdot 
% g(X'^t|X^t, A'^t)
p(
X^{t}
% G_{A^{t}}(X_{(1)})
|                   
% G_{\underline{A}^{t-1}}(X_{(1)}),   
\underline{X}^{t-1},
\underline{A}^{t-1}, 
A^t = a^t, 
\Xi)
% \nonumber 
\\
& \text{ for any }
a^t 
\in \mathcal{A}^t_\textit{adm}
(          
\underline{X}^{t-1}, 
\underline{A}^{t-1},
A'^t )  
\nonumber 
\\
%\nonumber
 V_{\textit{Semi}}&
(
\underline{A}'^{t},
\underline{X}^{t},
%G_{\underline{A}^{t}}(X_{(1)}),  
\underline{A}^{t}, 
\Xi)  
=
% = \\
%& =   
\sum_{A'^{t+1}}  
Q_{\textit{Semi}}(
\underline{A}'^{t+1},    
\underline{X}^{t},
%G_{\underline{A}^{t}}(X_{(1)}), 
\underline{A}^{t}, 
\Xi
)
\pi_{\alpha}^{t+1}(
A'^{t+1}|
\underline{X}'^{t},
% G_{\underline{A}'^{t}}(X_{(1)}),
\underline{A}'^{t})
\label{eq_bellman_2}
\end{align}
with semi-offline RL versions of the state-action value function $Q_{\textit{Semi}}$ and state value function $V_{\textit{Semi}}$:
\begin{align*}
 Q_{\textit{Semi}}^t 
 & \equiv 
 Q_{\textit{Semi}}(
 %\underline{X}'^{t-1},
\underline{A}'^{t},
\underline{X}^{t-1},
%G_{\underline{A}^{t}}(X_{(1)}),
\underline{A}^{t-1}, 
\Xi) 
\equiv  
\mathbb{E}_{p'}[
C'_{(\overline{\pi}^{t+1}_\alpha)}|
%\underline{X}'^{t-1},
\underline{A}'^{t},
\underline{X}^{t-1},
%G_{\underline{A}^{t}}(X_{(1)}),
\underline{A}^{t-1}, 
\Xi]  
\\
V_{\textit{Semi}}^t
& \equiv 
V_{\textit{Semi}}(
%\underline{X}'^t,
\underline{A}'^{t},
\underline{X}^{t},
\underline{A}^{t}, 
\Xi) 
\equiv 
\mathbb{E}_{p'}[C'_{(\overline{\pi}^{t+1}_\alpha)}|
% \underline{X}'^t,
\underline{A}'^{t},
\underline{X}^t,
% G_{\underline{A}^{t}}(X_{(1)}),
\underline{A}^{t},
\Xi]
\end{align*}
%\end{subequations}
where $C'_{(\overline{\pi}^{t+1}_\alpha)}$ denotes the potential outcome of $C'$ under interventions from time step $t+1$ onwards. 
$\Xi \subseteq \{ Y, O\}$, with $O$ denoting all features that are always available, denotes an optional subset of additional variables that can be conditioned on. 
%$$\pi_\alpha$ applied only from time step $t+1$ onwards. 
Furthermore, $Q_{\textit{Semi}}^t$ and $V_{\textit{Semi}}^t$ are identified if the regional positivity assumption (from Definition \ref{def_regional_positivity}) holds at 
$
%\underline{X}'^{t-1},
\underline{X}^{t-1},
\underline{A}^{t-1}, 
\underline{A}'^{t}$ and $a^t \in \mathcal{\tilde{A}}_\textit{adm}^t(
%\underline{X}'^{t-1},
\underline{X}^{t-1},
\underline{A}^{t-1}, 
\underline{A}'^{t})$.
\end{theorem}

\noindent 
The proofs for Theorems \ref{theorem_identification_semi_offline_RL} 
and \ref{theorem_Bellman_equation} 
are shown in Appendix \ref{app_theorem_identification}.

The functions $Q_{\textit{Semi}}$ and $V_{\textit{Semi}}$ are very similar to their counterparts from the offline RL view ($Q_{\textit{Off}}$ and $V_{\textit{Off}}$), with two differences: i) they are learned from a curated dataset $\mathcal{D'}$ which arises from sampling $p'$%, instead of the "actual" dataset $\mathcal{D}$
; and ii) they contain the simulated acquisitions $% \underline{X}'^t,
\underline{A}'^{t}$, but also the real features and actions $\underline{X}^t,
\underline{A}^{t}$ which are needed to adjust for confounding of the blocking operation. 

%\noindent 
The identification steps so far have been very specific to knowledge about $\mathcal{\tilde{A}}_\textit{adm}^t$ that needs to be assessed from the data. 
%A larger set  $\tilde{A}^t$ allows one to use more datapoints. 
We now look more closely at a specific, stronger positivity assumption (where $\mathcal{A}_\textit{adm} = \mathcal{\tilde{A}}_\textit{adm}$), that allows the use of the maximum amount of datapoints and therefore leads to easier to use estimators.

\begin{definition}
\label{def_maximal_regional_positivity}
(Maximal regional positivity assumption for semi-offline RL).
We say that the maximal regional positivity assumption holds for a datapoint 
$%\underline{x}'^{t-1},
\underline{x}^{t-1},
\underline{a}^{t-1}, 
\underline{a}'^{t}$
if 
$\mathcal{\tilde{A}}_\textit{adm}^t(
%\underline{x}'^{t-1},
\underline{x}^{t-1},
\underline{a}^{t-1}, 
\underline{a}'^{t})
=  
\mathcal{A}^t_\textit{adm}(
\underline{x}^{t-1},
\underline{a}^{t-1},
a'^{t})$ 
and the maximal regional positivity assumption further holds for all 
$% \underline{x}'^{t},
\underline{x}^{t},
\underline{a}^{t}, 
\underline{a}'^{t+1}$ such that $a^{t} \in \mathcal{A}^t_\textit{adm}(
\underline{x}^{t-1},
\underline{a}^{t-1},
a'^{t})$ and the following holds for $ 
a'^{t+1}$, and $x^t$: 
\begin{align*}
p(x^{t}|
\underline{x}^{t-1},
\underline{a}^{t}) 
%g(x'^t|
%x^t,
%a'^{t}) 
\pi_{\alpha}(
a'^{t+1} | \underline{x}'^{t
},\underline{a}'^{t})   
> 0.
\end{align*}
\end{definition}

%\vspace{10pt}
\noindent
\setcounter{myassumption}{6}
\setcounter{subassumption}{3}
\begin{subassumption}[Maximal global positivity assumption for semi-offline RL]
\label{assump:max_positivity_global_semi_offline_RL}
We say that the maximal global positivity assumption holds if the maximal regional positivity assumption holds for all datapoints $ x^0, a'^1$ s.t. 
$p(x^0)\pi_\alpha(a'^{1}|
\underline{x}'^{0})>0$.
\end{subassumption}

\noindent 
The maximal regional positivity and maximal global positivity assumptions ensure that we can use all available data points where $A^t \geq A'^t$ without running into positivity problems in later time steps. This makes the identification and estimation  significantly  easier as shown next.

%\noindent 
We can now propose the following corollary of Theorem \ref{theorem_identification_semi_offline_RL}, which states identificiation under the maximal global positivity assumption. 

\begin{corollary}
\label{corollary_identification_semi_offline_RL} 
%\begin{subequations}
(Identification of $J$ for the semi-offline RL view under maximal global positivity).
The reformulated AFAPE problem of estimating $J$ under the semi-offline RL view  (Eq. \ref{eq:AFAPE_objective_semi_offline_RL}) is under
Assumption \ref{assump:measurement_noise} (no measurement noise),  Assumption \ref{assump:consistency} (consistency),  Assumption \ref{assump:interference} (no interference), Assumption \ref{assump:nde} (NDE), Assumption \ref{assump:nuc} (NUC) and Assumption \ref{assump:max_positivity_global_semi_offline_RL} (maximum global positivity)
identified
by Eqs. \ref{eq_identificiation_semi_offline_RL_1} and \ref{eq_identificiation_semi_offline_RL_2} where 
\begin{align*}
\pi_{id}^t(A^t|
% \underline{X}'^{t-1},
\underline{A}'^{t-1},
A'^t = a'^t,&
\underline{X}^{t-1},
\underline{A}^{t-1}
) 
 = % \\ & = 
\mathbb{I}(A^t \geq a'^t)
\pi_\beta^t(
A^t
|
\underline{X}^{t-1}, 
\underline{A}^{t-1}
)  
f_{id}^t(
% \underline{X}'^{t-1},
\underline{A}'^{t-1},
A'^t = a'^t,
\underline{X}^{t-1},
\underline{A}^{t-1}
) 
\end{align*}
for any function $f_{id}^t$
s.t. $\pi_{id}^t$ is a valid density. This holds in particular for the choice of a truncated $\pi_\beta$: 
\begin{align*}
\pi_{id}^t(A^t|
%\underline{X}'^{t-1},
\underline{A}'^{t-1},
A'^t = a'^t,
\underline{X}^{t-1},
\underline{A}^{t-1}
) 
& = 
\pi_\beta^t(A^t
|
A^t \geq a'^t, 
\underline{X}^{t-1},
\underline{A}^{t-1}
)
=\\& = 
\frac{\mathbb{I}(A^t \geq a'^t)
\pi_\beta^t(A^t
|
\underline{X}^{t-1},
\underline{A}^{t-1}
)}{
\pi_\beta^t(A^t\geq a'^t
|
\underline{X}^{t-1},
\underline{A}^{t-1}
)}.
\end{align*}
\end{corollary}

\noindent
The proof for Corollary \ref{corollary_identification_semi_offline_RL} 
is shown in Appendix \ref{app_proof_corollary_identification_semi_offline_RL}.

%\noindent 
Lastly, we also provide the following Remark that brings the factorization of the observational distribution $p'$ from Eq. \ref{eq:semi_offline_sampling_distribution} into a comparable form to the identifying distribution $q'$. 

\begin{myremark}[Factorization of the semi-offline sampling distribution]
\label{remark_semi_offline_RL_observational}
The semi-offline sampling distribution $p'$ can be alternatively written as: 
% data under the simulations given by $p'$ factorizes as:
\begin{flalign}
\label{eq:semi_offline_sampling_distribution_v2}
    p'(A', G_A(X_{(1)}), Y, A)
     & =
    \prod_{t=1}^T \pi_\textit{sim}^{\prime t}(A'^t| G_{\underline{A}^{\prime t-1}}(X_{(1)}), \underline{A}^{\prime t-1}, A^t) 
    p(A,G_A(X_{(1)}), Y)
    \nonumber
    \\ & 
    =
    \prod_{t=1}^{T}  
    \underbrace{
    \pi_{sim}'^t(
    A'^{t} | 
    \underline{X}'^{t-1},
    \underline{A}'^{t-1}, 
    A^{t})
    }_{\text{known simulation policy}}
    \underbrace{
    \pi_{\beta}(
    A^{t} | 
    \underline{X}^{t-1},
    \underline{A}^{t-1})
    }_{\text{retro. acquisition policy}}
    %\\ 
    %& \cdot  
    \prod_{t=0}^{T}
    p(X^t|                 
    \underline{X}^{t-1},        
    \underline{A}^{t}, Y ) 
    p(Y) 
    % \underbrace{g(X'^t|   
    % X^t,
    % A'^{t})}_{\text{feature revelations}}
\end{flalign} 
\end{myremark}

\subsection{Estimation}
\label{sec_semi_offline_RL_estimation}
We propose the following novel estimators for $J$ which arise from the semi-offline RL viewpoint. 
We differentiate between estimators derived under the global positivity assumption (Assumption \ref{assump:positivity_global_semi_offline_RL}) and under the (stronger) maximal global positivity assumption (Assumption \ref{assump:max_positivity_global_semi_offline_RL}).

\vspace{10pt}
\noindent 
\textit{1) Inverse probability weighting (IPW):} 

\noindent 
The target cost that is estimated by the semi-offline IPW estimator is
\begin{align} 
\label{eq:semi-off-ipw_pi_id}
    \hat{J}_{\textit{IPW-Semi}}
    & = 
    \hat{\mathbb{E}}_n\left[
    \hat{\mathbb{E}}_{n'}
    \left[ 
    \hat{\rho}_{
    \textit{Semi}}^T
    \text{ } 
    C'
    \big| 
    A, X,Y
    \right]
    \right]. 
\end{align} 
\noindent
The inner expectation, denoted by $\hat{\mathbb{E}}_{n'}[.]$, represents the empirical average over the simulated values of $A'$, which can involve many more samples compared to the outer expectation, taken over the observed data.
%where $\hat{\mathbb{E}}_{n'}[.]$, the inner expectation, denotes the empirical average over the simulated $A'$, whereas the outer expecation is over the observed data. % $\mathcal{D}'$.
The inverse probability weights are under the global positivity assumption:  
 \begin{align}
\label{eq_ipw_semi-offline_RL_weights_1}
    \hat{\rho}^T_{\textit{Semi}} = 
    \hat{\rho}^T_{\textit{Semi}}(\pi_{id}) & = 
    \prod_{t=1}^T 
    \frac{
    \pi_\alpha^t(A'^{t}| \underline{X}'^{t-1}, \underline{A}'^{t-1}) 
    }
    {
    \pi'^t_{\textit{sim}}(
    A'^{t}| 
    \underline{X}'^{t-1}, \underline{A}'^{t-1},
    A^t
    )
    }
    \frac{
    \pi_{id}^t(
    A^{t } | 
    % \underline{X}'^{t-1},
    \underline{A}'^{t},
    \underline{X}^{t-1},
    \underline{A}^{t-1})
    }
    {
    \hat{\pi}_\beta^t(
    A^t| 
    \underline{X}^{t-1},
    \underline{A}^{t-1}
    )
    }
\end{align}
or under the maximal global positivity assumption (by choosing 
\newline 
$\pi_{id}^t = \pi_\beta^t(A^t
|
A^t \geq a'^t, 
\underline{X}^{t-1},
\underline{A}^{t-1}
)$: 
\begin{align} 
\label{eq_ipw_semi-offline_RL_weights_2}
    %J_{\textit{IPW-Semi}} & = \hat{\mathbb{E}}_{n'}
    %\left[ \rho_{
    %\textit{Semi}}^T 
    %\text{ } 
    %C'\right], 
    %\\ 
    %\text{  where  } 
    \hat{\rho}^T_{\textit{Semi}} & = 
    \prod_{t=1}^T 
    \frac{
    \pi_\alpha^t(A'^{t}| \underline{X}'^{t-1}, \underline{A}'^{t-1}) 
    }
    {
    \pi'^t_{\textit{sim}}(
    A'^{t}| \underline{X}'^{t-1}, \underline{A}'^{t-1},
    A^t
    )
    }
    \frac{
    \mathbb{I}(A^t \geq a'^t)
    }
    {
    \hat{\pi}_\beta^t(
    A^t \geq a'^t| 
    \underline{X}^{t-1},
    \underline{A}^{t-1}
    )
    }.
    %\nonumber
\end{align}

\noindent 
The following remarks state that the IPW estimators from the offline RL and missing data viewpoints are special cases of the proposed estimator: 

\begin{myremark}
    [Offline RL IPW estimator as a special version of the semi-offline RL IPW estimator] 
    %$\hat{J}_{\textit{IPW-Off}}$ as a special version of $\hat{J}_{\textit{IPW-Semi}}$
    The IPW estimator from the offline RL view, $\hat{J}_{\textit{IPW-Off}}$, is, for the choice $\pi'_{\textit{sim}} = \pi'_\alpha$, equal to $\hat{J}_{\textit{IPW-Semi}}$ with: 
    \begin{align*}
    \pi_{id}^t(
    A^{t} | 
    % \underline{X}'^{t-1},
    \underline{A}'^{t-1},A'^t = a'^t,
    \underline{X}^{t-1},
    \underline{A}^{t-1})
    = \mathbb{I}(A^t = a'^t).
    \end{align*} 
\end{myremark}

\begin{myremark}
    [Missing data IPW estimator as a special version of the semi-offline RL IPW estimator]
    %$\hat{J}_{\textit{IPW-Miss}}$ as a special version of $\hat{J}_{\textit{IPW-Semi}}$ 
    The IPW estimator from the missing data view, $\hat{J}_{\textit{IPW-Miss}}$, is, for the choice $\pi'_{\textit{sim}} = \pi'_\alpha$, equal to $\hat{J}_{\textit{IPW-Semi}}$ with: 
    \begin{align*}
    \pi_{id}(
    A^{t} | 
    % \underline{X}'^{t-1},
    \underline{A}'^{t},
    \underline{X}^{t-1},
    \underline{A}^{t-1})
    = \mathbb{I}(A^t = \vec{1}).
    \end{align*}
\end{myremark}

\noindent 
The IPW estimator $\hat{J}_{\textit{IPW-Semi}}$ under the maximal global positivity assumption demonstrates the large benefits of the semi-offline RL view over both the offline RL and missing data views. Its second fraction shows that not only datapoints where $A^t = A'^t$ are used (i.e. have positive weight), as in the offline RL view, neither only datapoints where $A^t = \vec{1}$ are used, as in the missing data view, but all datapoints where $A^t \geq A'^t$ can be used. However, this benefit diminishes as the target policy becomes more "data-hungry"—that is, as it acquires more features.  In fact, there is no difference between all three estimators for an "acquire all features" policy: 
\begin{myremark}
    [Equality of IPW estimators %$\hat{J}_{\textit{IPW-Miss}} = \hat{J}_{\textit{IPW-Miss}} = \hat{J}_{\textit{IPW-Miss}}$ 
    for an "acquire all features" policy]  
    The IPW estimators from the missing data view, $\hat{J}_{\textit{IPW-Miss}}$, from the offline RL view $\hat{J}_{\textit{IPW-Miss}}$, and from the semi-offline RL view $\hat{J}_{\textit{IPW-Semi}}$ are identical if $\pi_\alpha^t(A^t = \vec{1}|\underline{X}^{t-1} = \underline{x}^{t-1}, \underline{A}^{t-1} =  \vec{1}) = 1$ $\forall t, \underline{x}^{t-1}$. 
\end{myremark}

\noindent 
We show in Appendix \ref{app_comparison_caniglia} that $\hat{J}_{\textit{IPW-Semi}}$  (under maximal global positivity) is equivalent in simple AFA settings to an adapted version of the IPW estimator by \cite{caniglia_emulating_2019}. %(which was derived under a different viewpoint for dynamic testing and treatment regimes). 
Our IPW estimators $\hat{J}_{\textit{IPW-Semi}}$ can, however, be applied in more general AFA settings. 

\vspace{10pt}
\noindent
\textit{2) Direct method (DM):} 

\nopagebreak

\noindent
The target cost that is estimated by the semi-offline DM estimator is 
\begin{align} \label{eq:semi-off-dm}
    \hat{J}_{\textit{DM-Semi}} = 
    \hat{\mathbb{E}}_{n}[\hat{V}_{\textit{Semi}}^0]
\end{align}

\noindent 
This estimator is based on learning a semi-offline RL version of the state-action value function $Q_{\textit{Semi}}$ using the semi-offline version of the Bellman equation (Eqs. \ref{eq_bellman_1} and \ref{eq_bellman_2}). Using  $Q_{\textit{Semi}}$, one can compute the state value function: $V_{\textit{Semi}}^t = \mathbb{E}_{\pi_{\alpha}}[Q_{\textit{Semi}}^{t+1}]$. 

% Note that one can choose, depending on the given positivity, which datap

The training process of $Q_\textit{Semi}$ can benefit from using the off-policy aspect of the proposed semi-offline sampling distribution $p'$ (i.e. from using a simulation policy $\pi'_{sim}$ that is different from $\pi'_{\alpha}$). This is because a deterministic AFA policy, for example,  will only generate one exact trajectory of simulated actions $A'$ and costs $C'$ per datapoint $X,Y,A$. A stochastic simulation policy $\pi'_{\textit{sim}}$ can instead be used to generate multiple such trajectories. % which can improve the learning of $Q_\textit{Semi}$. 
Usually, a parametric working model, for example a multi-layer perceptron (MLP), is assumed for the nuisance function $\hat{Q}_\textit{Semi}$. The estimation under such a working model will benefit from the additional datapoints generated under a policy $\pi'_{\textit{sim}} \neq \pi'_\alpha$.

\vspace{10pt}
\noindent 
\textit{3) Double reinforcement learning (DRL):} 

\noindent
The target cost that is estimated by the semi-offline DRL estimator is 
\begin{align}
    \hat{J}_{\textit{DRL-Semi}} = %\mathbb{E}_{\mathcal{D}'}
     \hat{\mathbb{E}}_{n}
     \left[
     \hat{\mathbb{E}}_{n'}
    \left[
    \hat{\rho}_\textit{Semi}^T C' + 
 \sum_{t=1}^{T} 
 \left(- \hat{\rho}_{\textit{Semi}}^{t}  
 \hat{Q}_\textit{Semi}^t
 %(\underline{X}'^{t-1},
 %\underline{A}'^{t}, 
 %\underline{X}^{t-1}, 
 %\underline{A}^{t-1}) 
 +
 \hat{\rho}_{\textit{Semi}}^{t-1}  
 \hat{V}_\textit{Semi}^{t-1}
 %(\underline{X}'^{t-1},
 %\underline{A}'^{t-1}, 
 %\underline{X}^{t-1}, 
 %\underline{A}^{t-1})
 \right) 
 \Big| 
 A,X,Y
 \right]
 \right].
\end{align}

\noindent 
which holds for both choices for $\rho_\textit{Semi}$, given that the respective positivity assumption holds. 
%\noindent 
Similar to the DRL estimator from the offline RL view, this approach combines the other two estimators  (Eqs. \ref{eq:semi-off-ipw_pi_id} and \ref{eq:semi-off-dm}). 
%It is efficient and doubly robust, in the sense that it is consistent if either the propensity score model $\pi_\beta$, or the state-action value function $Q_\textit{Semi}$ is correctly specified (as will be shown in the following theorem and theorem).

The following theorems state some notable facts about these estimators. 

\begin{theorem}
\label{theorem_consistency_IPW} (Consistency of $\hat{J}_{\textit{IPW-Semi}}$). 
The estimator $\hat{J}_{\textit{IPW-Semi}}$ is consistent if the propensity score model $\hat{\pi}_\beta$ is correctly specified.
\end{theorem}

\begin{proof}
We apply the standard inverse probability weighting approach $\mathbb{E}_{q'}[C']=\mathbb{E}_{p'}[\frac{q'}{p'}C']$ and use the factorizations for $q'$ and $p'$ from Eqs. \ref{eq_identificiation_semi_offline_RL_2} (Theorem \ref{theorem_identification_semi_offline_RL}) and 
\ref{eq:semi_offline_sampling_distribution_v2}
%\ref{eq_fact_obs_p_prime}
(Remark \ref{remark_semi_offline_RL_observational}), respectively, to obtain Eq. \ref{eq_ipw_semi-offline_RL_weights_1} for the weights. The weights from Eq. \ref{eq_ipw_semi-offline_RL_weights_2} arise from inserting the special choice for $\pi_{id}$.
%\hfill %\qedsymbol 
\end{proof}
\begin{theorem}
\label{theorem_consistency_DM} (Consistency of $\hat{J}_{\textit{DM-Semi}}$). 
The estimator $\hat{J}_{\textit{DM-Semi}}$ is consistent if the Q-function $\hat{Q}_\textit{Semi}$ is correctly specified.
\end{theorem}

\begin{proof}
    The proof of the consistency of $\hat{J}_\textit{DM-Semi}$ follows simply from the semi-offline Bellman equation (Theorem \ref{theorem_Bellman_equation}) %\ref{theorem_Bellman_equation}) 
    and the law of total expectation. 
%\hfill %\qedsymbol 
\end{proof}

\begin{theorem}
\label{theorem_double_robustness}
(Double robustness of $\hat{J}_{\textit{DRL-Semi}}$).
The estimator $\hat{J}_{\textit{DRL-Semi}}$ is %RAL (regular and asymptotically linear) and
doubly robust, in the sense that it is consistent if either the Q-function $\hat{Q}_\textit{Semi}$ or the propensity score model $\hat{\pi}_\beta$ is correctly specified. 
\end{theorem}

\noindent 
The proof is shown in Appendix \ref{app_theorem_double_robustness}. 
The estimator $\hat{J}_{\textit{DRL-Semi}}$ is a 1-step estimator based on an influence function derived for $J$ under $p'$. Therefore, the DRL estimator is regular and asymptotically linear (RAL). The influence function is given by the following theorem: 

\begin{theorem}
\label{theorem_efficient_if} 
(An influence function under the semi-offline RL view). An influence function of $J$ 
is:  
\begin{align}
     \varphi_\textit{Semi}
     %\Psi_{\textit{eff}} 
     = 
     -J
     +
    \mathbb{E}
    \left[
      \rho_\textit{Semi}^T C' + 
 \sum_{t=1}^{T} 
 \left(
 - \rho_\textit{Semi}^t  Q_\textit{Semi}^t
+
 \rho_\textit{Semi}^{t-1}  
 V_\textit{Semi}^{t-1}
 \right)
 \Big| 
 A,X,Y
 \right].
\end{align}
\end{theorem}

\noindent 
We will prove this theorem in the next section, addressing semiparametric estimation from all three views.

\begin{myremark}[Efficiency of $\varphi_\textit{Semi}$ as a function of $\Xi$] 
\label{remark_efficiency_as_a_function_of_xi}
In Theorem \ref{theorem_Bellman_equation}, we showed that the semi-offline RL version of the Bellman equation holds for any subset $\Xi \subseteq \{Y, O\}$, and the same applies to Theorem \ref{theorem_efficient_if}. The choice of $\Xi$ affects the efficiency of the DRL estimator: a bigger set $\Xi$ corresponds to a higher efficiency of the corresponding DRL estimator. However, the class of influence functions presented here does not include the efficient influence function. In fact, as we will discuss in the next section, the efficient influence function lacks a closed-form expression. 
\end{myremark} % The proof is provided in Appendix \ref{app_theorem_efficient_if}.

\noindent
In Appendix \ref{app_other_variants}, we extend the estimators discussed in this section to other settings. These include: i) the estimation of $J_a$, and ii) scenarios where a prediction $Y^{*t}$ is required at each time step $t$.

\section{Semiparametric Theory under NUC and NDE} \label{sec_semiparametrics}

\noindent 
\textit{Assumptions in this section: Assumption \ref{assump:nde} (NDE), Assumption \ref{assump:nuc} (NUC)}

\noindent
In this section, we explore semiparametric estimation approaches for $J$ under both NDE and NUC assumptions.  Readers not interested in the detailed semiparametric theory can skip this section.  We demonstrate that all three views can be unified within an established semiparametric theory framework for MAR missing data problems. Using this framework, we prove Theorem \ref{theorem_efficient_if}, which defines a class of influence functions derived under the semi-offline RL view. Although no closed-form efficient influence function exists in this setting, efficiency improvements from Tsiatis et al. \cite{tsiatis_semiparametric_2006} and Liu et al. \cite{liu_efficient_2021} can be applied. However, these methods pose significant challenges, including strong positivity assumptions, applicability to a limited set of problems, and implementation complexity.

To discuss semiparametric approaches to AFAPE, we first remind the reader of the semiparametric theory review in Appendix \ref{app_semiparametric_theory}. The following section draws extensively on the foundational framework of Tsiatis et al. \cite{tsiatis_semiparametric_2006} to establish essential context.

Let the observed data influence function be denoted as $\varphi \equiv \varphi(A, G_A(X_{(1)}), Y)$, and the observed data tangent space as $\Lambda$, with the corresponding nuisance tangent space denoted by $\Lambda_{\textit{nuis}}$. The full data influence function - an influence function given the counterfactual variables $X_{(1)}$ - is denoted $\varphi^F \equiv \varphi^F(X_{(1)},Y)$, with full data tangent space $\Lambda^F$ and nuisance tangent space $\Lambda^F_{\textit{nuis}}$.
The two key relationships between these spaces and the influence functions are as follows: 
\begin{itemize} 
    \item An observed data influence function must lie in the orthocomplement of the observed data nuisance tangent space: $\varphi \in \Lambda_{\textit{nuis}}^\perp$. 
    \item The observed data efficient influence function must be in the observed data tangent space: $\varphi_{\textit{eff}} \in \Lambda$. 
\end{itemize}

\noindent 
Known semiparametric theory for MAR missing data methods can be applied to the AFAPE problem under the NUC assumption. We begin by defining the space of full data influence functions. Since we assume no restrictions on the full data, the full data influence function is unique, efficient, and given by the online RL part of the missing data + online RL view:
\begin{flalign}
\label{eq:full_data_if}
    \varphi^F(X_{(1)},Y) 
    = 
    \underbrace{\mathbb{E}[C_{(\pi_{\alpha})}|X_{(1)},Y]}_{
    \text{online RL}
    }  - J
    = 
    \sum_{a \in \mathcal{A}}
     f_C(a, G_{a}(X_{(1)}), Y ) 
     \pi_{\alpha}(
    a|
    G_{a}(X_{(1)}))
    - J.
\end{flalign}

%\noindent 
%The Monte Carlo version of $\varphi^F$ is given by:
%\begin{flalign}
%\label{eq:full_data_if_mc}
%    \varphi^F_\textit{MC}(X_{(1)},Y)  = 
%    \sum_{i}^{n_{\textit{MC}}} 
%     f_C(Y, G_{a_i}(X_{(1)}), a_i )
%    % f_c(Y, G_{\underline{a}_i^{*{T}}}(X_{(1)}))
%    - J
%\end{flalign}
%\noindent 
%with $n_{\textit{MC}}$ samples $a_i \sim \prod_{t=1}^T \pi_\alpha(A^t \mid G_{\underline{A}^{t-1}}(X_{(1)}), \underline{A}^{t-1})$.

\noindent 
While the full data influence function is unique, the space of observed data influence functions is generally not. To find it, we first define the orthogonal complement of the nuisance tangent space, which can be expressed as (Theorem 8.3 from \cite{tsiatis_semiparametric_2006}):
\begin{flalign}
\label{eq:space_of_ifs}
    \Lambda_{\textit{nuis}}^\perp = 
    \biggl{\{}  
    \left[ 
    h^*(A,G_A(X_{(1)}), Y)
    \oplus 
    \Lambda_2
    \right]  
    - 
    \Pi \biggl{(}
    \left[ 
    h^*(A,G_A(X_{(1)}), Y)
    \oplus 
    \Lambda_2
    \right]  \Bigl{\vert} 
    %\underbrace{
    \Lambda_{\textit{nuis},\psi} 
    %}_{
    %\text{space of acquisition process}
    %}
    \biggl{)}
    \biggl{\}},
\end{flalign}
\noindent
where $h^*$ belongs to the inverse probability weighting (IPW) space $\Lambda_{\textit{IPW}}^*$, $\Lambda_2$ is the augmentation space, and $\Lambda_{\textit{nuis},\psi} = \Lambda_{\psi}$ is the nuisance tangent space, or equivalently the tangent space of the acquisition process. We explain these three spaces now in more detail. 
\newline 

\noindent 
\textbf{The inverse probability weighting space $\Lambda_\textit{IPW}^*$:}

The function $h^*(A,G_A(X_{(1)}),Y)$ in Eq. \ref{eq:space_of_ifs} can be any function in the IPW space $\Lambda_\textit{IPW}^*$ which is defined as: 
\begin{flalign*}
    %\label{eq:h_star_definition}
    \Lambda_\textit{IPW}^* 
    \equiv 
    \biggl{\{}
    & h^*(A,G_A(X_{(1)}), Y) \in \mathcal{H}: 
    \mathbb{E}[h^*(A,G_A(X_{(1)}), Y)|X_{(1)}, Y] 
    = 
    \varphi^{*F}(X_{(1)}, Y) 
    %\\ & 
    %h^*(A,G_A(X_{(1)}), Y) \in \mathcal{H}
    \biggl{\}}.
\end{flalign*} 
where $\mathcal{H}$ denotes the space of random functions with zero mean and finite variance and $\varphi^{*F}(X_{(1)}, Y)$ denotes an element of the orthocomp of the full data nuisance tangent space:  $\varphi^{*F} \in \Lambda_{\textit{nuis}}^\perp$. 

In fact, if we further restrict the IPW space such that we don't allow any element of the orthocomp of the full data nuisance tangent space $\varphi^{*F}(X_{(1)}, Y)$, but only the full data influence function $\varphi^{F}(X_{(1)}, Y)$, then we also obtain only observed data influence functions by Eq. \ref{eq:space_of_ifs} (Theorem 8.3 from \cite{tsiatis_semiparametric_2006}). In the following, we thus restrict the IPW space to:
\begin{flalign*}
    \label{eq:h_definition}
    \Lambda_\textit{IPW} 
    \equiv 
    \biggl{\{}
    & h(A,G_A(X_{(1)}), Y) 
    \in \mathcal{H}: 
    \mathbb{E}[h(A,G_A(X_{(1)}), Y)|X_{(1)}, Y] 
    = 
    \varphi^{F}(X_{(1)}, Y)
    %; 
    %\nonumber 
    %\\ & 
    %h(A,G_A(X_{(1)}), Y)
    % \in \mathcal{H}
    \biggl{\}}.
\end{flalign*}

\noindent 
In that case, the IPW space contains functions that, when taken the conditional expected value with respect to the full data, equal the full data influence function. The space is denoted as the IPW space, because IPW-based identifying functions can be chosen to construct elements in this space. In fact, as will be shown in this section, all the IPW estimators - from the offline RL, missing data and semioffline RL views - are applicable and form valid elements in $\Lambda_\textit{IPW}$. 
\newline 

\noindent 
\textbf{The augmentation space 
$\Lambda_2$:}

\noindent 
The augmentation space $\Lambda_{2}$ is defined as follows (Lemma 7.4 from \cite{tsiatis_semiparametric_2006}):
\begin{flalign}
    \Lambda_{2} = \biggl\{  
    b(A, G_A(X_{(1)}),Y)\in \mathcal{H}: 
    \mathbb{E}
    \left[
    b(A, G_A(X_{(1)}),Y)|X_{(1)}, Y
    \right] = 0
    %;
    %b(A, G_A(X_{(1)}),Y) \in %\mathcal{H}
    \biggl\}. 
\end{flalign}
\noindent 
This space contains functions that, when taken in conditional expectation with respect to the full data, equal zero. This provides intuition behind the decomposition of the space of observed data influence functions: the space consists of one function that, in conditional expectation, equals the full data influence function, plus all functions that become zero in conditional expectation. One must, however, still obtain the residual of the projection of these functions onto the nuisance tangent space of the acquisition process, $\Lambda_{\textit{nuis}, \psi}$, introduced below.
\newline 

\noindent 
\textbf{The nuisance tanget space of the acquisition process 
$\Lambda_{\textit{nuis}, \psi}$:}

\noindent 
The space $\Lambda_{\textit{nuis},\psi}$ corresponds to the observed data nuisance tangent space of the acquisition process and is a subspace of $\Lambda_2$ (Theorem 8.1 from \cite{tsiatis_semiparametric_2006}). In our AFA setting, future feature values do not influence past acquisition decisions, resulting in the following conditional independences:
$
A^t 
\indep
\overline{X}_{(1)}^t, Y 
| 
G_{\underline{A}^{t-1}}(\overline{X}_{(1)}^{t-1}), \underline{A}^{t-1}
$ which can be translated into tangent space restrictions such that: 
\begin{flalign}
    \label{eq:nuisance_tangent_space_psi}
    \Lambda_{\textit{nuis},\psi}
    = 
    \Lambda_{\textit{nuis},\psi}^1 
    \oplus 
    \Lambda_{\textit{nuis},\psi}^2 
    \oplus 
    ... 
    \oplus 
    \Lambda_{\textit{nuis},\psi}^T
\end{flalign}
\noindent where each subspace $\Lambda_{\textit{nuis},\psi}^t$ is defined as:
\begin{flalign*}
    \Lambda_{\textit{nuis},\psi}^t 
    \equiv 
    \biggl\{ 
    &
    \gamma^t(
    A^t,
    \underline{A}^{t-1},
    G_{\underline{A}^{t-1}}(X_{(1)}) 
    ) 
    \in 
    \mathcal{H}
    : 
    % \\ &
    \mathbb{E} 
    \left[ 
    \gamma^t(A^t,
    \underline{A}^{t-1},
    G_{\underline{A}^{t-1}}(X_{(1)}))
    | 
    \underline{A}^{t-1},
    G_{\underline{A}^{t-1}}(X_{(1)})
    \right] = 0
    %;\\
    %& 
    %\gamma^t(A^t,
    %\underline{A}^{t-1},
    %G_{\underline{A}^{t-1}}%(X_{(1)}) 
    %)  
    %\in 
    %\mathcal{H}
    \biggl\}.
\end{flalign*}
% Since the target parameter $J$ does not depend on the acquisition process (see Eq. \ref{eq:factorization_MAR} in our semiparametric theory review), it follows that $\Lambda_{\textit{nuis},\psi} = \Lambda_{\psi}$. 

\noindent As these subspaces are orthogonal, projections onto them are available in closed form, and are known to be (as derived in Appendix \ref{app_semiparametric_theory}): 
\begin{flalign*}
    \Pi(  
    [.] 
    | 
    \Lambda_{\textit{nuis},\psi}^t
    )
    = 
    \mathbb{E}
    \left[
    [.]
    | 
    A^t,
    \underline{A}^{t-1},
    G_{\underline{A}^{t-1}}(X_{(1)})
    \right]
    - 
    \mathbb{E}
    \left[
    [.]
    | 
    \underline{A}^{t-1},
    G_{\underline{A}^{t-1}}(X_{(1)})
    \right].
\end{flalign*}

\noindent 
Amongst the class of observed data influence functions, one may further be interested in finding the one with the smallest asymptotic variance, i.e. the efficient observed data
influence function $\varphi_{\textit{eff}}$. 
It can be found via the following projection of $h$ (Theorem 10.1 from \cite{tsiatis_semiparametric_2006}):
\begin{flalign*}
    \varphi_{\textit{eff}} = 
    h(A,G_A(X_{(1)}), Y) - 
    \Pi 
    \left(  
    h(A,G_A(X_{(1)}), Y)
    | 
    \Lambda_2
    \right).
\end{flalign*}
\noindent 
Hence, to construct an efficient influence function, one needs to find an element $h$, construct the space $\Lambda_2$, and project onto it.
We now embed the three viewpoints of this work—missing data view, offline RL view, and semi-offline RL view—into this framework. We begin with the traditional missing data approach based on complete cases. 

\subsection{Missing data view}
Now, we discuss the standard, traditional missing data approach to choosing an element $h$ of the IPW space $\Lambda_\textit{IPW}$ and constructing the augmentation space $\Lambda_2$. Traditional semiparametric estimators for missing data problems rely on the missing data positivity assumption (Assumption \ref{assump:positivity_missing_data}), which requires complete cases. The corresponding estimators are referred to as augmented inverse probability weighting complete case (AIPWCC) estimators.

These estimators are called complete case estimators because they choose the missing data IPW estimator for $h$ (see \cite{tsiatis_semiparametric_2006} for more details):
\begin{flalign}
\label{eq:h_missing_data}
    h_\textit{Miss}(A,G_A(X_{(1)}), Y) 
    %& = 
    %\rho_\textit{Miss} f_C(Y,G_A(X_{(1)}),A)  - J 
    % \\
    %&
    =    
    \rho_\textit{Miss} \mathbb{E}[C_{(\pi_\alpha)}|X_{(1)},Y] - J 
    \in \Lambda_\textit{IPW}.
\end{flalign}
\noindent 
Furthermore, they also require the missing data positivity assumption (Assumption \ref{assump:positivity_missing_data}) for the construction of $\Lambda_2$, given as:
\begin{flalign}
    \label{eq:lambda_2_missing_data}
    \Lambda_2  = 
    \Biggl{\{} & 
    \sum_{a \in \mathcal{A} \backslash \vec{1}}
    \left[
    \mathbb{I}(A=a)
    - 
    \prod_{t=1}^T
    \frac{
    \mathbb{I}(A^t=\vec{1})
    \pi_\beta^t(A^t=a^t
    |
    G_{\underline{a}^{t-1}}(X_{(1)}),\underline{a}^{t-1})) 
    }
    {
    \pi_\beta^t(A^t=\vec{1}
    |
    G_{\underline{A}^{t-1}}(X_{(1)}),\underline{A}^{t-1} = \vec{1}
    )}
    \right] 
    b_a( 
    G_A(X_{(1)}), 
    Y): 
    \nonumber
    \\ &
     b_a( 
    G_A(X_{(1)}), 
    Y) \in \mathcal{H}
    \Biggl{\}}.
\end{flalign}
\noindent 
For the interested reader, we show how both of these choices are derived under the missing data positivity assumption in Appendix \ref{app:derivation_missing_data_semiparametrics}.

In addition to the strong positivity requirements, a key challenge is that general projections onto 
$\Lambda_2$
are not available in closed form 
\cite{tsiatis_semiparametric_2006}. 
Alternatives to still perform such projections and achieve full efficiency include iterative numerical methods. However, these methods are difficult to implement and involve significant computational challenges, which have hindered their practical application \cite{tsiatis_semiparametric_2006}.
Therefore, such methods are beyond the scope of this work, but we direct interested readers to \cite{tsiatis_semiparametric_2006} for further details.

% An alternative to this traditional missing data approach is to start from the the offline RL view, as we will now explain.

\subsection{Offline RL view}

An alternative to applying traditional semiparametric theory for missing data problems is to start from the offline 
RL view.
In this approach, one projects the influence function $\varphi_{\textit{Off}}$—associated with the DRL estimator—onto the restricted tangent space under the NDE assumption, denoted by $\Lambda_\textit{NDE}$: 
\begin{flalign*}
    \varphi_{\textit{eff}} = \Pi \left( \varphi_{\textit{Off}}\big| \Lambda_\textit{NDE}\right) = \varphi_{\textit{Off}} - \Pi \left( \varphi_{\textit{Off}}\big| \Lambda_\textit{NDE}^\perp \right).
\end{flalign*}

\noindent 
This approach was first introduced by Liu et al. \cite{liu_efficient_2021} in the context of dynamic testing and treatment regimes, though it can be adapted to the AFAPE setting. However, their derivation was limited to a single acquirable feature per time step ($A^t \in \{0,1\}$).

The space $\Lambda_\textit{NDE}^\perp$ is, adapted to AFAPE, given as: 
\begin{flalign*}
    \Lambda_\textit{NDE}^\perp = \Lambda_{*}
- 
\Pi(
\Lambda_{*}
| 
\Lambda_{\psi})
\end{flalign*}
with 
\begin{flalign*}
     \Lambda_*  
     \equiv 
     \Bigg\{
     & b^{t}_{
     \underline{i}^{t-1}, 
     \overline{i}^{t+1}}( \underline{A}^{t-1}, G_{\underline{A}^{t-1}}(X_{(1)}), X_{(1)}^{\underline{i}^{t-1}}, X_{(1)}^{\overline{i}^{t+1}} ,  Y) 
     \left(
     \frac{A^t}{\pi^t_\beta}-1
     \right) 
     \prod_{t^{\prime}=1}^{t-1}
     \left(\frac{A^{t^{\prime}}}{\pi^{t^{\prime}}_\beta}\right)^{i^{t^{\prime}}} \prod_{t^{\prime}=t+1}^{T}
     \left(
     \frac{A^{t^{\prime}}}
     {\pi^{t^{\prime}}_\beta}
     \right)^{i^{t^{\prime}}}:
     \\  
     & b^{t}_{\underline{i}^{t-1}, \overline{i}^{t+1}}
     ( \underline{A}^{t-1}, G_{\underline{A}^{t-1}}(X_{(1)}), X_{(1)},  Y) 
     \in 
     \mathcal{H} 
     \Bigg\}
\end{flalign*} 
\noindent 
where $\underline{i}^{t-1}$
$\equiv \{ 
i^1, .. , i^{t-1}
\}$ 
and $\overline{i}^{t+1}  \equiv \{i^{t+1}, .. , i^{T}\}$ 
index subsets of $\{0,1\}^{t-1}$ and $\{0,1\}^{T-t}$ respectively and 
$
X_{(1)}^{\underline{i}^{t-1}}
\equiv 
\{ X_{(1)}^{t'} : 1 \leq t' \leq t-1, i^{t'} = 1
\}$
and 
$
X_{(1)}^{\overline{i}^{t+1}}
\equiv 
\{ X_{(1)}^{t'} : t+1 \leq t' \leq T, i^{t'} = 1\}
$. Furthermore, we denote $\pi^{t}_\beta \equiv \pi_\beta^t(A^t = 1| G_{\underline{A}^{t-1}}(X_{(1)}),\underline{A}^{t-1})$ for this one acquisition per time-point setting.

The space $\Lambda_*$ includes a term for each combination of observed features (indexed by the subsets $i$). This construction also relies on a stringent positivity assumption, which is even stronger than what is required for identification under the offline RL view (Assumption \ref{assump:positivity_offline_RL}). Specifically, it demands that $\pi^{t}_\beta > 0$ for all $t$, $G_{\underline{A}^{t-1}}(X_{(1)})$, and $\underline{A}^{t-1}$, irrespective of the target policy.
Furthermore, since the target parameter doesn't depend on the acquisition process, we have $\Lambda_\psi = \Lambda_{\textit{nuis},\psi}$. 

In the following lemmas, we demonstrate the equivalence between this offline RL approach and the semiparametric theory for MAR missing data problems:

\begin{lemma}
\label{lemma_h_offline_RL}
    (Relating the offline RL IPW estimator to  the IPW space).
    The functional $h_{\textit{Off}} \equiv \rho_{\textit{Off}} C-J$, based on the IPW estimator from the offline RL view, is a valid element of the IPW space: $h_{\textit{Off}} \in \Lambda_\textit{IPW}$.
\end{lemma}

\begin{lemma}
\label{lemma_liu_lambda_2}
    ($\Lambda_*$ is equal to the augmentation space).
    The augmentation space $\Lambda_{2}$ is 
    %under the offline RL positivity assumption and 
    %for the setting of one acquisition option per time-step ($A^t \in \{0,1\}$) 
    equal to $\Lambda_*$. 
\end{lemma}

\noindent Both lemmas are proven in Appendix \ref{app:derivation_offline_RL_semiparametrics}.

We now demonstrate that both approaches—whether derived from the offline RL or the missing data view—yield the same influence function (as expected):
\begin{flalign*}
    \varphi_{\textit{eff}} & = 
    \varphi_{\textit{Off}} 
    - 
    \Pi (\varphi_{\textit{Off}}
    | 
    \Lambda_\textit{NDE}^\perp)
    \\ 
    & 
    \overset{*_1}{=}
    h_{\textit{Off}} 
    - 
    \Pi (h_{\textit{Off}} | 
    \Lambda_{\textit{nuis},\psi})  
    - 
    \Pi( h_{\textit{Off}} - \Pi (h_{\textit{Off}} |
    \Lambda_{\textit{nuis},\psi}) |  \Lambda_{*}
    - 
    \Pi(
    \Lambda_{*}
    | 
    \Lambda_{\textit{nuis},\psi}) )
    \\ 
    & = 
    h_{\textit{Off}}  
    - 
    \Pi (h_{\textit{Off}}  | 
    \Lambda_{\textit{nuis},\psi})  
    - 
    \Pi( h_{\textit{Off}}   |
    \Lambda_2 
    - 
    \Pi \left( \Lambda_2 | 
    \Lambda_{\textit{nuis},\psi} 
    \right)) 
    %+ 
    %\underbrace{\Pi( 
    %\Pi (h_{\textit{Off}}  | 
    %\Lambda_{\textit{nuis},\psi}
    %)   |
    %\Lambda_2 
    %- 
    %\Pi \left( \Lambda_2 | 
    %\Lambda_{\textit{nuis},\psi} \right))}_{0}
    \\ 
    & \overset{*_2}{=} 
    h_{\textit{Off}}  
    - 
    \Pi( h_{\textit{Off}}   |
    \Lambda_2 
    ) 
    \end{flalign*}
where we use in $*1)$ that the influence function can be decomposed: $\varphi_\textit{Off} = h_{\textit{Off}} 
- 
\Pi (h_{\textit{Off}} | 
\Lambda_{\textit{nuis},\psi})$. In $*2)$, we used that $\Lambda_{\textit{nuis}, \psi} \subset \Lambda_2$.

This shows that a separate projection onto $\Lambda_{\textit{nuis},\psi}$ is unnecessary. However, as noted earlier, a projection onto $\Lambda_\textit{NDE}^\perp$ (or $\Lambda_2$) is not available in closed form \cite{liu_efficient_2021}, as previously discussed in the missing data approach. Liu et al. \cite{liu_efficient_2021} suggest instead constructing an arbitrarily large subspace $\Omega$ onto which a projection is feasible. In this case, it becomes helpful to first project onto $\Lambda_{\textit{nuis},\psi}$ (where closed-form projections are available), ensuring that the resulting functional remains a valid influence function, even when the second projection is onto the approximated space $\Omega$.

The resulting estimator by Liu et al. \cite{liu_efficient_2021} is termed "nearly efficient". It is, however, still difficult to implement and has so far only been tested only for %an overly simplistic
a one time point, one acquisition setting.

%$\Lambda_{\textit{nuis}}^\perp$:
%\begin{flalign}
%    \Lambda_{\textit{nuis}}^\perp 
%    = 
%    \Pi \left( \Lambda_{\textit{nuis},\eta}^\perp | 
%    \Lambda_{\textit{nuis},\psi}^\perp \right) = 
%    \Lambda_{\textit{nuis},\eta} - 
%    \Pi  \left( 
%    \Lambda_{\textit{nuis},\eta}^\perp 
%    | 
%    \Lambda_{\textit{nuis},\psi} \right)
%\end{flalign}
%where $\Pi  \left( 
%    \Lambda_{\textit{nuis},\eta}^\perp 
%    | 
%    \Lambda_{\textit{nuis},\psi} \right)$ denotes the orthogonal projection of $\Lambda_{\textit{nuis},\eta}^\perp$ onto $\Lambda_{\textit{nuis},\psi}$. 
%It can also be shown that  $\Lambda_\psi \subset \Lambda_2$ (Theorem 8.1 from \cite{tsiatis_semiparametric_2006}).

\subsection{Semi-offline RL view}

In this section, we explore how the semi-offline RL view integrates with semiparametric theory and derive the corresponding influence function for the semi-offline DRL estimator given in Theorem \ref{theorem_efficient_if}.

We begin by establishing that the IPW estimator derived from the semi-offline RL framework can be used to construct an element of the IPW space $\Lambda_\textit{IPW}$:
\begin{lemma}
    \label{lemma_h_semioffline_RL}
    (Relating the semi-offline RL IPW estimator to  the IPW space).
    $h_\textit{Semi} \equiv h_\textit{Semi}(A, G_A({X_{(1)}}),Y) = 
    \hat{\mathbb{E}}_{n'} [\rho_{\textit{Semi}}^T C'| A,G_A({X_{(1)}}),Y]  - J$ is an element of the IPW space $\Lambda_\textit{IPW}$. 
\end{lemma}

\noindent 
The proof of this lemma can be found in Appendix 
\ref{app:derivation_semioffline_RL_semiparametrics}. 

Since constructing the augmentation space $\Lambda_2$ often involves strong positivity assumptions and projections onto $\Lambda_2$ are typically not available in closed form, we propose an alternative approach. Specifically, we suggest projecting onto subspaces of $\Lambda_2$ where closed-form projections are feasible. One such subspace is $\Lambda_{2,\textit{Semi}}$, a tractable subspace of $\Lambda_2$, defined as:
\begin{flalign*}
    \Lambda_{2,\textit{Semi}}(\Xi) = \Lambda_{2,\textit{Semi}}^1(\Xi) \oplus \Lambda_{2,\textit{Semi}}^2(\Xi) \oplus ... \oplus \Lambda_{2,\textit{Semi}}^T(\Xi)
\end{flalign*}
with each
\begin{flalign*}
\Lambda_{2,\textit{Semi}}^t(\Xi)
     = 
    \biggl{\{} & 
    b^t(A^t, \underline{A}^{t-1}, G_{\underline{A}^{t-1}}(X_{(1)}), \Xi)
    \in \mathcal{H} : % \\ &  
    \mathbb{E}[b(A^t , \underline{A}^{t-1}, G_{\underline{A}^{t-1}}(X_{(1)}), \Xi)
    | 
    \underline{A}^{t-1}, 
    G_{\underline{A}^{t-1}}(X_{(1)}),
    \Xi 
    %X_{(1)}, Y
    ] = 0
    %; 
    %\\
    %& \quad \quad  \quad  
    %b^t(A^t, \underline{A}^{t-1}, G_{\underline{A}^{t-1}}(X_{(1)}), \Xi) \in \mathcal{H}
    \biggl{\}}.
\end{flalign*}
%Here, $\Xi$ represent the optional set of additional variables that preserve the orthogonal structure of $\Lambda_{2,\textit{Semi}}$.

%Here we let $\Xi$ be an optional additional set of variables that will not break the orthogonal structure of $\Lambda_{2,\textit{Semi}}$, can be chosen according to reasonable positivity assumptions and that will further allow closed form projections. In particular, $\Xi$ may be empty, contain the label $Y$ or contain other future feature values $\overline{X}_{(1)}^t$ that are always available\footnote{Actually, also future observed feature values $\overline{X}^t$ that are not part of the intervention (i.e. not intervened on by $\pi_\alpha$) and whose acquisition is not affected by other acquisitions, may also be included in $\Xi$}.  

%\noindent 
%In fact, this space differs only from  $\Lambda_{\textit{nuis},\psi}$ in that it allows $b^t$ to depend on $Y$. In fact, it would also be possible to allow $b^t$ to depend on any future values $\overline{X}_{(1)}^t$ if they are always observed. 
%In fact, $\Lambda_{2,\textit{Semi}}$ differs from $\Lambda_{\psi}$ only in that it allows $b^t$ to depend on $Y$.

\noindent 
Notably, we have $\Lambda_{\psi} \subseteq \Lambda_{2,\textit{Semi}}(\Xi) \subset \Lambda_{2}$ with equality  $\Lambda_{\psi} = \Lambda_{2,\textit{Semi}}$ if $\Xi = \emptyset$.

This yields the following class of influence functions, proving Theorem \ref{theorem_efficient_if}: 
\begin{flalign*}
    \varphi_{\textit{Semi}}&(A, G_A{(X_{(1)}}),Y; \Xi)
     = 
    h_\textit{Semi}
    -  \Pi ( h_\textit{Semi} | \Lambda_{2,\textit{Semi}}(\Xi) )
% 1
\\ 
& 
= 
    h_\textit{Semi}  
    - 
    \sum_{t=1}^T
    \mathbb{E}
    \left[ 
    h_\textit{Semi}
    \Big| 
    A^t,  
    \underline{A}^{t-1}, 
    G_{\underline{A}^{t-1}}(X_{(1)}), 
    \Xi
    \right]
    + 
    \sum_{t=1}^T
    \mathbb{E}
    \left[ 
    h_\textit{Semi}
    \Big|    
    \underline{A}^{t-1}, 
    G_{\underline{A}^{t-1}}(X_{(1)}), 
    \Xi
    \right]
%2
\\ 
& 
=
    \mathbb{E}
    \left[ 
    \rho_{\textit{Semi}}^T f_C(A', G_{A'}(X_{(1)}), Y)
    | Y, G_{A}(X_{(1)}),A \right]
%2.2
\\ & -
    \sum_{t=1}^T
    \mathbb{E}
    \left[ 
    \rho_{\textit{Semi}}^t Q_\textit{Semi}^t 
    \Big|   
    A^t, 
    \underline{A}^{t-1}, 
    G_{\underline{A}^{t-1}}(X_{(1)}), 
    \Xi
    \right]
%2.3
\\ & +
    \sum_{t=1}^T
    \mathbb{E}
    \left[ 
    \rho_{\textit{Semi}}^{t-1} V_\textit{Semi}^{t-1}  
    \Big|   
    \underline{A}^{t-1}, 
    G_{\underline{A}^{t-1}}(X_{(1)}), 
    \Xi
    \right]
  - J
%3
\\ 
& 
=  
    \mathbb{E}
    \left[ 
     \rho_{\textit{Semi}}^T f_C(A', G_{A'}(X_{(1)}), Y)
     - 
     \sum_{t=1}^T
     \rho_{\textit{Semi}}^t Q_\textit{Semi}^t 
     +
     \sum_{t=1}^T
     \rho_{\textit{Semi}}^{t-1} V_\textit{Semi}^{t-1} 
     \Big|   
    A, 
    G_{A}(X_{(1)}), 
    Y
    \right]
  - J.
\end{flalign*}
\noindent 
We provide full details of this derivation in Appendix \ref{app:semioffline_RL_projection}.

From the semiparametric viewpoint, it holds that a larger set $\Xi$ will increase efficiency of the corresponding influence function as mentioned in Remark \ref{remark_efficiency_as_a_function_of_xi}. This is the case, since employing a larger set $\Xi$ will result in a larger subspace $\Lambda_{2,\textit{Semi}}(\Xi)\subset \Lambda_2$ which in turn implies a higher efficiency.

\section{Experiments}
\label{sec_experiments}

We evaluate the different estimators on synthetic datasets where the missingness is artificially induced to allow the comparison with the ground truth.

%test settings where both the NUC and NDE assumption hold and provide additional experiments where either the NUC or NDE assumptions are violated. % in Appendix \ref{Appendix_experiments}.

\subsection{Experiment Design}

We perform 5 experiments with different violations of the identifying assumptions: 
\begin{itemize}
    \item \textbf{Experiment 1: }
    Assumptions \ref{assump:nde} (NDE),
    \ref{assump:nuc} (NUC),
    \ref{assump:positivity_offline_RL} (offline RL positivity), \ref{assump:positivity_missing_data} (missing data positivity), and \ref{assump:max_positivity_global_semi_offline_RL} (maximal global positivity) all hold.
    \item  \textbf{Experiment 2: }
    Assumption \ref{assump:positivity_missing_data} (missing data positivity) is violated 
    \item  \textbf{Experiment 3: }
    Assumption \ref{assump:positivity_offline_RL} (offline RL positivity) is violated for some agents.
    \item  \textbf{Experiment 4: }
     Assumptions \ref{assump:nuc} (NUC) is violated. 
     \item  \textbf{Experiment 5: }
     Assumptions \ref{assump:nde} (NDE) is violated.
\end{itemize}
\noindent 
We evaluate random AFA policies and a proximal policy optimization (PPO)  RL agent \cite{schulman_proximal_2017} as AFA agents and use impute-then-regress classifiers \cite{le_morvan_whats_2021} with unconditional mean imputation and a logistic regression classifier. Nuisance functions ($\hat{Q}_\textit{Semi}$ and $\hat{\pi}_\beta$) are fitted using multi-layer perceptrons and logistic regression models, respectively. The assumed logistic regression model for the propensity score correctly matches the ground truth. 
We compare the following estimators:

\begin{itemize}
    \setlength{\itemsep}{0pt}
    \item \textit{Imp-Mean:} Mean imputation (biased estimator)
    \item  \textit{Blocking:} Blocks the acquisitions of not available features, but offers no correction. This corresponds to the estimate 
    $\hat{\mathbb{E}}_{n'}[C']$ (with $\pi'_{sim} = \pi'_\alpha$) which is biased. 
    %with $\pi'_{sim} = \pi'_\alpha$ as a semi-offline sampling policy.
    \item \textit{CC:} Complete case analysis (only unbiased under MCAR). 
    \item \textit{IPW-Miss}/\textit{IPW-Miss-gt:} The IPW estimator from the missing data view. The weights were normalized to reduce the variance of the estimator. \textit{IPW-Miss-gt} uses the ground truth propensity score model $\pi_\beta$ instead of its estimate $\hat{\pi}_\beta$.
    %\item \textit{IPW-Miss-gt:} The IPW estimator from the missing data view. The weights were normalized to reduce the variance of the estimator. \textit{IPW-Miss-gt} uses the ground truth propensity score model $\pi_\beta$ instead of its estimate $\hat{\pi}_\beta$.
    %\item \textit{MI-Miss:}: Multiple Imputation estimator based on ... \{... [to do find a suitable MI method]}
    \item \textit{IPW-Off}/\textit{IPW-Off-gt:} The IPW estimator from the offline RL view with normalized weights and with and without the ground truth propensity score model.
    %\item \textit{IPW-Off-gt:} The IPW estimator from the offline RL view with normalized weights and with the ground truth propensity score model.
    \item \textit{IPW-Semi}/\textit{IPW-Semi-gt:} The IPW estimator (for the maximal global positivity assumption) from the semi-offline RL view with normalized weights and with and without the ground truth propensity score model.
    %\item \textit{IPW-Semi-gt:} The IPW estimator (for the maximal global positivity assumption) from the semi-offline RL view with normalized weights and with and theground truth propensity score model.
    \item \textit{DM-Semi:} The semi-offline RL version of the direct method.
    \item \textit{DRL-Semi}/\textit{DRL-Semi-gt:} The semi-offline RL version of the double reinforcement learning estimator with normalized weights under the maximal global positivity assumption, with and without the ground truth propensity score model.
    \item  \textit{Ground Truth:} In the experiments where the NDE assumptions hold, the agent is evaluated on the fully observed dataset. This corresponds to estimating $J$ using a Monte Carlo estimate, $\hat{\mathbb{E}}\left[ C_{(\pi_{\alpha})}|X_{(1)},Y \right]$, derived from samples of the ground truth data without any missingness (i.e., samples from $p(X_{(1)},Y)$). Conversely, in experiments where the NDE assumption is violated, the ground truth is obtained by running the agent in an environment that is continuously sampled from the true data-generating process.
\end{itemize}
%We fit the propensity score $p(\bar{R}|X)$ and the AFA mechanism propensity score $p(\bar{R}|X, \hat{R}, \bar{R}_s = \vec{1})$ using dense deep neural networks. 
Complete experiment details are given in Appendix \ref{Appendix_experiments}.

\subsection{Results}

%%%%%%%%%%%%%%%%%%%%%%%%%%%%%%%%%%%%%%%%%%%%%%%%%%%%%%%%%%%%%%%%%%%% convergence  %%%%%%%%%%%%%%%%%%%%%%%%%%%%%%
%%%%%%%%%%%%%%%%%%%%%%%%%%%%%%%%%%%%%%%%%%%%%%%%%%%%%%%%

Figure \ref{figure_ts_convergence} shows the convergence plots of sampling-based estimators for Experiment 1, highlighting the data efficiency of all estimators when the identifying assumptions hold. As expected, the blocking and complete case estimators are biased and do not converge to the true value of $J$. Among the unbiased IPW estimators, the semi-offline RL IPW estimator is the most efficient, achieving the fastest convergence for the 'Random 50\%' agent, which acquires each feature with a 50\% probability. However, this data efficiency disappears when evaluating the 'Fixed 100\%' agent, which acquires all features every time. As noted in Remark \ref{remark_semi_offline_RL_observational}, all three IPW estimators perform identically in this scenario, as reflected by their equal convergence speeds.

\begin{figure}[h]
\centering	
\vspace{-10pt}
\includegraphics[width= 0.8\textwidth]{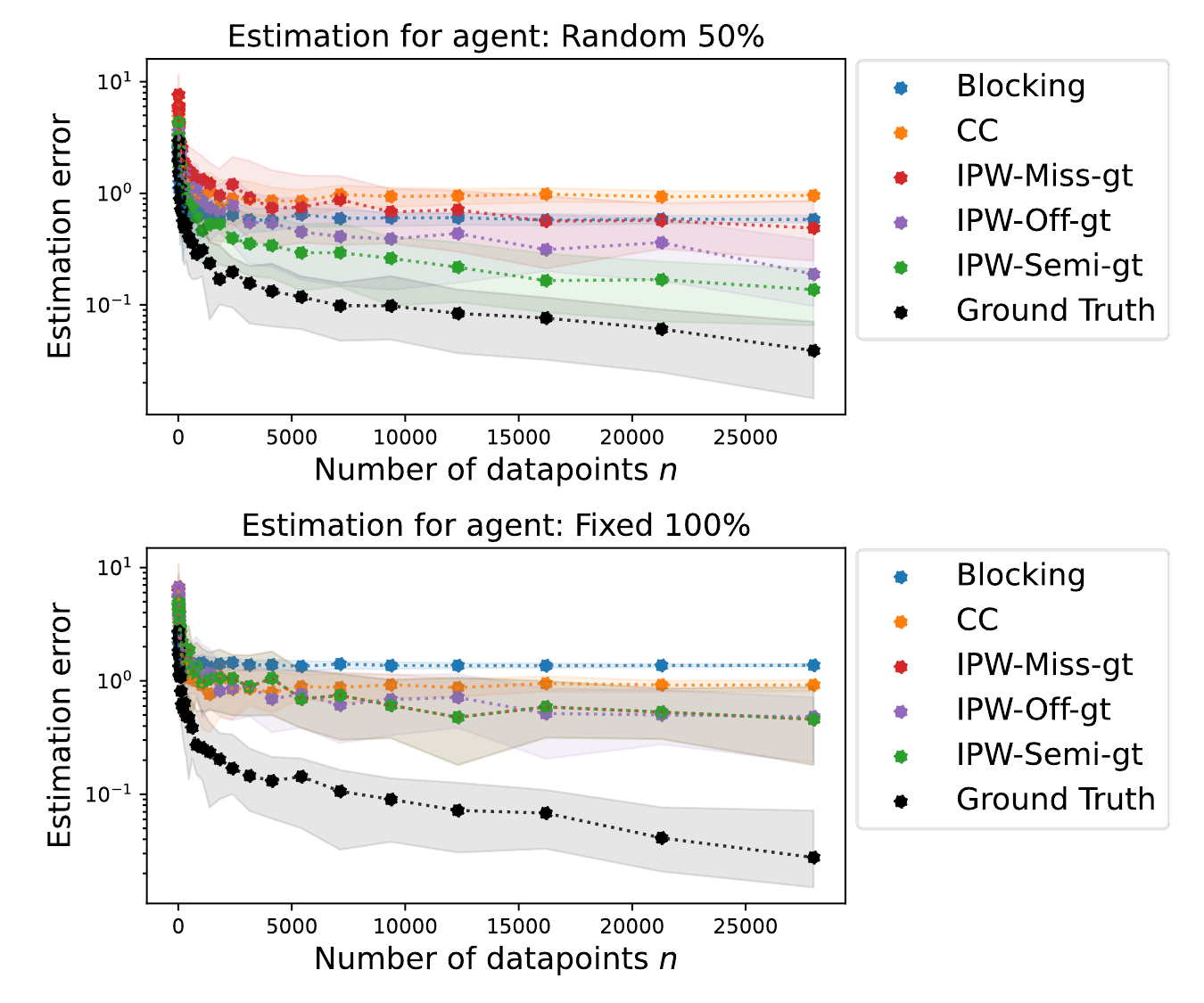}
\vspace{-1pt}
\caption{The plots depict convergence as a function of dataset size $n$ for sampling-based estimators in Experiment 1. Two agents are shown: one that acquires each costly feature with a probability of $50\%$ and another with $100\%$. Estimation error is measured as the absolute difference between the estimate and the ground truth, which is calculated on the full dataset ($n = 40,000$). The semi-offline RL IPW estimator converges the fastest for the 'Random 50\%' agent, while for the 'Fixed 100\%' agent, all IPW estimators perform identically, converging at a slower rate.
}
\vspace{-1 pt}
\label{figure_ts_convergence}
\end{figure}

%%%%%%%%%%%%%%%%%%%%%%%%%%%%%%%%%%%%%%%%%%%%%%%%%%%%%%%%%%%%%%%%%%%% general and doubly robust %%%%%%%%%%%%%%%%%
%%%%%%%%%%%%%%%%%%%%%%%%%%%%%%%%%%%%%%%%%%%%%%%%%%%%%%%%

\begin{figure}[ht]
\centering
\hspace*{-20pt} % Adjust the value as needed to move the figure left
\makebox[\textwidth][l]{ % Left-align the figure with the box
    \includegraphics[width=1.03\textwidth]{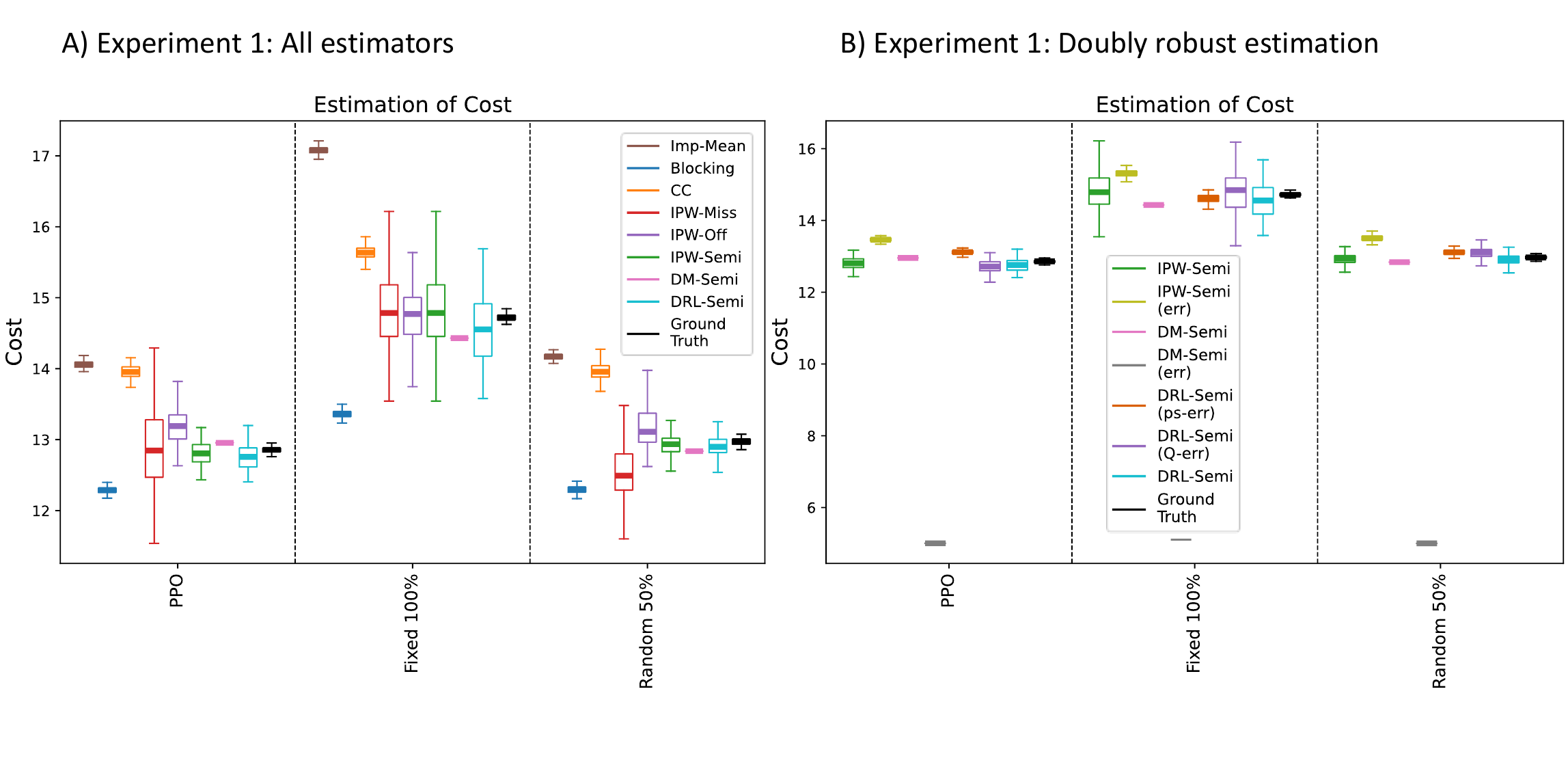}
}
\vspace{-1pt}
\caption{
A) General estimation results for Experiment 1. The \textit{Imp-Mean}, \textit{Blocking}, and \textit{CC} estimators show highly biased estimates, while the IPW, semi-offline DM, and DRL estimators align closely with the true target $J$.\\
B) Estimation results highlighting the double robustness property of the DRL estimator. The \textit{DRL-Semi} estimator continues to provide accurate estimates even when either the propensity score model $\hat{\pi}_\beta$ (\textit{DRL-Semi (ps-err)}) or the Q-model $\hat{Q}_\textit{Semi}$ (\textit{DRL-Semi (Q-err)}) is misspecified.
}
\vspace{-5pt}
\label{figure_ts_dr}
\end{figure}

Figure \ref{figure_ts_dr}A) displays the overall performance of various estimators in Experiment 1. Confidence intervals were computed using non-parametric bootstrap, excluding the retraining of nuisance functions due to high computational complexity. As a result, the confidence intervals are overly narrow, particularly for the semi-offline DM estimator. The experiment demonstrates that all semi-offline RL estimators accurately approximate the true target parameter $J$. However, the DM estimator may exhibit slight bias due to potential misspecification of the Q-function. In contrast, the biased mean imputation, blocking, and complete case analysis estimators fail to consistently estimate $J$.

Figure \ref{figure_ts_dr}B) illustrates the double robustness property of the semi-offline RL version of the DRL estimator. Even when one of the nuisance functions is misspecified, the DRL estimator still provides estimates that closely approximate the true value of $J$.

%%%%%%%%%%%%%%%%%%%%%%%%%%%%%%%%%%%%%%%%%%%%%%%%%%%%%%%%%%%%%%%%%%%% positivity violations %%%%%%%%%%%%%%%%%%%%%
%%%%%%%%%%%%%%%%%%%%%%%%%%%%%%%%%%%%%%%%%%%%%%%%%%%%%%%%

\begin{figure}[ht]
\centering
\hspace*{-20pt} % Adjust the value as needed to move the figure left
\makebox[\textwidth][l]{ % Left-align the figure with the box
    \includegraphics[width=1.03\textwidth]{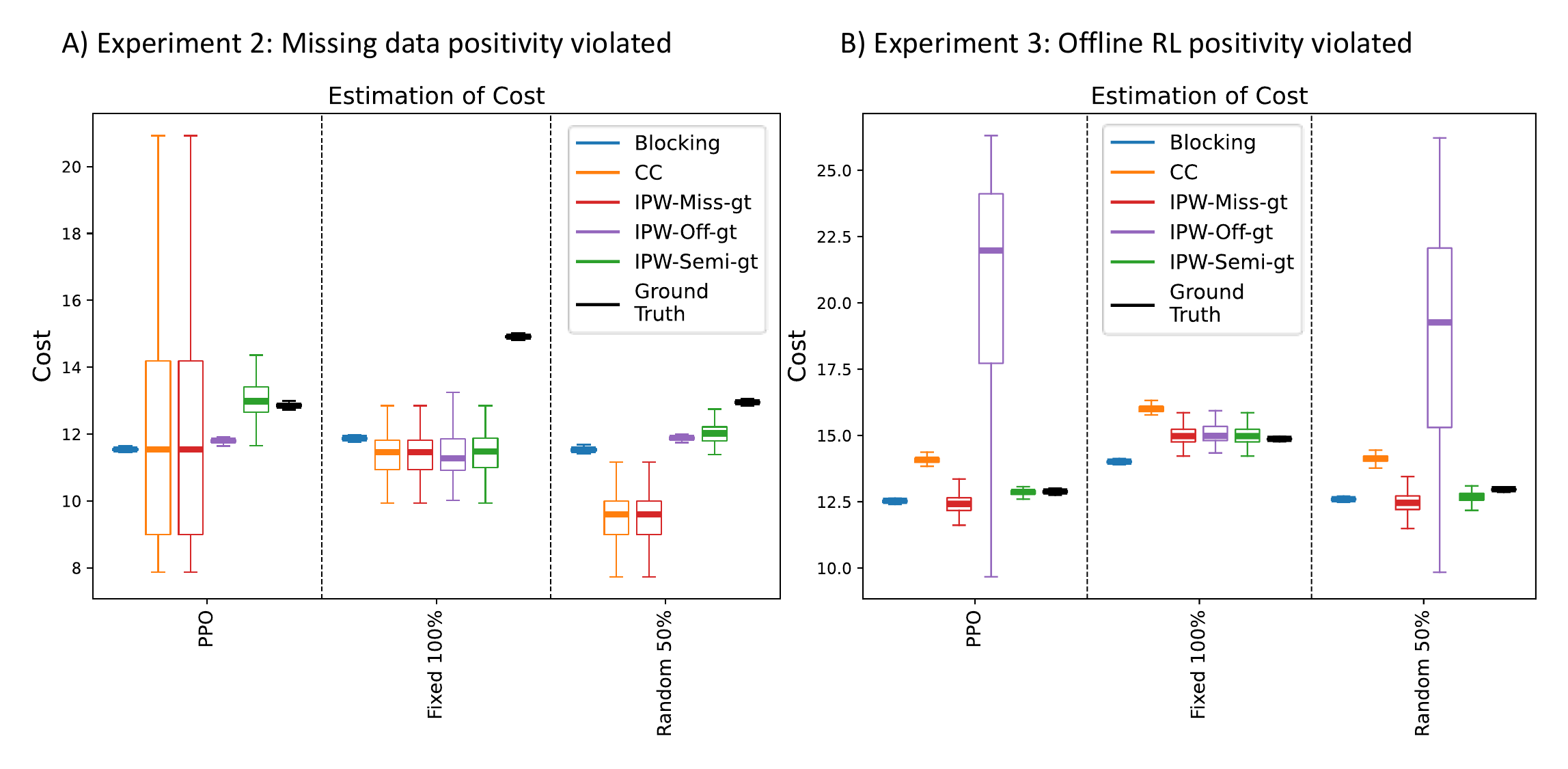}
}
\vspace{-1pt}
\caption{
A) Experiment 2: AFA setting with an extremely low fraction of complete cases (0.007\%), leading to violations of the missing data positivity assumption. For the 'Fixed 100\%' agent, the target $J$ is not identified from any view. However, the semi-offline IPW estimator is less impacted by the lack of complete cases when evaluating the other agents.
B) Experiment 3: AFA setting with positivity violations specific to the offline RL view. While other estimators continue to produce accurate estimates, the offline RL IPW estimator fails entirely for the PPO and 'Random 50\%' agents.}
\vspace{-1pt}
\label{figure_violations_positivity}
\end{figure}

Figures \ref{figure_violations_positivity}A) and B) underscore the importance of the positivity assumption. In Figure \ref{figure_violations_positivity}A), Experiment 2 shows the consequences of violating the missing data positivity assumption (Assumption \ref{assump:positivity_missing_data})—the fraction of complete cases is only 0.007\%. As a result, $J$ for the 'Fixed 100\%' agent cannot be identified from any of the views, and all IPW estimators fail to provide accurate estimates. The positivity assumption is also violated for the 'Random 50\%' agent, though the impact is less severe for the offline and semi-offline RL estimators.

In Figure \ref{figure_violations_positivity}B, Experiment 3 shows the failure of the positivity assumption required by the offline RL view (Assumption \ref{assump:positivity_offline_RL}) for two agents. In this experiment, $A_2 = \vec{1}$ for all data points, meaning trajectories where $A_{(\pi_\alpha)} = 0$ have no support. While the other IPW estimators still produce accurate estimates for $J$, the estimates from the offline RL IPW estimators fail completely.

%%%%%%%%%%%%%%%%%%%%%%%%%%%%%%%%%%%%%%%%%%%%%%%%%%%%%%%%%%%%%%%%%%%% NUC AND NDE %%%%%%%%%%%%%%%%%%%%%%%%%%%%%%%
%%%%%%%%%%%%%%%%%%%%%%%%%%%%%%%%%%%%%%%%%%%%%%%%%%%%%%%%

\begin{figure}[ht]
\centering
\hspace*{-20pt} % Adjust the value as needed to move the figure left
\makebox[\textwidth][l]{ % Left-align the figure with the box
    \includegraphics[width=1.03\textwidth]{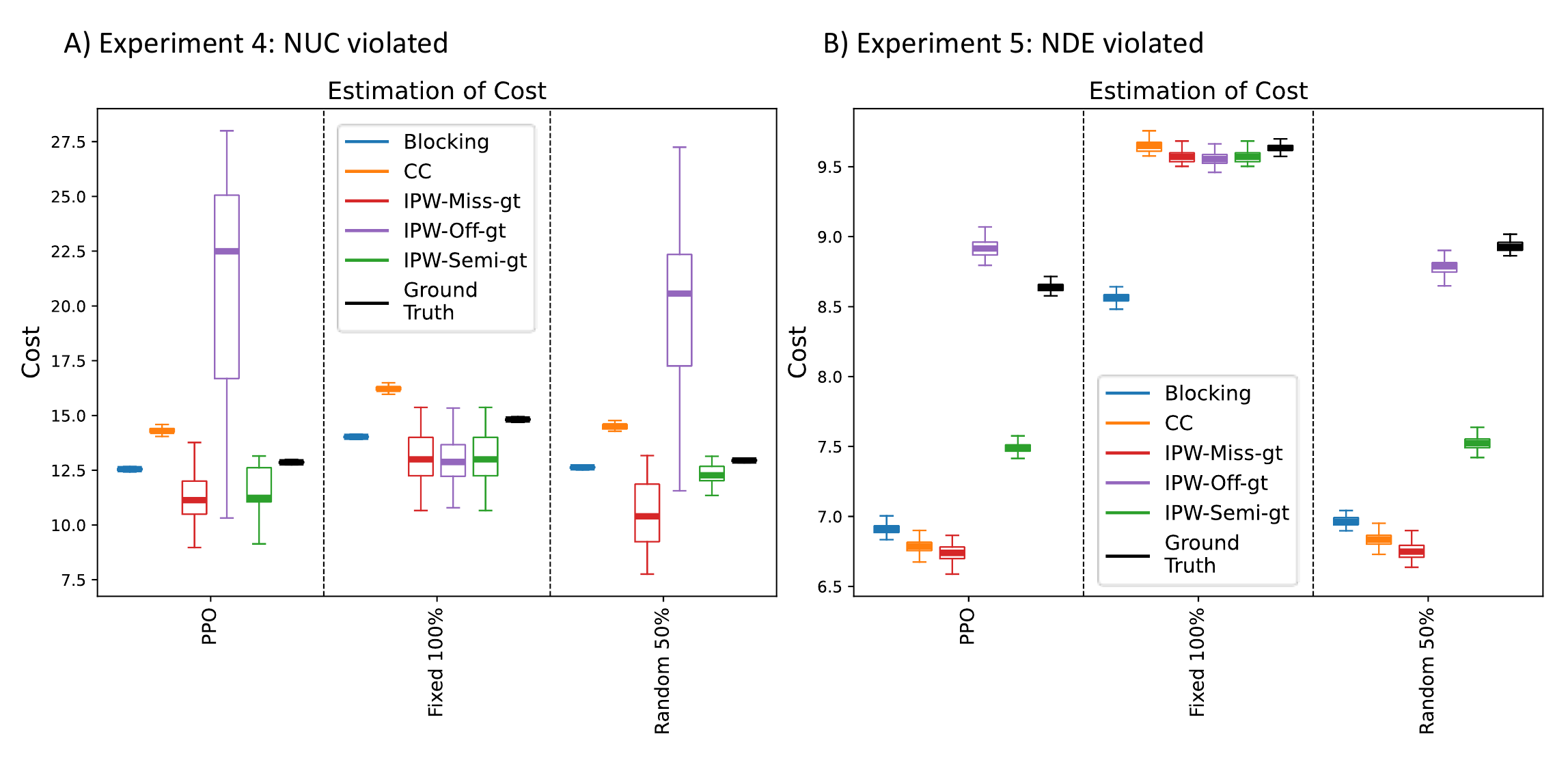}
}
\vspace{-1pt}
\caption{
A) Experiment 4: AFA setting with an MNAR acquisition process, resulting in violations of the NUC assumption. Despite this, the semi-offline RL IPW estimator still provides accurate estimates of the target $J$.
\\
B) Experiment 5: AFA setting with a violation of the NDE assumption. All estimators, except the offline RL IPW estimator, produce highly biased estimates. The offline RL IPW estimator also shows some slight deviation from the true $J$, too, potentially caused by minor positivity violations. 
}
\vspace{-1pt}
\label{figure_violations_nuc_nde}
\end{figure}

Finally, Figures \ref{figure_violations_nuc_nde}A) and B) explore the AFA setting when either the NUC or NDE assumptions are violated. In Figure \ref{figure_violations_nuc_nde}A, which depicts an MNAR scenario, all IPW estimators appear to perform well despite the violation of the NUC assumption. For the 'Fixed 100\%' agent, this is expected, as all estimators reduce to the missing data IPW estimator, which is identified. However, for the 'Random 50\%' agent, the AFAPE target $J$ is not identified from the offline or semi-offline RL views, though the estimation from the semi-offline RL IPW estimator remains relatively robust. It is worth cautioning that violating the NUC assumption may have more severe consequences in real-world applications.

Figure \ref{figure_violations_nuc_nde}B demonstrates the effect of violating the NDE assumption. In this case, only the offline RL view yields consistent estimators, as reflected in the experiment. All other estimators show significant deviations from the ground truth.
However, the offline RL IPW estimator shows slight biases too, possibly due to minor positivity violations.

\section{Discussion and Future Work}
\label{sec_discussion}

In this study, we explored the various aspects of solving the AFAPE problem. We acknowledge that there is no one-size-fits-all solution, as the choice of assumptions can vary across AFA settings. To facilitate this discussion, we propose a set of questions that data scientists should ask themselves when tackling the AFAPE problem before choosing a viewpoint and estimator.

\vspace{10pt}
\noindent 
\textit{1) What (conditional) independences hold in the data?}
The choice of conditional independence assumptions directly influences identifiability and the selection of optimal viewpoints and estimators. When both the NDE and NUC assumptions are violated, the target parameter becomes unidentifiable, making estimation infeasible. If only the NDE assumption fails, the offline RL view is still applicable, while violations of only the NUC assumption allow for the use of the missing data (+online RL) view. In each scenario, leveraging the absence of an edge in the causal graph can effectively eliminate estimation bias that would otherwise persist. 
When both assumptions are satisfied, one can select from the offline RL, missing data (+online RL), or the novel semi-offline RL view. 
In such cases, leveraging both the NDE and NUC assumptions through the semi-offline RL view can be advantageous, as it allows for relaxed positivity assumptions and a reduction in estimation variance.

\textit{Conclusion:} Under NUC, one can apply offline RL methods. Under NDE, one can apply missing data methods. Under both NUC and NDE, one can apply semi-offline RL methods.

\vspace{10pt}
\noindent
\textit{2) How much exploration was performed by the retrospective acquisition policy $\pi_\beta$?}
Positivity requirements, crucial for all viewpoints, demand certain action sequences to be present in the retrospective data. However, real-world data sets—especially in fields like medicine—often breach these assumptions due to the tendency of professionals to follow similar paths with minimal deviation. The semi-offline RL view, fortunately, imposes weaker positivity constraints than both the offline RL and missing data (+online RL) views. Nevertheless, for "data-hungry" AFA policies, the benefits may be less pronounced compared to the missing data view.

Additionally, the choice of the identifying policy $\pi_{\textit{id}}$ depends on which positivity assumptions hold in the data. We leave the adaptation of known positivity assessment methods \cite{petersen_diagnosing_2012} to the semi-offline RL setting as future work. 

\textit{Conclusion: } The semi-offline RL view requires significantly weaker positivity assumptions than the offline RL and missing data (+ online RL) viewpoints.

\vspace{10pt}
\noindent 
\textit{3) Can the nuisance models be correctly specified and trained?} 
The accuracy of different estimators hinges on the proper specification and training of nuisance functions. Despite the double robustness property of DRL estimators, achieving unbiased estimates still depends on how well these functions are modeled. Estimators like the multiple imputation (MI) method can outperform others in certain contexts, particularly when feature smoothness assumptions are reasonable and easily modeled over time. While machine learning techniques such as deep learning offer flexibility, they demand large data sets, which may not always be available.

\textit{Conclusion: } 
No single viewpoint or estimator is superior across all settings. The choice between MI and semi-offline RL estimators depends on prior knowledge and the feasibility of training the nuisance models.

\vspace{10pt}
\noindent 
\textit{4) Is the available data set size sufficient?} 
Efficient use of data is crucial for accurate estimation. Our experiments demonstrate that estimators based on the semi-offline RL view achieve greater data efficiency compared to both offline RL and missing data estimators. 

While our semiparametric analysis shows that the efficiency of all estimators can still be improved, a closed-form efficient influence function does not exist. Some computationally intensive methods, though complex to implement, may still enhance efficiency if the respective strong positivity assumptions hold \cite{tsiatis_semiparametric_2006, liu_efficient_2021}.

\textit{Conclusion: } Estimators derived from the semi-offline RL view demonstrate in experiments notably higher data efficiency compared to estimators from the offline RL and missing data (+ online RL) viewpoints. 

\vspace{10pt}
\noindent 
When the answers to the above questions are uncertain, it is advisable to use multiple views and estimators in tandem as part of a broader sensitivity analysis. This approach enhances confidence in the reliability and safety of AFA agents before deployment.

Our study assumes that feature values change over time, making the timing of measurements critical. In our companion paper \cite{von_kleist_evaluation_2023-1}, we address how a static feature assumption can be incorporated into the AFAPE problem, and we explore how the semi-offline RL and missing data views can be combined when the NUC assumption is violated (i.e., in MNAR scenarios).

Looking forward, we aim to tackle the AFA optimization problem outlined in Section \ref{sec_afa_optimization_problem}. Once the AFAPE problem is resolved and the estimation of the target parameter $J$ is successful, optimization can commence. This includes training new AFA agents and classifiers. A natural next step is adapting established DTR \cite{murphy_optimal_2003,bickel_optimal_2004} or offline RL methods—such as off-policy policy gradient methods, actor-critic techniques, and model-based RL approaches \cite{levine_offline_2020}—to the semi-offline RL framework for further development. These adapted methods could also integrate insights from the online RL literature, such as employing an adaptive exploration policy to enhance the sampling process.

\section{Conclusion}
\label{sec_conclusion}

We study the problem of active feature acquisition performance evaluation (AFAPE), which involves estimating the acquisition and misclassification costs that an AFA agent would generate after being deployed, using retrospective data. We demonstrate that, depending on the assumptions, one can apply different existing viewpoints to solve AFAPE. Under the no unobserved confounding (NUC) assumption, one can apply identification and estimation methods from the offline RL literature. Under the no direct effect (NDE) assumption, which assumes the underlying feature values are not affected by their measurement, one can instead apply missing data methods. 
For settings where both the NUC and the NDE assumptions hold, we propose a novel semi-offline RL viewpoint, which requires weaker positivity assumptions for identification.
Within the semi-offline RL viewpoint, we developed several novel estimators that correspond to semi-offline RL versions of the direct method (DM), inverse probability weighting (IPW), and double reinforcement learning (DRL).
Finally, we conducted synthetic data experiments to highlight the significance of utilizing proper unbiased estimators for AFAPE to ensure the reliability and safety of AFA systems.

\section*{Acknowledgments and Disclosure of Funding}
The present contribution is supported by the Helmholtz Association under the joint research school “HIDSS-006 - Munich School for Data Science @ Helmholtz, TUM \& LMU". Henrik von Kleist received a Carl-Duisberg Fellowship by the Bayer Foundation.

\newpage

%\acks{
%The present contribution is supported by the Helmholtz Association under the joint research school “HIDSS-006 - Munich School for Data Science @ Helmholtz, TUM \& LMU". Henrik von Kleist received a Carl-Duisberg Fellowship by the Bayer Foundation.}

\appendix

\section{Literature Review for Active Feature Acquisition (AFA)}
\label{app_AFA_methods}

In this appendix, we explain in more detail the difference between AFA and related fields and introduce some common approaches to training AFA agents from the literature. 

\subsection{Distinction between AFA and Related Fields}

AFA is different from active learning \cite{settles_active_2009}. In active learning, one assumes a classification task with a training dataset that contains many unlabeled data points. The active learning task is then to decide which label acquisitions will improve the training performance the most. Similar research also exists for the acquisition of features for optimal improvement of training. This task has been referred to as "active selection of classification features" \cite{kok_active_2021}, and unfortunately also as "active feature acquisition" \cite{huang_active_2018, beyer_active_2020}, but its objective differs fundamentally from ours. Huang et al. \cite{huang_active_2018} attempt to find out which missing values within the retrospective data set would improve training the most when retroactively acquired. In this paper, we are, however, interested which features, for a new data point, would improve the individual prediction for that data point the most.

\subsection{Approaches to Training AFA Agents}

The AFA setting is most generally described as a sequential decision process, which motivates the use of RL-based solutions. 
One variant, model-based RL focuses on learning a model for the state transitions. Under the NDE assumption, utilizing an imputation model to capture state transitions becomes feasible, exploiting the unique AFA structure for more straightforward learning
\cite{yoon_deep_2018, yin_reinforcement_2020, li_active_2021, li_dynamic_2021, ma_eddi_2019}.
During deployment, this imputation model can simulate potential outcomes of feature acquisitions, facilitating the derivation of optimal acquisition strategies.
Conversely, model-free RL methods do not require a state-transition function. One variant, Q-learning, involves estimating the expected cost of specific acquisition decisions \cite{chang_dynamic_2019, janisch_classification_2020, shim_joint_2018}. For instance, Shim et al. \cite{shim_joint_2018} illustrate the use of double Q-learning for the AFA agent, incorporating a deep neural network that shares network layers for the acquisition decision and classification tasks.

\section{Review of Semiparametric Theory}
\label{app_semiparametric_theory}

\noindent 
We give here a short review of basic concepts of semiparametric theory and some results for missing data problems. The review is based on work by Tsiatis et al. \cite{tsiatis_semiparametric_2006} which we recommend for more in-depth explanations. 
%For more in-depth explanations see for example  \cite{tsiatis_semiparametric_2006,bickel_efficient_1993,kennedy_semiparametric_2017}.

\subsection{General Semiparametric Theory}

Semi-parametric theory aims at finding data-efficient estimators for a target parameter $J=J(p)$ without imposing unnecessarily strict assumptions on $p$. In this review, we restrict ourselves to only scalar parameters $J$.
We let $p$ denote the distribution $p(Z)$ over a set of random variables $Z$ from which we have $n$ independent and identically distributed samples ($Z^1$,...,$Z^n$). It is possible in many cases to obtain estimators for $J$ that are consistent at a rate of $\sqrt{n}$
%,where 
%$n$ denotes the dataset size, 
without imposing many assumptions. The derivation of such estimators relies on influence functions which are discussed next. 
%It is possible in many cases to obtain $\sqrt{n}$-consistent estimators for $J$ without imposing many assumptions. This means it is often an easier statistical problem to estimate $J$ than to model all of $p$. 

\subsubsection{Influence functions and estimators}

A central element of semi-parametric theory are influence functions as they characterize asymptotically linear estimators in the following sense. An estimator $J_{est}$ is asymptotically linear and has an influence function $\varphi(Z) \equiv \varphi$ if it allows the following equality \cite{tsiatis_semiparametric_2006}:
\begin{align}
\label{eq_asymptotically_linear}
    J_\textit{est}(n) - J = \frac{1}{n} \sum_{i=1}^n \varphi(Z^i) + o_p(\frac{1}{\sqrt{n}})
\end{align}
where $\varphi \in \mathcal{H}$ and $\mathcal{H}$ represents the space of all random functions of zero mean and finite variance. 
The central limit theorem implies that $J_{est}$ is asymptotically normally distributed \cite{tsiatis_semiparametric_2006}: 
\begin{align*}
    % \frac{1}{\sqrt{n}}
    \sqrt{n}(J_\textit{est}(n) - J) \rightsquigarrow \mathcal{N}\left(0,\mathbb{E}[\varphi^2]\right)
\end{align*}
where  $\rightsquigarrow$ denotes convergence in distribution. The estimation error is thus asymptotically bounded by the variance of the influence function. The efficient influence function $\varphi_\textit{eff}$ is the one with the smallest asymptotic variance.

Many influence functions, such as the ones of the DRL estimators in this work, depend linearly on the target parameter $J$, such that $\varphi = f(Z) + J$ for some function $f$. In these cases, one can very easily derive a corresponding, so called "1-step", estimator by leveraging Eq. \ref{eq_asymptotically_linear} to obtain: 
\begin{align*}
  J_{est} \equiv - J + \frac{1}{n} \sum_{i=1}^n \varphi(Z_i) = -J + \frac{1}{n} \sum_{i=1}^n f(Z_i) + J = \frac{1}{n} \sum_{i=1}^n f(Z_i).
\end{align*}

\noindent 
% 
%Influence functions or more specifically the efficient influence function have to fulfill certain properties that can be used to derive them. 

\subsubsection{Deriving influence functions}
Deriving the space of influence functions or the efficient influence function for a new target parameter or new model restrictions can be complex.
There are, however, some known properties that influence functions in general, or the efficient influence function in particular, have to fulfill and these can be used for their derivation. We examine these now in more detail: 
\newline 

\noindent 
\textbf{1) An influence function must be in the orthocomp of the nuisance tangent space: $\varphi \in \Lambda_\textit{nuis}^\perp$}

\noindent 
To clarify this condition, we first separate the space of model parameters into $J$, the target parameter, and $\eta$, the nuisance parameters.
%An influence function must be an element of the space orthogonal to the nuisance tangent space.
We denote the nuisance tangent space as $\Lambda_\textit{nuis}$ and the space orthogonal to it, i.e. its orthocomplement (or orthocomp),  as $\Lambda_\textit{nuis}^\perp$. The nuisance tangent space can be seen as the collection of directions in which the nuisance parameters can vary without affecting the parameter of interest. 
It is defined as the mean square closure of parametric submodel nuisance tangent spaces, where a parametric submodel nuisance tangent space is the linear subspace spanned by the nuisance scores, which are defined as: %(Definition 1 from \cite{tsiatis_semiparametric_2006}): 
\begin{flalign*}
    S_\eta =\left.\frac{\partial \text{ log } p_Z(z,\eta, J)}
    {\partial \eta}
    \right|_{\eta=\eta_0}.
\end{flalign*}
Here $\eta_0$ denotes the true value of $\eta$.

% \noindent 
An influence function must be an element in $\Lambda_\textit{nuis}^\perp$, but must fulfill also the following normalization (Theorem 3.2 from \cite{tsiatis_semiparametric_2006}): 
\begin{flalign*}
    \mathbb{E}\left[ 
    \varphi(Z)
    S_J(Z)
    \right]
    = 
    1
\end{flalign*}
where $S_J(Z)$ denotes the scores with respect to the target parameter.

A nuisance function can thus be obtained by taking any nonzero element $h(Z) \in \mathcal{H}$, projecting it onto the orthocomp of the nuisance tangent space $\Lambda_\textit{nuis}^\perp$, and normalizing it. 
We denote the mentioned orthogonal projection by $\Pi$ such that: 
\begin{flalign*}
    \varphi^*(Z) = 
    \Pi (h(Z) 
    | \Lambda_{\textit{nuis}}^\perp) 
    = 
    h(Z)
    - 
    \Pi(h(Z) | 
    \Lambda_\textit{nuis}).
\end{flalign*}
\noindent 
where $\varphi^*(Z)$ represents an element of $\Lambda_{\textit{nuis}}^\perp$ which if normalized would be an influence function. 
\newline

\noindent 
\textbf{2) The efficient influence function must be in the tangent space: $\varphi \in \Lambda$}

\noindent 
The condition to obtain efficiency for an influence function is that we choose the influence function that is in the tangent space $\Lambda$. 
It can thus be obtained by projecting any influence function onto the tangent space: 
\begin{flalign*}
    \varphi_\textit{eff}(Z) = 
    \Pi(
    \varphi(Z) | \Lambda) 
    = 
    \varphi(Z)
    - 
    \Pi(
    \varphi(Z) | 
    \Lambda^\perp).
\end{flalign*}

\subsubsection{Constructing tangent spaces and projecting on them}
Here, we provide some useful further details about the construction of tangent spaces with specific restrictions and the projection of random variables on them. We start by decomposing the space of all zero mean, finite variance functions $\mathcal{H}$ into orthogonal subspaces for the case of a multivarate $Z$. We then examine how tangent space restrictions given by conditional independence assumptions can be incorporated. 
\newline 

\noindent 
\textbf{1) The decomposition of $\mathcal{H}$ for a multivariate variable $Z$}

\noindent 
Firstly, we look at a useful decomposition of the space of a multivariate variable $Z$ of dimensions $d$.
The space $\mathcal{H}$ of all functions $h(Z_1, Z_2, ..., Z_d)$ separates into orthogonal subspaces (Theorem 4.5 from \cite{tsiatis_semiparametric_2006}): 
\begin{flalign}
\label{eq:decomposition_H}
    \mathcal{H} = 
    \mathcal{H}_{Z_1} 
    \oplus 
    \mathcal{H}_{Z_2|Z_1} 
    \oplus 
    ... 
    \oplus 
    \mathcal{H}_{Z_d|Z_{d-1}, ..., Z_1} 
\end{flalign}
where 
$\mathcal{H}_{Z_i|Z_{i-1}, ..., Z_{1}}$ denotes the space spanned by the conditional scores: 
\begin{flalign*}
    \mathcal{H}_{
    Z_i|Z_{i-1}, ..., Z_1
    }  = 
    \biggl{\{}
    h(
    Z_i, Z_{i-1}, ..., Z_1
    ) \in 
    \mathcal{H}: 
    \mathbb{E}[h(
    Z_i, Z_{i-1}, ..., Z_1
    )|
    Z_{i-1}, ..., Z_1 
    ] = 0
    %;  
    %h 
    %\in 
    %\mathcal{H}
    \biggl{\}}.
\end{flalign*}
Fortunately, also the projection onto such a subspace is known and is given for an arbitrary element $h^*(Z) \in \mathcal{H}$ as: 
\begin{flalign*}
    \Pi 
    \left(
    h^*(Z) | 
    \mathcal{H}_{Z_i|Z_{i-1}, ..., Z_{1}}
    \right) 
    = 
    \mathbb{E}[h^*(Z)| Z_i, Z_{i-1}, ..., Z_1]
    - 
    \mathbb{E}[h^*(Z)| Z_{i-1}, ..., Z_1].
\end{flalign*}

\noindent 
\textbf{2) Tangent space restrictions under a conditional independence assumption}

\noindent 
It is also possible to define the tangent space under conditional independence restrictions and to project on it. 
In particular, let's revisit the decomposition of the tangent space from Eq. \ref{eq:decomposition_H}. Let's assume the following independence holds: 
$Z_i \indep Z_j | Z_{i-1}, ..., Z_{j+1}, Z_{j-1}, ..., Z_1$ for $i > j$. 
We want to find $\Lambda_r$, the tangent space restricted by the conditional independence, its orthocomp   $\Lambda_r^\perp$ and projections on it.  

The independence restriction only affects the space of the respective conditional scores such that: 
\begin{flalign*}
    \Lambda_r = 
    \mathcal{H}_{Z_1} 
    \oplus 
    \mathcal{H}_{Z_2|Z_1} 
    \oplus 
    ... 
    \oplus 
    \Lambda_{r,Z_i| Z_{i-1}, ..., Z_{j+1}, Z_{j-1}, ..., Z_1} 
    \oplus 
    ... 
    \oplus 
    \mathcal{H}_{Z_d|Z_{d-1}, ..., Z_1}. 
\end{flalign*}
and     $\Lambda_{r,Z_i| Z_{i-1}, ..., Z_{j+1}, Z_{j-1}, ..., Z_1} = \mathcal{H}_{Z_i| Z_{i-1}, ..., Z_{j+1}, Z_{j-1}, ..., Z_1}$. 
The projections onto $\Lambda_r$ are thus straight-forward:
\begin{flalign*}
    \Pi & 
    \left(
    h^*(Z) | 
    \Lambda_{r,Z_i| Z_{i-1}, ..., Z_{j+1}, Z_{j-1}, ..., Z_1} 
    \right) 
    = \Pi 
    \left(
    h^*(Z) | 
    \mathcal{H}_{Z_i| Z_{i-1}, ..., Z_{j+1}, Z_{j-1}, ..., Z_1} 
    \right)
    \\ & = 
    \mathbb{E}[h^*(Z)| Z_i,Z_{i-1}, ..., Z_{j+1}, Z_{j-1}, ..., Z_1]
    - 
    \mathbb{E}[h^*(Z)| Z_{i-1}, ..., Z_{j+1}, Z_{j-1}, ..., Z_1]
\end{flalign*}
Correspondingly, we can obtain the orthocomp as 
\begin{flalign*}
\Lambda^\perp_{r} 
 & =  \mathcal{H}_{Z_i|Z_{i-1}, ..., Z_1} - \Pi ( \mathcal{H}_{Z_i|Z_{i-1}, ..., Z_1} | \Lambda_{r,Z_i| Z_{i-1}, ..., Z_{j+1}, Z_{j-1}, ..., Z_1} ) 
 \\
 & = 
    \biggl{\{} 
    h(
    Z_i, Z_{i-1}, ..., Z_1
    ) 
    - 
    \mathbb{E}[h(
    Z_i, Z_{i-1}, ..., Z_1
    ) | 
    Z_i, Z_{i-1}, ..., Z_{j+1}, Z_{j-1}, ..., Z_1]
    \\ & 
    \quad  \quad 
    + \underbrace{
    \mathbb{E}[h(
    Z_i, Z_{i-1}, ..., Z_1
    ) | 
    Z_{i-1}, ..., Z_{j+1}, Z_{j-1}, ..., Z_1]
    }_{=0}
    : 
    \\ & 
    \quad  \quad 
    \mathbb{E}[h(
    Z_i, Z_{i-1}, ..., Z_1
    )|
    Z_{i-1}, ..., Z_1 
    ] = 0;  
    h 
    \in 
    \mathcal{H}
    \biggl{\}}
 \\
 & = 
    \biggl{\{} 
    h(
    Z_i, Z_{i-1}, ..., Z_1
    ) 
    - 
    \mathbb{E}[h(
    Z_i, Z_{i-1}, ..., Z_1
    ) | 
    Z_i, Z_{i-1}, ..., Z_{j+1}, Z_{j-1}, ..., Z_1]
    : 
    \\ & 
    \quad  \quad 
    \mathbb{E}[h(
    Z_i, Z_{i-1}, ..., Z_1
    )|
    Z_{i-1}, ..., Z_1 
    ] = 0;  
    h 
    \in 
    \mathcal{H}
    \biggl{\}}.
\end{flalign*}

\subsection{Semiparametric theory for missing data under the MAR assumption}
% \label{app:review_of_semiparametric_theory_for_missing_data}

Semiparametric methods have further been applied to missing data problems. As we take on a missing data view in this work, we now briefly introduce known results for such settings. We restrict our review to scenarios where the missingness process follows a missing-at-random (MAR) scenario. For more details, see also \cite{tsiatis_semiparametric_2006}. 

We distinguish now between observed data (with missingness) and full data (without missingness). 
Let the full data be denoted by $X_{(1)} \in \mathbb{R}^d$, the missingness indicators by $A \in \{0,1\}^d$ and the observed data by $X \equiv G_A(X_{(1)})$ such that $X_i = X_{(1),i}$ if $A_i = 1$ and $X_i = "?"$, otherwise. 
%In our paper, we will see this will be an influence function of the form $\varphi^F = f(X_{(1)}) - J$ where $J$ is the target of interest. \varphi^F(X_{(1)}) \equiv  \varphi^F$. 

Semiparametric theory methods for missing data aim at finding observed data (efficient) influence functions from full data influence functions. Throughout this review and the whole paper, we assume no restrictions on the full data $X_{(1)}$, meaning that there is only one full data influence function. There are, however, in general multiple corresponding observed data influence functions. 
In order to find these, the observed data nuisance tangent space needs to be constructed. 
This construction is simplified in MAR scenarios, as the likelihood factorizes into two separate terms related to the acquisition process and the observed data part of the likelihood \cite{tsiatis_semiparametric_2006}:
\begin{flalign}
\label{eq:factorization_MAR}
    p_{A,G_A(X_{(1)})}(A,G_A(X_{(1)}); \psi, J, \eta) = \underbrace{p_{A|G_A(X_{(1)}) }(A|G_A(X_{(1)}); \psi)}_{\text{acquisition process}} \sum_{G_A(X_{(1)})} 
    \underbrace{p_{X_{(1)}}(X_{(1)};J, \eta)}_{\text{process of full data}}
\end{flalign}
where we let $\psi$ denote the nuisance parameter of the acquisition process, and $\eta$ denote the nuisance parameter of the full data generating process.  

This factorization allows the following decomposition of the nuisance tangent space into orthogonal subspaces (Theorem 8.2 from \cite{tsiatis_semiparametric_2006}):
\begin{flalign}
    \Lambda_{\textit{nuis}} = \Lambda_{\textit{nuis},\psi} \oplus \Lambda_{\textit{nuis},\eta}
    = 
    \Lambda_{\psi} \oplus \Lambda_{\textit{nuis},\eta}.
\end{flalign}
\noindent 
Here, we used  $\Lambda_{\psi} =  \Lambda_{\textit{nuis},\psi}$ which holds since the target parameter is not part of this acquisition process. 

We continue with the discussion for the derivation of observed data influence functions in the main text in Section \ref{sec_semiparametrics}. 

\section{Glossary of Terms and Symbols}
\label{Appendix_glossary}
%\renewcommand{\arraystretch}{1}    

%\begin{center}
\begin{longtable}{@{} p{0.3\textwidth} p{0.67\textwidth} @{}} % p{5cm} p{11.16cm} }
%{@{\extracolsep{\fill}} l l }%p{10cm}}%[h!]
%\begin{tabularx}{\textwidth}{lX}
%\begin{tabular}{p{2.8cm}|p{0.6cm}}%{\textwidth}{lX}
\textbf{Term}           & \textbf{Description}         \\
\toprule
\textit{AFAPE} &  Active feature acquisition performance evaluation: The problem of estimating the counterfactual cost that would arise if an AFA agent was deployed.
\\
\midrule
\textit{NDE assumption} &  No direct effect assumption: States that the action of measuring a feature does not impact the values of any features or the label.
\\
\midrule
\textit{NUC assumption} &  No unobserved confounding assumption: States that acquisition decisions within the retrospective dataset were only based on measured feature values.
%\\
%\textit{RL} &  Reinforcement learning
\\
\midrule
\textit{Semi-offline RL} & Novel framework that allows an agent to interact with the environment (the online part), but forbids the exploration of certain actions (the offline part).
\\
\midrule
\textit{DTR} &  Dynamic treatment regimes 
\\
\midrule
\textit{G-formula} &  Identification formula from causal inference \cite{robins_new_1986}
\\
\midrule
\textit{Plug-in of the G-formula} &  Estimation formula from causal inference that replaces unknown densities in the G-formula with estimated versions\cite{robins_new_1986}.
\\
\midrule
\textit{IPW} &  Inverse probability weighting: Estimator that is also known as importance sampling or the Horvitz-Thompson estimator.
\\
\midrule
\textit{DM} &  Direct method: Estimator based on a Q-function.
\\
\midrule
\textit{DRL} &  Double reinforcement learning: Double robust estimator that uses IPW weights and a Q-function.
\\
\midrule
\textit{m-graph} &  Missing data graph: Graph to visualize assumptions in missing data problems.
\\
\midrule
\textit{MI} &  Multiple imputation: Estimator for missing data problems that is a special case of the plug-in of the G-formula.
\\
\midrule
\textit{influence function} & Function of mean zero and finite variance that is used to analyze the asymptotic properties of regular and asymptotically linear (RAL) estimators.
\\
\midrule
\textit{MCAR assumption} &  Missing-completely-at-random assumption: States that the reason for missingness of certain features does not depend on any feature values. \\
\midrule
\textit{MAR assumption} &  Missing-at-random assumption: States that the reason for missingness of certain features does only depend on observed feature values.\\
\midrule
\textit{MNAR assumption} &  Missing-not-at-random assumption: States that the reason for missingness of certain features may depend on feature values that are not observed.
\\
\midrule
\textit{nuisance function} &
Function that needs to be fitted from data in order to use a corresponding estimator, but which is not of primary interest itself. Examples are the propensity score model and the Q-function. \\
\midrule
\textit{local positivity assumption } &
Positivity assumption for semi-offline RL that ensures the simulation of a desired next action is possible from the retrospective dataset.\\
\midrule
\textit{regional positivity assumption } &
Positivity assumption for semi-offline RL that ensures the simulation of all future desired actions is possible from the retrospective dataset.\\
\midrule
\textit{global positivity assumption } &
Positivity assumption for semi-offline RL that ensures the simulation of all desired actions is possible from step 1 on.
\\
\midrule
\textit{maximal regional positivity assumption } &
Special, stronger version of the regional positivity assumption. 
\\
\midrule
\textit{maximal global positivity assumption } &
Special, stronger version of the global positivity assumption. 
\\
\midrule
\textit{tangent space } &
Space of scores (i.e. derivatives of the log-likelihood) 
\\
\midrule
\textit{nuisance tangent space } &
Space of nuisance scores (i.e. the scores with respect to the nuisance parameters) 
\\
\bottomrule
%\end{tabularx}
%\end{table}
%\begin{tabularx}{\textwidth}{lX}
\end{longtable}

\begin{longtable}{@{} p{0.25\textwidth} p{0.72\textwidth} @{}} 
%\begin{longtable}{ p{4cm} p{9.16cm} }
\textbf{Symbol}           
& \textbf{Description}         
\\
\toprule
\endhead
$t \in (0,...,T)$ &  Time \\
\midrule
$U^t$ &  Unobserved state variables at time t \\
\midrule
$d$ &  Number of features, i.e. dimension of $U^t$ \\
\midrule
$X^t = G_{A^t}(U^t)$ &  Observed feature values at time t (retrospective dataset) \\
\midrule
$A^t$ &   Acquisition action at time t (retrospective dataset) \\
\midrule
$\mathcal{A}$ &   Space of A\\
\midrule
$Y$ &  Label \\
\midrule
$Y^* = f_\textit{cl}(\underline{X}^T, \underline{A}^T)$ &  Predicted label based on classifier $f_\textit{cl}$ \\
\midrule
$C_a^t$ &  Acquisition cost for action $A^t$  \\
\midrule
$C_{mc} = f_C(Y^*, Y)$ &  Misclassification cost (if $Y$ and $Y^*$ differ)  \\
\midrule
$\pi_\beta$ &  Retrospective acquisition policy  \\
\midrule
$\pi_\alpha$ &  AFA policy \\
\midrule
$C_{mc,(\pi_\alpha)}$ &  Counterfactual misclassification cost had  $\pi_\alpha$ instead of $\pi_\beta$ been applied  \\
\midrule
$g(.)$
& known deterministic distribution\\
%\midrule
%$g(Y^*|\underline{X}^T,\underline{A}^T)$ &  Classifier %predicting $Y^*$ \\
\midrule
$J$ / $J_{mc}$ &  Expected misclassification cost under the AFA policy and classifier \\
\midrule
$J_a$ &  Expected acquisition cost under the AFA policy and classifier \\
\midrule
$\phi_1^*$ , $\phi_2^*$ &  Sets of parameters that parameterize the AFA policy and the classifier, respectively \\
\midrule
$q(.)$
& Counterfactual distribution\\
 \midrule
$Q^t_{\textit{Off}}$
& State-action value function from offline RL (at time t)
\\
\midrule
$V^t_{\textit{Off}}$
& State value function from offline RL (at time t)
\\
\midrule
$\pi'$ &  Blocked policy 
\\
\midrule
$\pi'_{\textit{sim}}$ &  (Blocked) simulation policy \\
\midrule
$p'(.)$
& Simulated distribution\\
\midrule
$C',Y'^*,X',A'$
& Simulated cost, predicted label, features and actions\\
 \midrule
$\mathcal{D}$
& Retrospective dataset \\
\midrule
$\mathcal{D}'$
& Simulated dataset \\
\midrule
$\mathcal{A}_\textit{adm}$
& Local admissible set \\
\midrule
$\mathcal{\tilde{A}}_\textit{adm}$
& Regional admissible set \\
\midrule
$\pi_{\text{id}}$
& Distribution for $A$ that allows identification of $J$ under the semi-offline RL view (subject to support restrictions)  \\
\midrule
$Q^t_{\textit{Semi}}$
& State-action value function from semi-offline RL (at time t)
 \\
\midrule
$\Xi \subseteq \{O,Y\}$
& Arbitrary subset of the always observed features $O \subset X_{(1)}$ and the label $Y$ \\
 \midrule
$q'(.)$
& Counterfactual simulated distribution\\
\midrule
$\varphi$
& Influence function\\
\midrule
$\Lambda$
& (Observed data) tangent space\\
\midrule
$\Lambda^\perp$
& Orthocomplement of the (observed data) tangent space\\
\midrule
$\Pi([.]|\Lambda)$
& (Orthogonal) projection onto the tangent space\\
\midrule
$\Lambda_\textit{nuis}$
& (Observed data) nuisance tangent space\\
\midrule
$\Lambda^F$
& Full data tangent space\\
\midrule
$\Lambda_\textit{nuis}^F$
& Full data nuisance tangent space\\
\midrule
$\Lambda_{\textit{nuis},\psi} = \Lambda_{\psi}$
& (Nuisance) tangent space of the acquisition process\\
\midrule
$\Lambda_{\textit{nuis},\eta}$
& Nuisance tangent space of the observed part of the full data process\\
\midrule
$\Lambda_{\textit{IPW}}$
& IPW space\\
\midrule
$\Lambda_{2}$
& Augmentation space\\
\midrule
$\Lambda_{2,\textit{Semi}}(\Xi)$
& Subspace of $\Lambda_{2}$ proposed for projection onto under the semi-offline RL view \\
\bottomrule
\end{longtable}
%\end{center}

\section{Identification of the Block-conditional Model}
\label{app_mnar}

In this appendix, we demonstrate how identification of $p(X_{(1)},Y)$ can be achieved when the NUC assumption (Assumption \ref{assump:nuc}) is violated. This corresponds to the block-conditional model \cite{zhou_block-conditional_2010}. 
We show that the propensity score model $p(A=\vec{1}|X_{(1)},Y)$ is identified, which in turn results in identification of $p(X_{(1)},Y)$, as $p(X_{(1)},Y) = \frac{p(X_{(1)},Y,A=\vec{1})}{p(A=\vec{1}|X_{(1)},Y)}$. 

\noindent 
\textit{Identification of $p(A=\vec{1}|X_{(1)},Y)$:}

The propensity score is identified by  
\begin{align*}
    p(   A =\vec{1} |X_{(1)},Y)  
    & =  
    \prod_{t=1}^T p( A^t =\vec{1}  |X_{(1)} , Y, \underline{A}^{t-1} =\vec{1}) \nonumber 
    \\ & \overset{*_1}{=}   
    \prod_{t=1}^T \pi_\beta^t( A^t =\vec{1}  |  \underline{X}_{(1)}^{t-1}, \underline{A}^{t-1} =\vec{1}) \nonumber 
    \\ & \overset{*_2}{=}   
    \prod_{t=1}^T \pi_\beta^t( A^t =\vec{1} |  \underline{X}^{t-1}, \underline{A}^{t-1} =\vec{1}) \nonumber 
\end{align*}

\noindent 
where we used in $*1)$ the fact that future feature values do not affect current acquisition decisions: $A^t \indep \overline{X}_{(1)}^t, Y | \underline{X}_{(1)}^{t-1}, \underline{A}^{t-1}$. We further use in $*2)$ that counterfactual feature values $\underline{X}_{(1)}^{t-1}$ are equal to $\underline{X}^{t-1}$ if $\underline{A}^{t-1} =\vec{1}$. The last expression is a function of only observed variables and is thus identified.

\section{Multiple Imputation (MI) for the AFAPE Problem}
\label{app_missing_data_estimators}

In this appendix, we aim to delve deeper into the multiple imputation (MI) estimator in the AFAPE context and highlight advantages as well as some common pitfalls associated with using MI approaches in AFA.

Let us begin by emphasizing a significant advantage of the MI estimator compared to other estimators discussed in this paper. It offers an elegant solution to the temporal coarsening problem. In time-series settings, where fixed time intervals are assumed ($t \in \{0,1,...,T\}$), employing a very fine resolution of time steps would inevitably result in a considerable increase in missingness, thereby making the AFAPE problem more challenging. The MI estimator can typically overcome this issue by assuming an often justifiable temporal smoothness of the feature distributions. %, which is often a justifiable approach. %For instance, one can use a Gaussian Process for imputation 

However, there are  drawbacks to MI. MI requires modeling joint distributions, which is a complex task in practice, particularly in high-dimensional settings and when dealing with complex missingness patterns. For instance, the multiple imputation by chained equations (MICE) method \cite{van_buuren_multiple_2007} necessitates fitting $d$ conditional densities for $d$ partially observed features in static settings. In comparison, IPW only requires the specification of the propensity score, which is often more feasible. This effect is especially drastic for high-dimensional features such as images, which necessitate modeling for each pixel, when using multiple imputation, but only the modeling of one joint missingness indicator when using IPW. 

Furthermore, the MI estimator implies imputation of the missing features $X_{m}$ by conditioning on the observed features $X_{o}$ and the label $Y$ (i.e., estimating $\hat{p}(X_{m}|X_{o},Y)$). This introduces the risk of data leakage, as the imputed features may carry predictive information not because of the true data generation mechanism but due to the imputation itself, resulting in potentially overoptimistic estimation of prediction performance. A common alternative, frequently employed in machine learning, is to impute the data without conditioning on $Y$. However, this assumption implies that a missing feature $X_{(1),i} \in X_\textit{m}$ is conditionally independent of the label given the observed features ($X_{(1),i} \indep Y | X_o$). Determining marginal predictive value of a feature for predicting $Y$, is however, the whole task of AFA, which renders this approach impractical.

Conditional mean imputation represents a simplified imputation approach that reduces the complexity of modeling. It has been applied in AFA settings \cite{an_reinforcement_2022, erion_coai_2021, janisch_classification_2020}. In this approach, missing values are imputed using a conditional mean model for $\hat{\mathbb{E}}[X_{m}|X_o]$ (or $\hat{\mathbb{E}}[X_{m}|X_o, Y]$). Therefore, conditional mean imputation assumes:
\begin{align*}
    J_{\textit{MI-Miss}} & = \sum_{X_{m},X_{o},Y} \mathbb{E}[C_{(\pi_\alpha)}|X_{m},X_{o},Y]  
    \hat{p}(X_{m}|X_{o},Y)
    p(X_{o},Y) 
    \\ & \approx 
    \sum_{X_{o},Y} \mathbb{E}[C_{(\pi_\alpha)}|\hat{\mathbb{E}}[X_{m}|X_o],X_{o},Y] 
    p(X_{o},Y) 
\end{align*}
which does not hold in general and can lead to strongly biased results when $\mathbb{E}[C_{(\pi_\alpha)}|X_{m},X_{o},Y]$ is nonlinear as is the case generally in AFA settings.

\section{
Proof of Lemma \ref{lemma_comparison_positivity}
}
\label{app_proof_lemma_comparison_positivity}
In this appendix, we prove Lemma \ref{lemma_comparison_positivity}, which we repeat here for readability: 

\vspace{5pt}
\begin{lma}{
\ref{lemma_comparison_positivity}}
    (Sufficiency conditions for global positivity). The global positivity assumption for semi-offline RL (Assumption \ref{assump:positivity_global_semi_offline_RL}) holds if the positivity assumption from offline RL (Assumption \ref{assump:positivity_offline_RL}) or from missing data (Assumption \ref{assump:positivity_missing_data}) holds. 
\end{lma}
\vspace{5pt}

%The lemma states that the global positivity assumption for semi-offline RL (Assumption \ref{assump:positivity_global_semi_offline_RL}) is weaker than both the positivity assumptions under the offline RL and missing data views (Assumptions \ref{assump:positivity_offline_RL} and \ref{assump:positivity_missing_data}).

\noindent 
We split the proof into the following propositions:

\begin{proposition}
\label{prop_positivity_comparison_offline_RL}
If the positivity assumption for offline RL (Assumption \ref{assump:positivity_offline_RL}) holds, then the global positivity assumption for semi-offline RL also holds.
\end{proposition}

\begin{proposition}
\label{prop_positivity_comparison_missing_data}
If the positivity assumption for missing data (Assumption \ref{assump:positivity_missing_data}) holds, then the global positivity assumption for semi-offline RL also holds.
\end{proposition}

\noindent 
We begin by proving Proposition \ref{prop_positivity_comparison_offline_RL}:
\newline 

\begin{proof}
Consider the initial time step $t=1$, focusing on data points $(a'^1, x^0)$ such that $p(x^0)\pi_\alpha(a'^1|x^0) > 0$. To establish the global positivity assumption, we need to ensure that the regional positivity assumption is satisfied.

We will show that $a^1 = a'^1$ belongs to the regional admissible set $\tilde{\mathcal{A}}_\textit{adm}^1(x^0,a'^1)$, which implies that the regional positivity condition is met. Two conditions must be fulfilled for $a^1 = a'^1$ to be included in $\tilde{\mathcal{A}}_\textit{adm}^1(x^0,a'^1)$. 
First, $a^1 = a'^1$ must belong to the local admissible set $\mathcal{A}_\textit{adm}^1(x^0,a'^1)$, which directly follows from the offline RL positivity assumption.
Second, we must ensure that for $a^1 = a'^1$ the regional admissible set exists at the subsequent time step. Specifically, since $a^1 = a'^1$ (and therefore $x^1 = x'^1$), the regional admissible set at time step 2, $\tilde{\mathcal{A}}_\textit{adm}^2(\underline{a}'^2,\underline{x}^1,\underline{a}^1)$, must exist for all $x^1$ and $a'^2$ such that:
\begin{align*}
p(x^{1} | x^{0}, a^{1}) \pi_{\alpha}(a'^{2} | \underline{x}^{1}, \underline{a}^{1}) > 0.
\end{align*}
For these conditions, the positivity assumption under the offline RL view again implies that $a^2 = a'^2$ is included in the local admissible set $\mathcal{A}_\textit{adm}^2(a'^2,\underline{x}^1,\underline{a}^1)$. This reasoning can be extended iteratively through all time steps up to $T$, thus proving that regional positivity holds at every prior time point, which in turn establishes global positivity.
%\hfill %\qedsymbol 
\end{proof}

\noindent 
Next, we prove Proposition \ref{prop_positivity_comparison_missing_data}, following a similar strategy:
\newline 

\begin{proof}
Again, consider the initial time step $t=1$, focusing on data points $(a'^1, x^0)$ such that $p(x^0)\pi_\alpha(a'^1|x^0) > 0$. 
We show that  $a^1 = \vec{1}$ is included in the regional admissible set $\tilde{\mathcal{A}}_\textit{adm}^1(x^0,a'^1)$, irrespective of the value of $a'^1$. 
First, $a^1 = \vec{1}$ must be included in the local admissible set $\mathcal{A}_\textit{adm}^1(x^0,a'^1)$, a condition that directly follows from the missing data positivity assumption.

Second, $a^1 = \vec{1}$ must permit the existence of a regional admissible set at the next time step. Specifically, given $a^1 = \vec{1}$, the regional admissible set at time step 2, $\tilde{\mathcal{A}}_\textit{adm}^2(
%\underline{x}'^1,
\underline{a}'^2,\underline{x}^1,\underline{a}^1) = \tilde{\mathcal{A}}_\textit{adm}^2(
%\underline{x}'^1,
\underline{a}'^2,\underline{x}^1,\underline{a}^1=\vec{1})$, must exist for all $x^1$ and $a'^2$ such that:
\begin{align*}
p(x^{1} | x^{0}, A^{1}=\vec{1}) %g(x'^1|x^1, a'^1) 
\pi_{\alpha}(a'^{2} | \underline{x}'^{1},\underline{a}'^{1}) > 0.
\end{align*}
Under these conditions, the missing data positivity assumption again directly implies that $a^2 = \vec{1}$ belongs to the local admissible set at time step 2, $ \tilde{\mathcal{A}}_\textit{adm}^2(%\underline{x}'^1,
\underline{a}'^2,\underline{x}^1,\underline{a}^1=\vec{1})$, regardless of the values of $a'^2$ and $x'^1$. This reasoning can be extended step-by-step until $T$, thereby ensuring that regional positivity holds at each prior time point and, consequently, that global positivity is satisfied as well.
% \hfill % \qedsymbol 
\end{proof}

\section{Proof of Theorems \ref{theorem_identification_semi_offline_RL} and \ref{theorem_Bellman_equation}}
\label{app_theorem_identification}

In this Appendix, we prove Theorems \ref{theorem_identification_semi_offline_RL} and \ref{theorem_Bellman_equation}. We also demonstrate how the positivity assumption arises. 
We restate the theorems here for clarity and ease of reference.

\vspace{5pt}
\begin{thma}{\ref{theorem_identification_semi_offline_RL}}
(Identification of $J$ for the semi-offline RL view).
The reformulated AFAPE problem of estimating $J$ under the semi-offline RL view  (Eq. \ref{eq:AFAPE_objective_semi_offline_RL}) is under Assumption \ref{assump:measurement_noise} (no measurement noise),  Assumption \ref{assump:consistency} (consistency),  Assumption \ref{assump:interference} (no interference), Assumption \ref{assump:nde} (NDE), Assumption \ref{assump:nuc} (NUC) and Assumption \ref{assump:positivity_global_semi_offline_RL} (global positivity)
identified by
\begin{align}
\tag{\ref{eq_identificiation_semi_offline_RL_1}}
    J  
     = 
    \mathbb{E}_{p'}[C'_{(\pi_\alpha)}]
     = 
    \sum_{A',A,G_A(X_{(1)}),Y} 
    f_C(             
A', 
X',        
Y) 
q'(            
A',         
A, 
X, 
%G_A(X_{(1)}), 
Y) 
\end{align}
\noindent
with the distribution
\begin{align}
\tag{\ref{eq_identificiation_semi_offline_RL_2}}
q'(
A',A ,   
X, 
Y)
& = 
\prod_{t=1}^{T}  
\underbrace{\pi_{\text{id}}^t(
A^{t}| 
\underline{A}'^{t},
\underline{X}^{t-1}, 
\underline{A}^{t-1})}_{\text{distr. subject to constraints}}
\underbrace{
\pi_{\alpha}^t(A'^t|            
\underline{X}'^{t-1},            
\underline{A}'^{t-1})
}_{\text{target policy}}  
%\nonumber
% \\ & \cdot 
\prod_{t=0}^{T}
p(
X^{t}| 
\underline{X}^{t-1}, 
\underline{A}^{t}, Y) 
p(Y)
\end{align}

\noindent 
where 
\begin{flalign}
\tag{\ref{eq_pi_id}}
\pi_{id}^t(
A^t|  
\underline{A}'^t,&
\underline{X}^{t-1},
% G_{\underline{A}^{t-1}}(X_{(1)}),
\underline{A}^{t-1}) 
%\nonumber
=
% \\ & = 
\underbrace{\mathbb{I}(A^t \in 
\mathcal{
\tilde{A}}_\textit{adm}^t( 
\underline{X}^{t-1},
%G_{\underline{A}^{t-1}}(X_{(1)}),
\underline{A}^{t-1}, 
\underline{A}'^t
))}
_{\text{support restriction}}
f_{id}^t(
\underline{A}'^t,
\underline{X}^{t-1},
%G_{\underline{A}^{t-1}}(X_{(1)}),
\underline{A}^{t-1}
) 
\end{flalign}
for any function $f_{id}^t$
%(
%\underline{X}'^{t-1},
%\underline{A}'^t,
%\underline{X}^{t-1},
%\underline{A}^{t-1}
%) $ 
s.t. $\pi_{id}^t$ is a valid density.
\end{thma}

\vspace{5pt}
\begin{thma}{\ref{theorem_Bellman_equation}}
(Bellman equation for semi-offline RL). 
The semi-offline RL view admits under 
Assumption \ref{assump:measurement_noise} (no measurement noise),  Assumption \ref{assump:consistency} (consistency),  Assumption \ref{assump:interference} (no interference), Assumption \ref{assump:nde} (NDE), Assumption \ref{assump:nuc} (NUC) and the local positivity assumption at datapoint $\underline{x}^{t-1},\underline{a}^{t-1},a'^t$ (from Definition \ref{def_local_positivity}), the following semi-offline RL version of the Bellman equation:   
\begin{align}
\tag{\ref{eq_bellman_1}}
Q_{\textit{Semi}}( &
\underline{A}'^{t},         
\underline{X}^{t-1},
%G_{\underline{A}^{t-1}}(X_{(1)}), 
\underline{A}^{t-1}, 
\Xi) 
 = 
 %\\ & =  
\sum_{\mathclap{X^{t}}}                
V_{\textit{Semi}}
(    
\underline{A}'^{t}, 
\underline{X}^{t},
% G_{\underline{A}^{t}}(X_{(1)}),
\underline{A}^{t-1}, 
A^{t}=a^t, 
\Xi) 
%\nonumber
%\\ & \cdot 
% g(X'^t|X^t, A'^t)
p(
X^{t}
% G_{A^{t}}(X_{(1)})
|                   
% G_{\underline{A}^{t-1}}(X_{(1)}),   
\underline{X}^{t-1},
\underline{A}^{t-1}, 
A^t = a^t, 
\Xi)
% \nonumber 
\\
& \text{ for any }
a^t 
\in \mathcal{A}^t_\textit{adm}
(          
\underline{X}^{t-1}, 
\underline{A}^{t-1},
A'^t )  
\nonumber 
\\
%\nonumber
 V_{\textit{Semi}}&
(
\underline{A}'^{t},
\underline{X}^{t},
%G_{\underline{A}^{t}}(X_{(1)}),  
\underline{A}^{t}, 
\Xi)  
=
% = \\
%& =   
\sum_{A'^{t+1}}  
Q_{\textit{Semi}}(
\underline{A}'^{t+1},    
\underline{X}^{t},
%G_{\underline{A}^{t}}(X_{(1)}), 
\underline{A}^{t}, 
\Xi
)
\pi_{\alpha}^{t+1}(
A'^{t+1}|
\underline{X}'^{t},
% G_{\underline{A}'^{t}}(X_{(1)}),
\underline{A}'^{t})
\tag{\ref{eq_bellman_2}}
% \label{eq_bellman_2}
\end{align}
with semi-offline RL versions of the state-action value function $Q_{\textit{Semi}}$ and state value function $V_{\textit{Semi}}$:
\begin{align*}
 Q_{\textit{Semi}}^t 
 & \equiv 
 Q_{\textit{Semi}}(
 %\underline{X}'^{t-1},
\underline{A}'^{t},
\underline{X}^{t-1},
%G_{\underline{A}^{t}}(X_{(1)}),
\underline{A}^{t-1}, 
\Xi) 
\equiv  
\mathbb{E}_{p'}[
C'_{(\overline{\pi}^{t+1}_\alpha)}|
%\underline{X}'^{t-1},
\underline{A}'^{t},
\underline{X}^{t-1},
%G_{\underline{A}^{t}}(X_{(1)}),
\underline{A}^{t-1}, 
\Xi]  
\\
V_{\textit{Semi}}^t
& \equiv 
V_{\textit{Semi}}(
%\underline{X}'^t,
\underline{A}'^{t},
\underline{X}^{t},
\underline{A}^{t}, 
\Xi) 
\equiv 
\mathbb{E}_{p'}[C'_{(\overline{\pi}^{t+1}_\alpha)}|
% \underline{X}'^t,
\underline{A}'^{t},
\underline{X}^t,
% G_{\underline{A}^{t}}(X_{(1)}),
\underline{A}^{t},
\Xi]
\end{align*}
%\end{subequations}
where $C'_{(\overline{\pi}^{t+1}_\alpha)}$ denotes the potential outcome of $C'$ under interventions from time step $t+1$ onwards. 
$\Xi \subseteq \{ Y, O\}$, with $O$ denoting all features that are always available, denotes an optional subset of additional variables that can be conditioned on. 
%$$\pi_\alpha$ applied only from time step $t+1$ onwards. 
Furthermore, $Q_{\textit{Semi}}^t$ and $V_{\textit{Semi}}^t$ are identified if the regional positivity assumption (from Definition \ref{def_regional_positivity}) holds at 
$
%\underline{X}'^{t-1},
\underline{X}^{t-1},
\underline{A}^{t-1}, 
\underline{A}'^{t}$ and $a^t \in \mathcal{\tilde{A}}_\textit{adm}^t(
%\underline{X}'^{t-1},
\underline{X}^{t-1},
\underline{A}^{t-1}, 
\underline{A}'^{t})$.
\end{thma}
\vspace{5pt}

\vspace{10pt}
\begin{proof}
Firstly, we factorize the counterfactual distribution, denoted by $q'$, expressing it as a function of the observed (simulated) data. 
We factorize the graph in a step-by-step fashion to show how the semi-offline RL version of the Bellman equation arises. We split identification in each step into two parts to emphasize the two parts of the Bellman equation. 
To help guide the identification, we duplicate Figure \ref{graph_semi_offline_RL}
of the causal graph 
describing the simulation process in Figure \ref{graph_semi_offline_RL_with_intervention}A). Alongside it, we show the counterfactual graph (for identification step $t=1$) in Figure \ref{graph_semi_offline_RL_with_intervention}B).

\vspace{10pt}
\noindent 
\textbf{Step 0}

\noindent 
\textit{Counterfactual factorization (step $t = 0$, part 1):}
\begin{align*}                
p'   
\left( 
C'_{(\pi_\alpha)}
\right)  
\overset{*_1}{\equiv}  
p'   
\left( 
C'_{(\overline{\pi}^1_{\alpha})}
\right)  
=  & 
\sum_{X^0, \Xi} 
p'
\left(   
C'_{(   
\overline{\pi}^1_{\alpha})
}
\Big\vert                   
X^0, 
\Xi
\right)  
 p(X^0, \Xi) 
\end{align*} 

\noindent 
where we denote in $*1)$ $C'_{(\overline{\pi}^1_{\alpha})}$ as the counterfactual $C'$ under an intervention of $\pi_\alpha$ from step $t=1$ onwards. 
The extension by $X^0$ is needed for adjustment. 
% To simplify notation we write $X^0 \equiv X^0 \equiv X^0$. 
The inclusion of $\Xi \subseteq \{Y, O\}$, where $O$ denotes the subset of always observed features amongst $X_{(1)}$, is optional.

\vspace{10pt}
\noindent 
\textit{Counterfactual factorization (step $t = 0$, part 2):}
\begin{align*}   
p'
\left(   
C'_{(   
\overline{\pi}^1_{\alpha})
}
\Big\vert                   
X^0  , 
\Xi
\right) 
&
=
\boldsymbol{
\sum_{a'^1} }
p'
\left(   
C'_{\boldsymbol{(   
\overline{\pi}^2_{\alpha},
% A'^1= 
a'^1)}
}        
\Big\vert                   
X^0, 
\Xi
\right) 
\boldsymbol{\pi_\alpha^1(
a'^1 | X^0)}
\\ &
\overset{*_1}{=} 
\sum_{a'^1} 
p'
\left(   
C'_{(   
\overline{\pi}^2_{\alpha},
% A'^1= 
a'^1, 
\boldsymbol{\pi_{id}^1})
}        
\Big\vert                   
X^0     ,
\Xi
\right) 
\pi_\alpha^1(
a'^1 | 
X^0
)  
\\ & 
=  
\sum_{a'^1, \boldsymbol{a^1}}
p'
\left(   
C'_{(   
\overline{\pi}^2_{\alpha},
% A'^1= 
a'^1, 
\boldsymbol{a^1})
}        
\Big\vert                   
X^0    ,
\Xi
\right) 
\boldsymbol{\pi_{id}^1(a^1|X^0, a'^1) }
\pi_\alpha^1(
a'^1 |X^0)  
 \\ & \overset{*_2}{=}  
 \sum_{a'^1,a^1}
p'
\left(   
C'_{(   
\overline{\pi}^2_{\alpha},
% A'^1= 
a'^1, 
a^1)
}        
\Big\vert                   
X^0, 
\boldsymbol{a^1}, 
\Xi
\right) 
\pi_{id}^1(a^1|X^0, a'^1) 
\pi_\alpha^1(
a'^1 |X^0)  
 \\ & 
\overset{*_3}{=}
\sum_{a'^1, a^1}
p'
\left(   
C'_{\boldsymbol{(   
\overline{\pi}^2_{\alpha},
% A'^1= 
a'^1)}
}        
\Big\vert                   
X^0   , 
a^1, 
\Xi
\right)  
\pi_{id}^1(a^1|X^0, a'^1) 
\pi_\alpha^1(
a'^1 | X^0)  
 \\ & 
\overset{*_4}{=}   
\sum_{a'^1, a^1}
p'
\left(   
C'_{(   
\overline{\pi}^2_{\alpha},
% A'^1= 
a'^1)
}        
\Big\vert                   
X^0, 
a'^1, 
\boldsymbol{a^1}, 
\Xi
\right) 
\pi_{id}^1(a^1|X^0, a'^1) 
\pi_\alpha^1(
a'^1 | X^0)  
 \\ & 
\overset{*_5}{=}   
\sum_{a'^1, a^1}
p'
\left(   
C'_{\boldsymbol{(   
\overline{\pi}^2_{\alpha})}
}        
\Big\vert                   
X^0   , 
a'^1, a^1, 
\Xi
\right) 
\pi_{id}^1(a^1|X^0, a'^1) 
\pi_\alpha^1(
a'^1 | X^0)  
 \\ & 
\overset{*_6}{=}   
\sum_{a'^1, a^1}
p'
\left(   
C'_{(   
\overline{\pi}^2_{\alpha})
}        
\Big\vert                   
\boldsymbol{X^0, 
a'^1, 
\Xi}
\right) 
\pi_{id}^1(a^1|X^0, a'^1) 
\pi_\alpha^1(
a'^1 | X^0)  
\end{align*}

\noindent 
with the following explanations:

\begin{figure}
\centering
%\vspace{-10pt}
\includegraphics[width=0.85\textwidth]
    {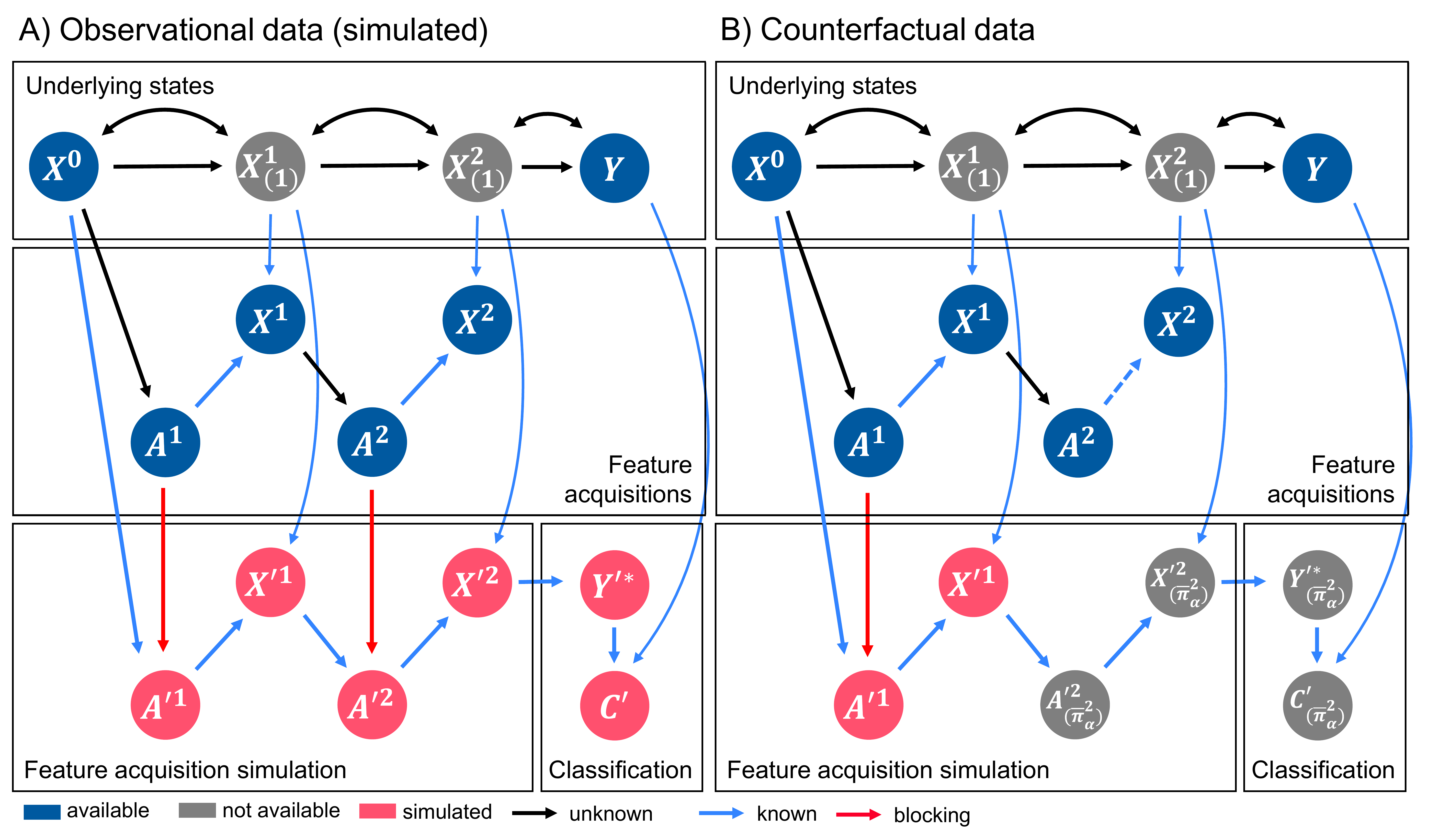}
    \caption{ 
    Causal graph for the distribution $p'$. A) Simulated ("observational") distribution. B) Counterfactual distribution under the intervention $\overline{\pi}_\alpha^2$. % (with $t = 2$). % and $\overline{\pi}_{id}^t$ (with $t = 2$). 
    %The following holds: $\overline{A}'^t_{(\overline{\pi}_{\alpha}^t)} = \overline{A}'^t_{(\overline{\pi}_{\alpha}^t, \overline{\pi}_{id}^t)}$, $\overline{X}'^t_{(\overline{\pi}_{\alpha}^t)} = \overline{X}'^t_{(\overline{\pi}_{\alpha}^t, \overline{\pi}_{id}^t)}$, $Y'^*_{(\overline{\pi}_{\alpha}^t)} = Y'^*_{(\overline{\pi}_{\alpha}^t, \overline{\pi}_{id}^t)}$, and $C'_{(\overline{\pi}_{\alpha}^t)} = C'_{(\overline{\pi}_{\alpha}^t, \overline{\pi}_{id}^t)}$.
    Edges showing long-term dependencies are omitted from the graphs for visual 
    clarity. 
    These include: $\underline{X}_{(1)}^{t-1} 
    \leftrightarrow X_{(1)}^{t}$;  $\underline{X}_{(1)}^{T}
    \leftrightarrow Y$;   
 $\underline{X}^{t-1}
 %/
%\underline{X}^{t-1}_{(\overline{\pi}_{\alpha}^2)}
, \underline{A}^{t-1} 
%/ 
%\underline{A}^{t-1}_{(\overline{\pi}_{\alpha}^2)}
\rightarrow 
A^{t} 
%/ 
%A^{t}_{(\overline{\pi}_{\alpha}^2)}
$;  
$\underline{X}'^{t-1}/
\underline{X}'^{t-1}_{(\overline{\pi}_{\alpha}^2)}
, \underline{A}'^{t-1} / 
\underline{A}'^{t-1}_{(\overline{\pi}_{\alpha}^2)}
\rightarrow 
A'^{t} / 
A'^{t}_{(\overline{\pi}_{\alpha}^t)}
$; and 
$\underline{X}'^{T}/
\underline{X}'^{T}_{(\overline{\pi}_{\alpha}^2)}
, \underline{A}'^{T} / 
\underline{A}'^{T}_{(\overline{\pi}_{\alpha}^2)}
\rightarrow 
Y'^* / 
Y'^*_{(\overline{\pi}_{\alpha}^2)}$.}
\label{graph_semi_offline_RL_with_intervention}
\vspace{-1 pt}
%\end{wrapfigure}
\end{figure}

\begin{itemize}
    \item $*1)$: We notice that $C'_{(\overline{\pi}^2_\alpha,a'^1)} $ is independent of any interventions $\pi_{id}^1$ on $A^1$. This step prevents positivity problems in subsequent steps. 
    \item $*2)$: 
    We use the exchangeability 
    $C'_{(\overline{\pi}^2_{\alpha},%A'^1=
    a'^1, a^1)} \perp\!\!\!\perp  
    A^1 |
    X^0, 
    \Xi$ which follows from the NUC assumption. 
    \item $*3)$: 
    We use the consistency assumption:  \\   $p'
    \left(  
    C'_{(\overline{\pi}^2_{\alpha}, 
    a'^1, a^1)}     
    \Big\vert                   
    X^0,
    a^1 ,
    \Xi
    \right) 
    = 
    p'
    \left(  
    C'_{(\overline{\pi}^2_{\alpha}, 
    a'^1)}     
    \Big\vert                   
    X^0, 
    a^1, 
    \Xi
    \right)$ 
    \item $*4)$: We use the exchangeability: 
    $C'_{(\overline{\pi}^2_{\alpha},%A'^1=
    a'^1)} \perp\!\!\!\perp  
    A'^1 |
    X^0, A^1, 
    \Xi$ 
    \item $*5)$: We use the consistency assumption: 
    $p'
    \left(  
    C'_{(\overline{\pi}^2_{\alpha}, 
    a'^1)}     
    \Big\vert                   
    X^0, 
    a'^1,
    a^1 , 
    \Xi
    \right) 
    = 
    p'
    \left(  
    C'_{(\overline{\pi}^2_{\alpha})}     
    \Big\vert                   
    X^0,
    a'^1,
    a^1, 
    \Xi
    \right)$ 
    \item $*6)$: We use the conditional independence 
    $C'_{(\overline{\pi}^2_{\alpha})} \perp\!\!\!\perp  
    A^1 |
    X^0, A'^1,
    \Xi$
\end{itemize}

\noindent 
%This identification step included the addition and subsequent removal of an intervention/ conditioning on $A^1$. We included this seemingly unnecessary step to demonstrate the required positivity assumption which has to ensure that 
We must also ensure that $p'
\left(  
C'_{(\overline{\pi}^2_{\alpha}, 
a'^1)}     
\Big\vert                   
X^0, 
a'^1,
a^1 , 
\Xi
\right) 
$, i.e. conditioning on $X^0, A'^1, A^1, \Xi $, is well specified in $*4)$. 
To understand what positivity requirements are necessary, we first factorize the "observational" (i.e. simulated) distribution for step $t=0$. By observational distribution for step $t=0$, we refer to a distribution which only contains interventions from step $t=2$ onwards: 

\vspace{10pt}
\noindent 
\textit{Observational factorization (step $t = 0$):}
\begin{align*}  
p'   
\left( 
C'_{(\overline{\pi}^2_{\alpha})}
\right)  
% \\ 
= % & 
\sum_{X^0,A^1,A'^1, \Xi}
%\sum_{\mathclap{X'^0,X^0,A^1,A'^1}}
%\sum_{\mathclap{X'^0,X^0,A^1,a'^1}}
p'
\left(   
C'_{(   
\overline{\pi}^2_{\alpha})
}        
\Big\vert                   
X^0,
A'^1, 
\Xi
%= a'^1 
\right) 
\underbrace{
\pi_{sim}'^1 (
A'^1
%= a'^1 
| 
X^0,
A^1
)}  
_{\text{simulation policy}} 
\underbrace{
\pi_{\beta}^1(
A^1 %= a^1 
| 
X^0)
}_{\text{retro. acq. policy}}  
p(X^0, \Xi)
\end{align*}

\noindent 
By comparing the observational and counterfactual factorizations, we see that the following positivity assumption is required: 
\begin{flalign}
\text{if }  \hspace{14 pt} \quad \quad \quad \quad \quad 
&p(x^{0}) q'( a'^1, a^1 | x^{0})  =  
\nonumber
% \\  & \quad \quad  =  
p(x^0) 
\pi_\alpha^1(
 a'^1 | x^0)     
\pi_{{id}}^1 (
a^1 | x^0, a'^1)  
>0
\nonumber
&&
\\
\text{then } \quad \quad \quad \quad \quad 
&
p(x^{0}) p'( a'^1, a^1 | x^{0}) =  
\nonumber
% \\ & \quad \quad  =
p(x^0) 
\pi_{sim}'^1(
%A'^1 = 
a'^1 | x^0, a^1)  
\pi_{\beta}^1 (
a^1 | x^0)  
\geq \mathcal{O}
\nonumber
&&
\\ 
&
\forall 
x^0,
a'^1,
a^1, \text{and some constant } \mathcal{O} > 0
\end{flalign}
\noindent 
Since none of the distributions $\pi_\alpha$, $\pi_\textit{id}$, $\pi'_\textit{sim}$, and $\pi_\beta$ depend on $\Xi$, the choice for $\Xi$ will not influence the positivity requirements.
We can further simplify the positivity assumption, by using knowledge about the known simulation policy $\pi'_{sim}$. 
By the construction of the blocking operation of the simulation policy $\pi_{sim}'$ (Definition \ref{def_blocked_policy}), one observes that 
\begin{flalign*}
 &\text{if } 
\pi_\alpha^1(
%A'^1 = 
a'^1 | x^0) > 0, \quad \quad \quad \quad 
 \quad \quad  \text{then }
\pi_{sim}'^1(
%A'^1 = 
a'^1 | x^0, 
%A^1=
a^1) \geq \mathcal{O}_1,  \quad \quad 
&& \text{if and only if } a'^1 \leq a^1   
\end{flalign*}

%if 
%$ \pi_\alpha^t(
%%A'^1 = 
%a'^1 | x'^0) > 0$, then 
%$\pi_{sim}'(
%%A'^1 = 
%a'^1 | x'^0, 
%%A^1=
%a^1) > 0$ 
%if and only if $a'^1 \leq a^1$.
\noindent 
where, $\mathcal{O}_1$ is some constant $>0$, and, as before, we let $a'^1 \leq a^1$ denote the element-wise comparison. 
The resulting positivity violation for the case $a'^1 \not \leq  a^1$ can be avoided by restricting $\pi_{{id}}$ in the following way: 

\vspace{10pt}
\noindent
\textit{Restriction 1 for $\pi_{{id}}$ (step $t=0$):}
\begin{flalign*}
& \text{if } a'^1 \not \leq  a^1, \quad \quad \quad \quad 
\quad \quad \quad \quad \quad \text{then }
\pi_{{id}}^1(
a^1 | x^0, a'^1)   = 0 \quad \quad && \forall  x^0,a'^1, a^1. 
\end{flalign*}
A second possible positivity violation arises if $\pi_{\beta}^1(
a^1 | x^0) = 0$ for some values of $a^1$. This poses a second requirement for $\pi_{{id}}$:

\vspace{10pt}
\noindent 
\textit{Restriction 2 for $\pi_{{id}}$ (step $t=0$):}
\begin{flalign*}
& \text{if } \pi_{\beta}^1(
%A^1 = 
a^1 | x^0) = 0, \quad \quad \quad \quad \quad
\quad \quad  \text{then }
\pi_{{id}}^1(
%A^1 = 
a^1 | x^0, 
a'^1) = 0 
&& \forall  x^0, a'^1, a^1.
\end{flalign*}

\noindent 
Since $\pi_{{id}}$ is required to be a valid probability distribution (it cannot be 0 for all $a^1$), this imposes the following requirement for $\pi_\beta$:
%that there exists at least one value $a^1$ such that: 
% $ a'^1 \leq  a^1$ and $\pi_{\beta}(
% a^1 | X^0)>0$. %
\begin{flalign*}
& \text{if } p(x^0) \pi_\alpha^1(a'^1|x^0) > 0, 
\quad \quad 
\text{then }
\pi_{\beta}^1(
A^1 \geq a'^1 | x^0) \geq \mathcal{O} 
&& \forall  x^0, a'^1, \text{ and some constant } \mathcal{O}>0.
\end{flalign*}
%This corresponds to the local positivity assumption at step $0$ (Definition \ref{def_local_positivity}). 

\noindent 
The positivity assumption implies that for any desired action $a'^1$ by the target policy $\pi_\alpha$, that there exists at least positive support for one set of acquisitions $a^1$ that include equal or more acquisitions than what is contained in $a'^1$.
This is equivalent to the local positivity assumption at $x^0,a'^0$ (i.e. the existence of $\mathcal{A}_\textit{adm}^1$ from 
Definition \ref{def_local_positivity}).  
In the next steps, we show that these are only minimal requirements for $\pi_{id}^1(A^1|X^0, A'^1)$. To avoid running into positivity violations in later time steps, a further restriction can be necessary.

\vspace{10pt}
\noindent 
\textbf{Step 1}

\noindent 
In the following, we continue the identification for step $t=1$. 

\vspace{10pt}
\noindent 
\textit{Counterfactual factorization (step $t = 1$, part 1):}
\begin{align*}                
p'   
\biggl( 
C'_{(\overline{\pi}^2_{\alpha})} 
\Big\vert  
X^0, A'^1, \Xi
\biggl)  & = 
p'    
\biggl( 
C'_{(\overline{\pi}^2_{\alpha})} 
\Big\vert  
X^0, A'^1, 
a^1, \Xi
\biggl) 
=
\\ &  
= 
\sum_{X^1} 
p'
\left(   
C'_{(
\overline{\pi}^2_{\alpha}
)} 
\Big\vert                   
%\underline{X}'^1,
\underline{A}'^1,
\underline{X}^1, 
%A^1 = 
a^1, 
\Xi
\right) 
%\underbrace{g(
%X'^1|X^1, A'^1 
%)}_{\text{feature revelation}}   
p(
X^1|
X^0,
a^1, 
\Xi
)   
\end{align*}   

\noindent 
which holds for any $a^1 \in \mathcal{A}_\textit{adm}^1(X^0, A'^1)$ (because local positivity must hold). 
%The step contains the expansion by $A^1$, $X^1$, and $X'^1$, the variables needed for adjustment. 
Therefore, the term $p'
\left(   
C'_{(
\overline{\pi}^2_{\alpha}
)} 
\Big\vert                   
%\underline{X}'^1,
\underline{A}'^1,
\underline{X}^1, 
%A^1 = 
a^1, 
\Xi
\right) $ needs to be only identified for \textit{at least one value} $a^1 \in \mathcal{A}^1_\textit{adm}(X^0, A'^1)$. 
%Note also that 
%$X'^1$ is a deterministic function of $X^1$ and thus would not necessarily need a separate expected value. 
%Furthermore, note that 
%$p(X^1| X^0, a^1, \Xi) = p(X_{(1), a^1}^1| X^0)$ corresponds to the counterfactual where we let $X_{(1), a^1}^1$ denote $X_{(1)}^1$ indexed at all $i$ such that $a_i^1 = 1$.  

\vspace{10pt}
\noindent 
\textit{Counterfactual factorization (step $t = 1$, part 2):}
\begin{align*}    
p' \biggl(  &  
C'_{(
\overline{\pi}^2_{\alpha}
)} 
\Big\vert                   
%\underline{X}'^1,
\underline{A}'^1,
\underline{X}^1, 
\underline{A}^1, 
\Xi
\biggl) =
\\  & = 
\boldsymbol{\sum_{ a'^2} }
p'
\left(   
C'_{\boldsymbol{(\overline{\pi}^3_{\alpha}, 
a'^2)}} 
\Big\vert                   
%\underline{X}'^1,
\underline{A}'^1,
\underline{X}^1, 
\underline{A}^1, 
\Xi
\right) 
\boldsymbol{
\pi_\alpha^2(
a'^2 | 
\underline{X}'^1,
\underline{A}'^1)  
}
\\ & \overset{*_1}{=}
\sum_{ a'^2} 
p'
\left(   
C'_{(\overline{\pi}^3_{\alpha}, 
a'^2, 
\boldsymbol{\pi_{id}^2})} 
\Big\vert                   
%\underline{X}'^1,
\underline{A}'^1,
\underline{X}^1, 
\underline{A}^1 ,
\Xi
\right) 
\pi_\alpha^2(
a'^2 | 
\underline{X}'^1,
\underline{A}'^1) 
\\ & = 
\sum_{ a'^2, \boldsymbol{a^2}} 
p'
\left(   
C'_{(\overline{\pi}^3_{\alpha}, 
a'^2,
\boldsymbol{a^2} )} 
\Big\vert                   
%\underline{X}'^1,
\underline{A}'^1,
\underline{X}^1, 
\underline{A}^1, 
\Xi
\right) 
\boldsymbol{\pi_{{id}}^2( 
a^2 | 
%\underline{X}'^1,
\underline{A}'^1,
a'^2,
\underline{X}^1, 
\underline{A}^1 
%\Xi
)  }
% \\ & \cdot 
\pi_\alpha^2(%A'^2= 
a'^2 | 
\underline{X}'^1,
\underline{A}'^1)  
\\ & \overset{*_2}{=}
\sum_{ a'^2, a^2} 
p'
\left(   
C'_{\boldsymbol{(\overline{\pi}^3_{\alpha}, 
a'^2)}} 
\Big\vert                   
%\underline{X}'^1,
\underline{A}'^1,
\underline{X}^1, 
\underline{A}^1, 
\boldsymbol{a^2},
\Xi
\right) 
\pi_{{id}}^2( 
a^2 | 
%\underline{X}'^1, 
\underline{A}'^1,
a'^2,
\underline{X}^1 , 
\underline{A}^1
)  
% \\ & \cdot 
\pi_\alpha^2(%A'^2= 
a'^2 | 
\underline{X}'^1,
\underline{A}'^1)    
\\ & \overset{*_3}{=}
\sum_{a'^2, a^2} 
p'
\left(   
C'_{\boldsymbol{(\overline{\pi}^3_{\alpha})}} 
\Big\vert                   
%\underline{X}'^1,
\underline{A}'^1,
\boldsymbol{a'^2},
\underline{X}^1, 
\underline{A}^1, 
a^2, 
\Xi
\right) 
\pi_{{id}}^2(%A^2= 
a^2 | 
%\underline{X}'^1, 
\underline{A}'^1,
a'^2,
\underline{X}^1 , 
\underline{A}^1
%A'^2 = 
%a'^2
) 
%\\ & \cdot 
\pi_\alpha^2(
%A'^2= 
a'^2 | 
\underline{X}'^1,
\underline{A}'^1)   
\\ &\overset{*_4}{=}
\sum_{a'^2, a^2} 
p'
\left(   
C'_{(\overline{\pi}^3_{\alpha})} 
\Big\vert                   
\boldsymbol{
%\underline{X}'^1,
\underline{A}'^1,
a'^2,
\underline{X}^1, 
\underline{A}^1}, 
\Xi
\right) 
\pi_{{id}}^2(%A^2= 
a^2 | 
%\underline{X}'^1, 
\underline{A}'^1,
a'^2,
\underline{X}^1 , 
\underline{A}^1
%A'^2 = 
%a'^2
) 
%\\ & \cdot 
\pi_\alpha^2(
%A'^2= 
a'^2 | 
\underline{X}'^1,
\underline{A}'^1)      
\end{align*}

\noindent 
where we denote $\underline{X}'^1 \equiv G_{\underline{A}'^1}(X_{(1)})$ which contains a subset of the features in $\underline{X}^1 = G_{\underline{A}^1}(X_{(1)})$. 
The derivation was based on the following arguments:

\begin{itemize}
    \item $*1)$: We use that  $C'_{(\overline{\pi}^3_\alpha,a'^2)} $ is independent of any interventions $\pi_{id}^2$ on $A^2$.
    \item $*2)$: 
    We use exchangeability :
    $C'_{(\overline{\pi}^3_{\alpha},%A'^1=
    a'^2, a^2)} \perp\!\!\!\perp  
    A^2 |%\underline{X}'^1,
    \underline{A}'^1,\underline{X}^1,\underline{A}^1,\Xi$
    and consistency: \newline $p'
    \left(  
    C'_{(\overline{\pi}^3_{\alpha}, 
    a'^2, a^2)}     
    \Big\vert                   %\underline{X}'^1,
    \underline{A}'^1,
    \underline{X}^1,
    \underline{A}^1,
    a^2, 
    \Xi
    \right) 
    = 
    p'
    \left(  
    C'_{(\overline{\pi}^3_{\alpha}, 
    a'^2)}     
    \Big\vert                   
    % \underline{X}'^1,
    \underline{A}'^1,
    \underline{X}^1,
    \underline{A}^1, 
    a^2, 
    \Xi
    \right)$ for $A^2$. 
    \item $*3)$: We use the exchangeability
    $C'_{(\overline{\pi}^3_{\alpha}, a'^2)} \perp\!\!\!\perp  
    A'^2 |
    %\underline{X}'^1,
    \underline{A}'^1,
    \underline{X}^1,
    \underline{A}^1, 
    a^2, 
    \Xi$ and consistency:
    $ p'
    \left(  
    C'_{(\overline{\pi}^3_{\alpha}, 
    a'^2)}     
    \Big\vert                   
    %\underline{X}'^1 ,
    \underline{A}'^1 ,
    %A'^2 =  
    a'^2,
    \underline{X}^1 ,
    \underline{A}^1, 
    a^2,
    \Xi
    \right) 
    = 
    p'
    \left(  
    C'_{(\overline{\pi}^3_{\alpha})}    
    \Big\vert                   
    % \underline{X}'^1 ,
    \underline{A}'^1 ,
    %A'^2 =  
    a'^2,
    \underline{X}^1 ,
    \underline{A}^1,
    a^2, 
    \Xi
    \right) $ for $A'^2$.
    \item $*4)$: We use the conditional independence 
    $C'_{(\overline{\pi}^3_{\alpha})} \perp\!\!\!\perp  
    A^2 |
    %\underline{X}'^1,
    \underline{A}'^1,
    a'^2,
    \underline{X}^1,
    \underline{A}^1, 
    \Xi$
\end{itemize}

\noindent 
%This identification step included the addition and subsequent removal of an intervention/ conditioning on $A^2$. We included this step to demonstrate the required positivity assumption which 
We must also ensure in $*3)$ that $p'
    \left(  
    C'_{(\overline{\pi}^3_{\alpha}, 
    a'^2)}     
    \Big\vert                   
    %\underline{X}'^1 ,
    \underline{A}'^1 ,
    %A'^2 =  
    a'^2,
    \underline{X}^1 ,
    \underline{A}^1, 
    a^2, 
    \Xi
    \right)$, 
i.e. conditioning on $    
%\underline{X}'^1 ,
    \underline{A}'^1 ,
    %A'^2 =  
    a'^2,
    \underline{X}^1 ,
    \underline{A}^1, a^2, 
    \Xi$, is well specified. 
To understand what positivity requirements are necessary, we factorize the "observational" (i.e. simulated) distribution for step $t=1$.

\vspace{10pt}
\noindent 
\textit{Observational factorization (step $t = 1$):}
\begin{align*}  
p'   
\biggl( 
C'_{(\overline{\pi}^3_{\alpha})}
\Big\vert                  
A'^1,&
X^0,
A^1, 
\Xi
\biggl)   
=  
%\\ 
%&=  
\sum_{X^1,a^2,a'^2} 
p'
\left(   
C'_{(   
\overline{\pi}^3_{\alpha})
}        
\Big\vert                   
% \underline{X}'^1 ,
\underline{A}'^1 ,
%A'^2 = 
a'^2,
\underline{X}^1 ,
\underline{A}^1 , 
\Xi
\right) 
\\  & \cdot
\underbrace{
\pi_{sim}'^2(
%A'^2 = 
a'^2 | 
\underline{X}'^1,
\underline{A}'^1, 
%A^2 = 
a^2) 
}  
_{\text{known simulation policy}}
\underbrace{
\pi_{\beta}^2(
%A^2 = 
a^2 | 
\underline{X}^1,
\underline{A}^1) 
}_{\text{retro. acquisition policy}} 
%\underbrace{g(
%X'^1|X^1, A'^1 
%)}_{\text{feature revelation}}   
p(
X^1|
X^0,
A^1,
\Xi
)   
\end{align*}

\noindent 
By comparing the observational and counterfactual factorizations, we see that the following positivity assumption is required: 
\begin{flalign*}
\nonumber
\text{if }  \hspace{14 pt} \quad \quad
& q'(
% \underline{x}'^1,
\underline{a}'^1, 
%A'^2 = 
a'^2, 
\underline{x}^1, 
\underline{a}^1,
%A^2 = 
a^2) = 
\\ & \quad \quad \quad \quad  = 
 q'(
% \underline{x}'^1,
\underline{a}'^1,
\underline{x}^1, 
\underline{a}^1) 
%\\ & \cdot 
\pi_\alpha^2(%A'^2= 
a'^2 | 
\underline{x}'^1,
\underline{a}'^1) 
\pi_{{id}}^2(%A^2= 
a^2 | 
%\underline{x}'^1,
\underline{a}'^1,
a'^2,
\underline{x}^1,
\underline{a}^1)   
>0
\nonumber
&&
\\
\text{then } \quad \quad
& 
p'(
%\underline{x}'^1,
\underline{x}'^1, 
%A'^2 = 
a'^2, 
\underline{x}^1, 
\underline{a}^1,
%A^2 =
a^2)  = 
\nonumber 
\\ & \quad \quad \quad \quad
= 
p'(
%\underline{x}'^1,
\underline{a}'^1,
\underline{x}^1, 
\underline{a}^1)  
\pi_{sim}'^2(
%A'^2 = 
a'^2 | 
\underline{x}'^1,
\underline{a}'^1, 
%A^2 = 
a^2)   
\pi_{\beta}^2(
%A^2 =
a^2 | 
\underline{x}^1,
\underline{a}^1) 
\geq \mathcal{O}
\nonumber
&&
\\ 
&
\forall 
%\underline{x}'^1,
\underline{x}^1, 
\underline{a}'^1,
\underline{a}^1,
a'^2, 
a^2, \text{ and some constant }\mathcal{O}>0
\end{flalign*}
with the following factorizations: 
\begin{align*}
q'(
% \underline{X}'^1,
\underline{A}'^1,
\underline{X}^1, \underline{A}^1)  
& = 
q'(
% \underline{X}'^0,
\underline{A}'^1,
\underline{X}^0,
\underline{A}^1)
%g(
%X'^1|X^1, A'^1
%)  
p(
X^1|
\underline{X}^0,
\underline{A}^1
)  
\\
p'(
%\underline{X}'^1,
\underline{A}'^1,
\underline{X}^1, \underline{A}^1)  
& = 
p'(
%\underline{X}'^0,
\underline{A}'^1,
\underline{X}^0,
\underline{A}^1)
%g(
%X'^1|X^1, A'^1
%)  
p(
X^1|
\underline{X}^0,
\underline{A}^1
)  
\end{align*}

\noindent
%\begin{align*}
%& \text{if }  &&
%q'(X'^0,X^0,A'^1, A^1 = a^{1*}) 
%q'(X'^1,X^1,A'^2=a'^2, A^2 = %a^2|\underline{X}'^0,\underline{%A}'^1,\underline{X}^0, %\underline{A}^1 = %\underline{a}^{1*})
%>0
%\\
%& \text{then  } &&
%p'(X'^0,X^0,A'^1, A^1 = a^{1*}) 
%p'(X'^1,X^1,A'^2=a'^2, A^2 = %a^2|
%\underline{X}'^0,
%\underline{A}'^1,
%\underline{X}^0, %\underline{A}^1= %\underline{a}^{1*})
%>0
%\\ 
%& &&  \forall 
%\underline{X}'^1,
%\underline{X}^1, 
%A'^1, 
%a'^2, 
%a^2, \text{and at least one %value for } \underline{a}^{1*}.
%\end{align*}
%

\noindent 
Note that we can again ignore $\Xi$ since it doesn't affect positivity requirements. 
The positivity condition can again be simplified through the two restrictions on $\pi_{id}^2$:

\vspace{10pt}
\noindent
\textit{Restrictions 1 and 2 for $\pi_{{id}}$ (step $t=2$):}
\begin{flalign*}
& \text{if }   
a'^2\not \leq  a^2 \text{ or } \pi_{\beta}^2(
%A^2 = 
a^2 | 
\underline{x}^1,
\underline{a}^1) = 0, 
%&& \\
\quad  \quad  \quad  \quad  \quad 
\text{then } %\quad \quad \quad \quad \quad \quad \quad \quad
\pi_{{id}}^2(
%A^2 = 
a^2 | 
%\underline{x}'^1,
\underline{a}'^1, 
a'^2,
\underline{x}^1,
\underline{a}^1 
%A'^2=
%a'^2
) = 0 
%\\ 
&& \forall  
%\underline{x}'^1, \underline{a}^1,
%\underline{x}^1,
%\underline{a}^1, 
a^2, a'^2
\end{flalign*}

\noindent 
This imposes the requirement for $\pi_\beta^2$ that there exists at least one value $a^2$ such that $ a'^2 \leq  a^2$ and $
\pi_{\beta}^2(
a^2 | 
\underline{x}^1, 
\underline{a}^1) \geq \mathcal{O}$ (i.e. local positivity at $\underline{x}^1, 
\underline{a}^1, a'^1$). 
Notice, however, that this has to hold for \textit{all values} $a^1$
 that were "allowed" in step $t=1$ (i.e. where $\pi_{id}^1(a^1|x^0, a'^1) \geq \mathcal{O}$). 
As $\pi_{id}^1$ only needs to have support for at least one $a^1 \in \mathcal{A}^1_\textit{adm}$, we can restrict $\pi_{id}^1$ at step $t=0$ further to reduce the positivity assumption for step $t=2$. We do, however, only want to restrict $\pi_{id}^1$ as much as necessary, because if $\pi_{id}^1$ has wider support, this means that more data points are used in the analysis. 
Therefore, we introduce the notion of regional positivity and the regional admissible set $\mathcal{\tilde{A}}_\textit{adm}$ (from Definition \ref{def_regional_positivity}). 
In particular, Definition \ref{def_regional_positivity} defines $\mathcal{\tilde{A}}^1_\textit{adm}(x^0, a'^1)$ as the subset of $\mathcal{A}^1_\textit{adm}(x^0, a'^1)$ such that local positivity holds at step $t=1$ for all possible values of $x^1,$ and $a'^2$. As this has to hold for future time-steps as well (as will be shown next), the definition for $\mathcal{\tilde{A}}_\textit{adm}$ even states regional positivity has to hold recursively, i.e. also at $t=1$.

In summary, local positivity at step $t=2$ ensures that the available data allows the simulation of the currently desired action $a'^2$. Regional positivity at step $t=0$ ensures that only those simulations are used at step $t=0$ such that simulations of desired actions in the future (at step $t=1$) are possible with the data.

\vspace{10pt}
\noindent 
\textbf{Step t}

\noindent 
Now, we generalize the factorization to step $t$. 

\vspace{10pt}
\noindent 
\textit{Counterfactual factorization (step $t$, part 1):}
\begin{align} 
\label{eq_bellman_1_derivation}
p'   
\biggl( 
C'_{(\overline{\pi}^{t+1}_{\alpha})} &
 \Big\vert  
% \underline{X}'^{t-1}, 
\underline{A}'^{t}, 
\underline{X}^{t-1}, 
\underline{A}^{t-1} , 
\Xi
\biggl)
 = 
p'   
\biggl( 
C'_{(\overline{\pi}^{t+1}_{\alpha})} 
 \Big\vert  
% \underline{X}'^{t-1}, 
\underline{A}'^{t}, 
\underline{X}^{t-1}, 
\underline{A}^{t-1},
a^t , 
\Xi
\biggl) 
= \\ &
=   
\sum_{X'^t,X^t} 
p'
\left(   
C'_{(\overline{\pi}^{t+1}_{\alpha})} 
\Big\vert                   
% \underline{X}'^t,
\underline{A}'^t,
\underline{X}^t, 
\underline{A}^{t-1},
a^t, 
\Xi
\right) 
% \underbrace{g(
% X'^t|
% X^t,
% A'^{t})}
% _{\text{feature revelation}}
p(X^t|
\underline{X}^{t-1},
\underline{A}^{t-1},
a^t 
, 
\Xi) 
\nonumber
\end{align}  

\noindent 
which holds for any $a^t \in \mathcal{A}^t_\textit{adm}(\underline{X}^{t-1}, \underline{A}^{t-1}, A'^t)$ (because local positivity must hold). 
%The step contains the expansion by $A^1$, $X^1$, and $X'^1$, the variables needed for adjustment. 
Therefore, the term 
$p'
\left(   
C'_{(
\overline{\pi}^{t+1}_{\alpha}
)} 
\Big\vert                   
% \underline{X}'^t,
\underline{A}'^t,
\underline{X}^t, 
\underline{A}^{t-1}, 
%A^1 = 
a^t,  
\Xi
\right) $ needs to be only identified for \textit{at least one value} $a^t \in \mathcal{A}_\textit{adm}^t(\underline{X}^{t-1}, \underline{A}^{t-1}, A'^t)$. 
%Similarly as before, we have $p(X^t|
%\underline{X}^{t-1},
%\underline{A}^{t-1},a^t, 
% , 
%\Xi) = p(X_{(1), a^t}^t| \underline{X}_{(1), \underline{a}^{t-1}}^{t-1})$ where we let $\underline{X}_{(1), \underline{a}^{t}}^t$ denote $\{X^0, X^1_{(1), a^1}, ..., X^t_{(1), a^t}\}$. 

\vspace{10pt}
\noindent 
\textit{Counterfactual factorization (step $t$, part 2):}
\begin{align}    
\label{eq_bellman_2_derivation}
& p' \biggl(   
C'_{(
\overline{\pi}^{t+1}_{\alpha}
)} 
\Big\vert                   
% \underline{X}'^t,
\underline{A}'^t,
\underline{X}^t, 
\underline{A}^t,  
\Xi
\biggl) =
\\ = &
\nonumber
\boldsymbol{\sum_{ a'^{t+1}} }
p'
\left(   
C'_{\boldsymbol{(\overline{\pi}^{t+2}_{\alpha}, 
%A'^{t+1} = 
a'^{t+1})}} 
\Big\vert                   
% \underline{X}'^t,
\underline{A}'^t,
\underline{X}^t, 
\underline{A}^t, 
\Xi
\right) 
\boldsymbol{
\pi_\alpha^{t+1}(%A'^{t+1}= 
a'^{t+1} | 
\underline{X}'^t,
\underline{A}'^t) }  
\\ \overset{*_1}{=} &
\nonumber
\sum_{ a'^{t+1}} 
p'
\left(   
C'_{(\overline{\pi}^{t+2}_{\alpha}, 
%A'^{t+1} = 
a'^{t+1}, 
\boldsymbol{\pi_{id}^{t+1})}} 
\Big\vert                   
% \underline{X}'^t,
\underline{A}'^t,
\underline{X}^t, 
\underline{A}^t, 
\Xi
\right) 
\pi_\alpha^{t+1}(%A'^{t+1}= 
a'^{t+1} | 
\underline{X}'^t,
\underline{A}'^t)   
\\ =  &
\nonumber
\sum_{ 
\mathclap{a'^{t+1}, \boldsymbol{a^{t+1}}} }
p'
\left(   
C'_{(\overline{\pi}^{t+2}_{\alpha}, 
%A'^{t+1} = 
a'^{t+1}, 
\boldsymbol{a^{t+1}} )} 
\Big\vert                   
% \underline{X}'^t,
\underline{A}'^t,
\underline{X}^t, 
\underline{A}^t, 
\Xi
\right) 
% \\ \cdot & 
%\nonumber
%\quad 
\boldsymbol{\pi_{{id}}^{t+1}(%A^{t+1}= 
a^{t+1} | 
% \underline{X}'^t,
\underline{A}'^t,
a'^{t+1},
\underline{X}^t,
\underline{A}^t%, % = \underline{a}^{t*} , 
%A'^{t+1} = 
%a'^{t+1}
)  }
\pi_\alpha^{t+1}(%A'^{t+1}= 
a'^{t+1} | 
\underline{X}'^t,
\underline{A}'^t)    
\\ 
\overset{*_2}{=} &
\nonumber
\sum_{\mathclap{a'^{t+1}, a^{t+1}}} 
p'
\left(   
C'_{\boldsymbol{(\overline{\pi}^{t+2}_{\alpha}, 
%A'^{t+1} = 
a'^{t+1})}} 
\Big\vert                   
% \underline{X}'^t,
\underline{A}'^t,
\underline{X}^t, 
\underline{A}^t, 
\boldsymbol{a^{t+1}}%= \underline{a}^{t*}
, \Xi
\right) 
% \\ \cdot & 
%\nonumber
%\quad 
\pi_{{id}}^{t+1}(%A^{t+1}= 
a^{t+1} | 
% \underline{X}'^t,
\underline{A}'^t,
a'^{t+1},
\underline{X}^t,
\underline{A}^t
%A'^{t+1} = 
)  
\pi_\alpha^{t+1}(%A'^{t+1}= 
a'^{t+1} | 
\underline{X}'^t,
\underline{A}'^t)   
\\\overset{*_3}{=} &
\nonumber
\sum_{ \mathclap{a'^{t+1}, a^{t+1}}} 
p'
\left(   
C'_{\boldsymbol{(\overline{\pi}^{t+2}_{\alpha})}} 
\Big\vert                   
% \underline{X}'^t,
\underline{A}'^t,
%A'^{t+1}
%= 
\boldsymbol{a'^{t+1}},
\underline{X}^t, 
\underline{A}^t, 
a^{t+1}, 
\Xi
%=\underline{a}^{t*} 
\right) 
%\\ \cdot & 
%\nonumber
%\quad 
\pi_{{id}}^{t+1}(  %A^{t+1}= 
a^{t+1} | 
% \underline{X}'^t,
\underline{A}'^t,
a'^{t+1},
\underline{X}^t,
\underline{A}^t % = \underline{a}^{t*} , 
%A'^{t+1} = a'^{t+1}
)  
\pi_\alpha^{t+1}(%A'^{t+1}= 
a'^{t+1} | 
\underline{X}'^t,
\underline{A}'^t)    
\\\overset{*_4}{=} &
\nonumber
\sum_{\mathclap{a'^{t+1}, a^{t+1}}} 
p'
\left(   
C'_{(\overline{\pi}^{t+2}_{\alpha} )} 
\Big\vert                   
\boldsymbol{
%\underline{X}'^t,
\underline{A}'^t,
%A'^{t+1} = 
a'^{t+1},
\underline{X}^t, 
\underline{A}^t}, 
\Xi%= \underline{a}^{t*} 
\right) 
%\\ \cdot & 
%\nonumber
%\quad 
\underbrace{
\pi_{{id}}^{t+1}(%A^{t+1}= 
a^{t+1} | 
% \underline{X}'^t,
\underline{A}'^t,
a'^{t+1},
\underline{X}^t,
\underline{A}^t ) %= \underline{a}^{t*} , A'^{t+1} = a'^{t+1})  
}_{\text{arbitrary dist. subject to constraints}} 
\underbrace{
\pi_\alpha^{t+1}(%A'^{t+1}= 
a'^{t+1} | 
\underline{X}'^t,
\underline{A}'^t) 
}_{\text{target policy}}  
\end{align}

\noindent 
with the following explanations:

\begin{itemize}
    \item $*1)$: We use that  $C'_{(\overline{\pi}^{t+2}_\alpha,a'^{t+1})} $ is independent of any interventions $\pi_{id}^{t+1}$ on $A^{t+1}$.
    \item $*2)$: 
    We use exchangeability :
    $C'_{(\overline{\pi}^{t+2}_{\alpha},%A'^1=
    a'^{t+1}, a^{t+1})} \perp\!\!\!\perp  
    A^{t+1} |
    %\underline{X}'^t,
    \underline{A}'^t,\underline{X}^t,\underline{A}^t,\Xi $
    and consistency: \newline $p'
    \left(  
    C'_{(\overline{\pi}^{t+2}_{\alpha}, 
    a'^{t+1}, a^{t+1})}     
    \Big\vert                   % \underline{X}'^t,
    \underline{A}'^t,
    \underline{X}^t,
    \underline{A}^t,
    a^{t+1}, 
    \Xi
    \right) 
    = 
    p'
    \left(  
    C'_{(\overline{\pi}^{t+2}_{\alpha}, 
    a'^{t+1})}     
    \Big\vert                   
    % \underline{X}'^t,
    \underline{A}'^t,
    \underline{X}^t,
    \underline{A}^t, 
    a^{t+1}, 
    \Xi
    \right)$ for $A^{t+1}$. 
    \item $*3)$: We use the exchangeability
    $C'_{(\overline{\pi}^{t+2}_{\alpha}, a'^{t+1})} \perp\!\!\!\perp  
    A'^{t+1} |
    % \underline{X}'^t,
    \underline{A}'^t,
    \underline{X}^t,
    \underline{A}^t, 
    a^{t+1}, 
    \Xi$ and consistency:
    $ p'
    \left(  
    C'_{(\overline{\pi}^{t+2}_{\alpha}, 
    a'^{t+1})}     
    \Big\vert                   
    % \underline{X}'^t ,
    \underline{A}'^t ,
    %A'^2 =  
    a'^{t+1},
    \underline{X}^t ,
    \underline{A}^t, 
    a^{t+1}, 
    \Xi
    \right) 
    = 
    p'
    \left(  
    C'_{(\overline{\pi}^{t+2}_{\alpha})}    
    \Big\vert                   
    % \underline{X}'^t ,
    \underline{A}'^t ,
    %A'^2 =  
    a'^{t+1},
    \underline{X}^t ,
    \underline{A}^t,
    a^{t+1}, 
    \Xi
    \right) $ for $A'^{t+1}$.
    \item $*4)$: We use the conditional independence 
    $C'_{(\overline{\pi}^{t+2}_{\alpha})} \perp\!\!\!\perp  
    A^{t+1} |
    % \underline{X}'^t,
    \underline{A}'^t,
    a'^{t+1},
    \underline{X}^t,
    \underline{A}^t, 
    \Xi$
\end{itemize}

\noindent 
As before, we have to make sure in $*3)$ that $p'
    \left(  
    C'_{(\overline{\pi}^{t+2}_{\alpha}, 
    a'^{t+1})}     
    \Big\vert                   
    % \underline{X}'^t ,
    \underline{A}'^t ,
    %A'^2 =  
    a'^{t+1},
    \underline{X}^t ,
    \underline{A}^t, 
    a^{t+1}, 
    \Xi
    \right)$, 
i.e. conditioning on $   %  \underline{X}'^{t} ,
    \underline{A}'^t ,
    %A'^2 =  
    a'^{t+1},
    \underline{X}^t ,
    \underline{A}^t, a^{t+1}, 
    \Xi$, is well specified. 
To understand what positivity requirements are necessary, we factorize the "observational" (i.e. simulated) distribution for step $t$:

\vspace{10pt}
\noindent 
\textit{Observational factorization (step $t$):}
\begin{align*}  
p'   
\biggl( 
C'_{(\overline{\pi}^{t+2}_{\alpha})}
\Big\vert               
% \underline{X}'^{t-1} ,
\underline{A}'^t ,
\underline{X}^{t-1} ,&  
\underline{A}^t , 
\Xi
%= \underline{a}^{t*} 
\biggl)   
%=  
% \\ 
%= & 
= 
\sum_{X^t,a^{t+1},a'^{t+1}} 
p'
\left(   
C'_{(   
\overline{\pi}^{t+2}_{\alpha})
}        
\Big\vert                 
% \underline{X}'^t ,
\underline{A}'^t ,
%A'^{t+1} = 
a'^{t+1},
\underline{X}^t ,
\underline{A}^t, 
\Xi
%=  \underline{a}^{t*} 
\right) 
\\ \cdot & 
\underbrace{
\pi_{sim}'^{t+1}(
%A'^{t+1} = 
a'^{t+1} | 
\underline{X}'^t,
\underline{A}'^t, 
a^{t+1}
%= a^{t+1}
)
}  
_{\text{known simulation policy}}   
\underbrace{
\pi_{\beta}^{t+1}(
%A^{t+1} = 
a^{t+1} | 
\underline{X}^t,
\underline{A}^t%= \underline{a}^{t*}
) 
}_{\text{retro. acquisition policy}}  
%\underbrace{g(
%X'^t|
%X^t,
%A'^{t})}
%_{\text{feature revelation}}
p(X^t|
\underline{X}^{t-1},
\underline{A}^{t}, 
\Xi) 
\end{align*}

\noindent 
By comparing the observational and counterfactual factorizations, we see that the following positivity assumption is required: 
\begin{flalign*}
\nonumber
\text{if }  \hspace{14 pt} \quad \quad
& q'(
%\underline{x}'^t,
\underline{a}'^t, 
%A'^2 = 
a'^{t+1}, 
\underline{x}^t, 
\underline{a}^t,
%A^2 = 
a^{t+1}) = 
\\ & \quad \quad \quad \quad  = 
 q'(
%\underline{x}'^t,
\underline{a}'^t,
\underline{x}^t, 
\underline{a}^t) 
%\\ & \cdot 
\pi_\alpha^{t+1}(%A'^2= 
a'^{t+1} | 
\underline{x}'^t,
\underline{a}'^t) 
\pi_{{id}}^{t+1}(%A^2= 
a^{t+1} | 
%\underline{x}'^t,
\underline{a}'^t,
a'^{t+1},
\underline{x}^t,
\underline{a}^t)   
>0
\nonumber
&&
\\
\text{then } \quad \quad
& 
p'(
% \underline{x}'^t, 
%A'^2 = 
a'^{t+1}, 
\underline{x}^t, 
\underline{a}^t,
%A^2 =
a^{t+1})  = 
\nonumber 
\\ & \quad \quad \quad \quad
= 
p'(
\underline{a}'^t,
\underline{x}^t, 
\underline{a}^t)  
\pi_{sim}'^{t+1}(
%A'^2 = 
a'^{t+1} | 
\underline{x}'^t,
\underline{a}'^t, 
%A^2 = 
a^{t+1})   
\pi_{\beta}^{t+1}(
%A^2 =
a^{t+1} | 
\underline{x}^t,
\underline{a}^t) 
\geq \mathcal{O}
\nonumber
&&
\\ 
&
\forall 
\underline{x}^t, 
\underline{a}'^t,
\underline{a}^t,
a'^{t+1}, 
a^{t+1}, \text{ and some constant }\mathcal{O} > 0
\end{flalign*}
with the following factorizations: 
\begin{align*}
q'(
% \underline{X}'^t,
\underline{A}'^t,
\underline{X}^t, 
\underline{A}^t)  
& = 
q'(
% \underline{X}'^{t-1},
\underline{A}'^t,
\underline{X}^{t-1},
\underline{A}^t)
% g(
% X'^t|X^t, A'^t
% )  
p(
X^t|
\underline{X}^{t-1},
\underline{A}^t
)  
\\
p'(
% \underline{X}'^t,
\underline{A}'^t,
\underline{X}^t, 
\underline{A}^t)  
& = 
p'(
% \underline{X}'^{t-1},
\underline{A}'^t,
\underline{X}^{t-1},
\underline{A}^t)
% g(
% X'^t|X^t, A'^t
% )  
p(
X^t|
\underline{X}^{t-1},
\underline{A}^t
)  
\end{align*}

\noindent 
The positivity condition can again be simplified through the two restrictions on $\pi_{id}$:

\vspace{10pt}
\noindent
\textit{Restrictions 1 and 2 for $\pi_{{id}}$ (step $t$):}
\begin{flalign*}
\text{if }  \hspace{14 pt} \quad \quad \quad \quad \quad  \quad \quad \quad \quad 
&a'^{t+1}\not \leq  a^{t+1} \text{ or } \pi_{\beta}^{t+1}(
%A^2 = 
a^{t+1} | 
\underline{x}^t,
\underline{a}^t) = 0,
\nonumber
&&
\\
\text{then } \quad \quad \quad \quad \quad  \quad \quad \quad \quad 
&
\pi_{{id}}^{t+1}(
%A^2 = 
a^{t+1} | 
% \underline{x}'^t,
\underline{a}'^t, 
a'^{t+1},
\underline{x}^t,
\underline{a}^t 
%A'^2=
%a'^2
) = 0 
\nonumber
&&
\\ 
&
\forall  
a^{t+1}, a'^{t+1}
\end{flalign*}

\noindent 
This imposes the requirement for $\pi_\beta$ that there exists at least one value $a^{t+1}$ such that $ a'^{t+1} \leq  a^{t+1}$ and $
\pi_{\beta}^{t+1}(
a^{t+1} | 
\underline{x}^t, 
\underline{a}^t) \geq \mathcal{O}$ (i.e. local positivity at $\underline{x}^t, 
\underline{a}^t, a'^t$). 
This has to hold for \textit{all values} $\underline{a}^{t}$
that were "allowed" in \text{all} previous steps (i.e.  all $a^t$, for all $\tau \leq t$, s.t. $\pi_{id}^t(a^\tau|\underline{a}'^{\tau-1}, a'^\tau,\underline{x}^{\tau-1},
 \underline{a}^{\tau-1}) > 0$ and which could later on have let to the current state). 
As $\pi_{id}$ only needs to have support for at least one $a^t \in \mathcal{A}^t_{\textit{adm}}$ per step, we can restrict $\pi_{id}$ at \textit{all} previous steps to reduce the positivity assumption for step $t$. Again, we do not want to restrict $\pi_{id}$ too much, because if $\pi_{id}$ has wider support, this means that more data points are used in the analysis. 
The regional positivity assumption (from Definition \ref{def_regional_positivity}) ensures in this case that only those simulations  are used (and exist) at all previous steps such that simulations of the desired actions can be performed at step $t$ (and for future steps).

\vspace{10pt}
\noindent
\textbf{Full factorization}

\noindent
Bringing all time-steps $t = 0, ... , T$ together and including $Y$ (as a part of $\Xi$) one obtains the full factorization of the identifying distribution $q'$:
\begin{align*} 
q' & (A', A, X, Y)  =   
\prod_{t=1}^{T}  
\underbrace{\pi_{id}^t(A^{t}| 
\underline{A}'^{t},
\underline{X}^{t-1},
\underline{A}^{t-1})}_{\text{arb. distr. subject to constraints}}
%\\ & \cdot
\underbrace{\pi_\alpha^t(A'^t|             
\underline{X}'^{t-1},             
\underline{A}'^{t-1})}_{\text{target policy}}  
%\\ & \cdot 
\prod_{t=0}^{T}
p(X^t|   
\underline{X}^{t-1},        
\underline{A}^{t}, Y) 
p(Y)
\\
&
=   
\prod_{t=1}^{T}  
\pi_{id}^t(A^{t}| 
\underline{A}'^{t},
%\underline{X}^{t-1}
G_{\underline{A}^{t-1}}(X_{(1)}),
\underline{A}^{t-1})
%\\ 
%& \cdot
\pi_\alpha^t(A'^t|             
G_{\underline{A}'^{t-1}}(X_{(1)}),             
\underline{A}'^{t-1})
\\ 
& \quad \cdot
\prod_{t=0}^{T}
p(G_{A^{t}}(X_{(1)})|   
G_{\underline{A}^{t-1}}(X_{(1)}),        
\underline{A}^{t}, Y) 
p(Y)
\\
& 
= q'(A',A , G_A(X_{(1)}), Y)
\end{align*}

\noindent 
In order for this expression to hold, $\pi_{id}^t$ must be restricted to have support only on $\mathcal{\tilde{A}}_\textit{adm}$. This leads to the following restriction: 
\begin{align*}
\pi_{id}^t(A^t|
%\underline{X}'^{t-1},
\underline{A}'^t,
\underline{X}^{t-1},
\underline{A}^{t-1}) 
= 
\underbrace{\mathbb{I}(A^t \in 
\mathcal{
\tilde{A}}_\textit{adm}^t(
%\underline{X}'^{t-1},
\underline{A}'^t,
\underline{X}^{t-1},
\underline{A}^{t-1}
))}
_{\text{support restriction}}
f_{id}^t(
%\underline{X}'^{t-1},
\underline{A}'^t,
\underline{X}^{t-1},
\underline{A}^{t-1}
). 
\end{align*}
\noindent 
where $f_{id}^t$ is an arbitrary function that ensures that $\pi_{id}^t$ is a valid density. 
This concludes the proof of Theorem \ref{theorem_identification_semi_offline_RL}.

\vspace{10pt}
\noindent 
\textbf{Bellman equation}

\noindent 
Equations \ref{eq_bellman_1_derivation} and \ref{eq_bellman_2_derivation} correspond to the two parts of the semi-offline RL version of the Bellman equation:
\begin{align*} 
\mathbb{E}   &
\biggl[ 
C'_{(\overline{\pi}^{t+1}_{\alpha})}  
 \Big\vert  
% \underline{X}'^{t-1}, 
\underline{A}'^{t}, 
\underline{X}^{t-1}, 
\underline{A}^{t-1}, 
\Xi
\biggl]
%&
%= \\ 
=
\sum_{X^t} 
\mathbb{E}
\left[   
C'_{(\overline{\pi}^{t+1}_{\alpha})} 
\Big\vert                   
% \underline{X}'^t,
\underline{A}'^t,
\underline{X}^t, 
\underline{A}^{t-1},
a^t, 
\Xi
% \underline{A}^t =  \underline{a}^{t*} 
\right]
%g(X'^t
%|X^t,A'^t) 
p(X^t|
\underline{X}^{t-1},
\underline{A}^{t-1},
a^t, 
\Xi) 
\nonumber  
\\ 
& = 
\sum_{G_{A^t}(X_{(1)})} 
\mathbb{E}
\left[   
C'_{(\overline{\pi}^{t+1}_{\alpha})} 
\Big\vert                   
% \underline{X}'^t,
\underline{A}'^t,
G_{\underline{A}^t}(X_{(1)}),
% \underline{X}^t, 
\underline{A}^{t-1},
a^t, 
\Xi
% \underline{A}^t =  \underline{a}^{t*} 
\right]
%g(X'^t
%|X^t,A'^t) 
p(
G_{A^t}(X_{(1)})|
G_{\underline{A}^{t-1}}(X_{(1)}),
\underline{A}^{t-1},
a^t, 
\Xi) 
\nonumber  
\\
& \text{ for any }
a^t 
\in \mathcal{A}^t_\textit{adm}
(          
\underline{X}^{t-1}, 
\underline{A}^{t-1},A'^t )  
\nonumber 
\end{align*}

\begin{align*}
\mathbb{E} 
\biggl[   
C'_{(
\overline{\pi}^{t+1}_{\alpha}
)} 
\Big\vert                   
%\underline{X}'^t,
\underline{A}'^t,&
\underline{X}^t,  
\underline{A}^t , 
\Xi
\biggl] 
%=
%\\
=   
\nonumber
\sum_{
A'^{t+1}
%, 
%a^{t+1}
} 
\mathbb{E}
\left[   
C'_{(\overline{\pi}^{t+2}_{\alpha} )} 
\Big\vert                   
% \underline{X}'^t,
\underline{A}'^{t+1},
%A'^{t+1} = 
%a'^{t+1},
\underline{X}^t, 
\underline{A}^t, 
% = \underline{a}^{t*} ,
\Xi
\right]
%\\ \cdot & 
%\nonumber
%\quad 
%\underbrace{
%\pi_{{id}}(A^{t+1}= a^{t+1} | 
%\underline{X}^t,
%\underline{A}^t= \underline{a}^{t*} , A'^{t+1} = a'^{t+1})  
%}_{\text{arbitrary intervention}} 
\pi_\alpha^{t+1}(A'^{t+1}%= a'^{t+1} 
| 
\underline{X}'^t,
\underline{A}'^t) 
\\
& =   
\nonumber
\sum_{
A'^{t+1}
%, 
%a^{t+1}
} 
\mathbb{E}
\left[   
C'_{(\overline{\pi}^{t+2}_{\alpha} )} 
\Big\vert                   
% \underline{X}'^t,
\underline{A}'^{t+1},
%A'^{t+1} = 
%a'^{t+1},
G_{\underline{A}^{t}}(X_{(1)}),
% \underline{X}^t, 
\underline{A}^t, 
% = \underline{a}^{t*} ,
\Xi
\right]
%\\ \cdot & 
%\nonumber
%\quad 
%\underbrace{
%\pi_{{id}}(A^{t+1}= a^{t+1} | 
%\underline{X}^t,
%\underline{A}^t= \underline{a}^{t*} , A'^{t+1} = a'^{t+1})  
%}_{\text{arbitrary intervention}} 
\pi_\alpha^{t+1}(A'^{t+1}%= a'^{t+1} 
| 
G_{\underline{A}'^{t-1}}(X_{(1)}),
% \underline{X}'^t,
\underline{A}'^t) 
\end{align*}

\noindent 
The factorization holds under local positivity (if $\mathcal{A}^t_\textit{adm} \neq \emptyset$ exists). Furthermore, the individual terms are identified if regional positivity holds which concludes the proof of Theorem \ref{theorem_Bellman_equation}. 
% \hfill % \qedsymbol 
\end{proof}

\section{Proof of Corollary \ref{corollary_identification_semi_offline_RL}   }
\label{app_proof_corollary_identification_semi_offline_RL}

In this appendix, we prove Corollary \ref{corollary_identification_semi_offline_RL}, stating identification under the maximal global positivity assumption. We repeat it here for ease of reference: 

\vspace{5pt}
\begin{cora}{\ref{corollary_identification_semi_offline_RL}}
(Identification of $J$ for the semi-offline RL view under maximal global positivity).
The reformulated AFAPE problem of estimating $J$ under the semi-offline RL view  (Eq. \ref{eq:AFAPE_objective_semi_offline_RL}) is under
Assumption \ref{assump:measurement_noise} (no measurement noise),  Assumption \ref{assump:consistency} (consistency),  Assumption \ref{assump:interference} (no interference), Assumption \ref{assump:nde} (NDE), Assumption \ref{assump:nuc} (NUC) and Assumption \ref{assump:max_positivity_global_semi_offline_RL} (maximum global positivity)
identified
by Eqs. \ref{eq_identificiation_semi_offline_RL_1} and \ref{eq_identificiation_semi_offline_RL_2} where 
\begin{align*}
\pi_{id}^t(A^t|
% \underline{X}'^{t-1},
\underline{A}'^{t-1},
A'^t = a'^t,&
\underline{X}^{t-1},
\underline{A}^{t-1}
) 
 = \\ & = 
\mathbb{I}(A^t \geq a'^t)
\pi_\beta^t(
A^t
|
\underline{X}^{t-1}, 
\underline{A}^{t-1}
)  
f_{id}^t(
% \underline{X}'^{t-1},
\underline{A}'^{t-1},
A'^t = a'^t,
\underline{X}^{t-1},
\underline{A}^{t-1}
) 
\end{align*}
for any function $f_{id}^t$
s.t. $\pi_{id}^t$ is a valid density. This holds in particular for the choice of a truncated $\pi_\beta$: 
\begin{align*}
\pi_{id}^t(A^t|
%\underline{X}'^{t-1},
\underline{A}'^{t-1},
A'^t = a'^t,
\underline{X}^{t-1},
\underline{A}^{t-1}
) 
& = 
\pi_\beta^t(A^t
|
A^t \geq a'^t, 
\underline{X}^{t-1},
\underline{A}^{t-1}
)
=\\& = 
\frac{\mathbb{I}(A^t \geq a'^t)
\pi_\beta^t(A^t
|
\underline{X}^{t-1},
\underline{A}^{t-1}
)}{
\pi_\beta^t(A^t\geq a'^t
|
\underline{X}^{t-1},
\underline{A}^{t-1}
)}.
\end{align*}
\end{cora}

\vspace{10pt}
\begin{proof}
Under the maximal global positivity assumption, we have 
$
\mathcal{\tilde{A}}_\textit{adm}^t
(
%\underline{x}'^{t-1},
\underline{a}'^{t},
\underline{x}^{t-1},
\underline{a}^{t-1})
=  
\mathcal{A}_\textit{adm}^t(
\underline{x}^{t-1},
\underline{a}^{t-1},
a'^{t})$. 
We can now insert this assumption into Eq. \ref{eq_pi_id} which states the identificiation of $J$ under global positivity: 

\begin{align*}
\pi_{id}^t(A^t|
%\underline{X}'^{t-1},
\underline{A}'^t,
\underline{X}^{t-1},
\underline{A}^{t-1}) 
& = 
\mathbb{I}(A^t \in 
\mathcal{
\tilde{A}}_\textit{adm}^t(
%\underline{X}'^{t-1},
\underline{A}'^t,
\underline{X}^{t-1},
\underline{A}^{t-1}
))
f_{id}^t(
%\underline{X}'^{t-1},
\underline{A}'^t,
\underline{X}^{t-1},
\underline{A}^{t-1}
)
\\
& = 
\mathbb{I}(A^t \in 
\mathcal{
{A}}_\textit{adm}^t(
\underline{X}^{t-1},
\underline{A}^{t-1},
A'^t
))
f_{id}^t(
%\underline{X}'^{t-1},
\underline{A}'^t,
\underline{X}^{t-1},
\underline{A}^{t-1}
)
\\ & 
= 
\mathbb{I}(A^t \geq a'^t)
\pi_\beta^t(
A^t
|
\underline{X}^{t-1}, 
\underline{A}^{t-1}
) 
f_{id}^{*t}(
\underline{X}'^{t-1},
\underline{A}'^t,
\underline{X}^{t-1},
\underline{A}^{t-1}
) 
\end{align*}

\noindent 
where we let $f_{id}^{*t}$ denote another arbitrary function that ensures that $\pi_{id}^t$ is a valid density. 
This concludes the proof for Corollary \ref{corollary_identification_semi_offline_RL}. 
    
% \hfill %\qedsymbol 
\end{proof}

\section{Comparison of the Semi-offline RL IPW Estimator with Related Methods}
\label{app_comparison_caniglia}

In this appendix, we demonstrate that our proposed IPW estimator $J_\textit{IPW-Semi}$ is a more general version of an adapted version of the IPW estimator introduced by \cite{caniglia_emulating_2019}. We refer to this estimator as $\hat{J}_\textit{IPW-Cen}$ since it is derived from a censoring viewpoint.
While $\hat{J}_\textit{IPW-Cen}$ was developed for a scenario where both feature acquisition decisions and treatment decisions are made by the agent, it can be adapted to the AFA setting. 
However, $\hat{J}_\textit{IPW-Cen}$ is only applicable to simpler settings with one acquisition option per time-point ($A^t \in \{0,1\}$). $\hat{J}_\textit{IPW-Cen}$ is also only consistent if the maximal global positivity assumption holds, as will be shown. 

\cite{caniglia_emulating_2019} derived $J_\textit{IPW-Cen}$ under the NDE and NUC assumptions. 
Instead of using a semi-offline sampling policy that avoids the acquisition of non-available features as proposed in this paper, \cite{caniglia_emulating_2019} simply sample from $\pi_{\alpha}$, even without knowledge about $X_{(1)}$. As the feature revelation is not possible if a non-available feature is acquired, they treat the resulting trajectory as censored. Known missing data methods are then applied to adjust for this censoring. 
Hence, in the wording of this paper, we would describe this viewpoint as an \textit{online RL + censoring viewpoint}. 

Adapted to the AFA setting under the consideration of deterministic AFA policies $\pi_\alpha$, $\hat{J}_\textit{IPW-Cen}$ becomes:   
\begin{align}
    \hat{J}_\textit{IPW-Cen} = %\mathbb{E}_{\mathcal{D}'} 
    \hat{\mathbb{E}}_{n, \textit{uncen}}
    [\hat{\rho}^T_\textit{Cen} C_{(\pi_\alpha)}]
    \text{ where }    
    \hat{\rho}^T_\textit{Cen} & = 
    \prod_{t = 1}^T 
    \left(
    \frac{
    \mathbb{I}(A^t = 1) 
    }
    {\hat{\pi}_{\beta}^t(A^t = 1| 
    \underline{X}^{t-1},
    \underline{A}^{t-1})
    }\right)^{A_{(\pi_\alpha)}^t}
\end{align}

\noindent 
where $\hat{\mathbb{E}}_{n', \textit{uncen}}[.]$ denotes the empirical average over the uncensored data points which have the known deterministic counterfactuals $A_{(\pi_\alpha)}^t$,$X_{(\pi_\alpha)}^t$ and $C_{(\pi_\alpha)}$.

Since $A^t \in \{0,1\}$, it can be observed that the propensity score for a specific time-point $t$ only appears in the factorization if the corresponding action is $A_{(\pi_\alpha)}^t = 1$. An "acquire nothing" AFA policy  (where 
$\pi_{\alpha}^t(A^t = 0| \underline{X}^{t-1}, \underline{A}^{t-1}) = 1$ $\forall t$) would thus require no adjustment ($\rho^T_\textit{Cen} = 1$). In their example, this estimator achieved a 50-fold increase in data efficiency compared to the standard offline RL IPW estimator \cite{caniglia_emulating_2019}. 

We establish the equivalence of our estimator $\hat{J}_\textit{IPW-Semi}$ and $\hat{J}_\textit{IPW-Cen}$ in the following proposition:

\begin{proposition}
\label{proposition_equivalence}
(Equivalance of $\hat{J}_\textit{IPW-Cen}$ and $\hat{J}_\textit{IPW-Semi}$).
The estimators $\hat{J}_\textit{IPW-Cen}$ and $\hat{J}_\textit{IPW-Semi}$ are equivalent for AFA settings with one action option per time-point, deterministic AFA policies $\pi_\alpha$, the maximal global positivity assumption (Assumption \ref{assump:max_positivity_global_semi_offline_RL}), and a simulation policy $\pi_{sim} = \pi_{\alpha}$.
\end{proposition}

\begin{proof}
Firstly, we clarify the blocking operation (from Definition \ref{def_blocked_policy}) for this setting:
\begin{align*}
\pi_{sim}'(A'^t=a'^t|  
\underline{X}'^{t-1}, 
\underline{A}'^{t-1},   A^{t}=a^t) 
& = 
\begin{cases}
    1, & \text{if $a'^{t}=0$ \& $a^{t}=0$}\\
    0, & \text{if $a'^{t}=1$ \& $a^{t}=0$}\\
    \pi_{\alpha}^t(A'^t = a'^t|  
    \underline{X}'^{t-1}, 
    \underline{A}'^{t-1}) , & \text{if $a^{t}=1$}.
\end{cases}
\end{align*}

\noindent 
The inverse probability weights of $\hat{J}_\textit{IPW-Semi}$ become: 
\begin{align*} 
    \rho^T_{\textit{Semi}} & = 
    \prod_{t=1}^T 
    \frac{
    \pi_\alpha^t(A'^{t}=a'^t| \underline{X}'^{t-1}, \underline{A}'^{t-1}) 
    }
    {
    \pi'^t_{sim}(
    A'^{t}=a'^t| \underline{X}'^{t-1}, \underline{A}'^{t-1},
    A^t=a^t
    )
    }
    \frac{
    \mathbb{I}(A^t \geq a'^t)
    }
    {
    \pi_\beta^t(
    A^t \geq a'^t| 
    \underline{X}^{t-1},
    \underline{A}^{t-1}
    )
    }
     \\ & \overset{*_1}{=} 
    \prod_{t=1}^T 
    \left(
    \pi_\alpha^t(A'^{t}=a'^t| \underline{X}'^{t-1}, \underline{A}'^{t-1}) 
    \right)^{1-a^t}
    \left(\frac{
    \mathbb{I}(A^t = 1) 
    }
    {
    \pi_{\beta}^t
    (A^t = 1| 
    \underline{X}^{t-1},
    \underline{A}^{t-1})
    }
    \right)^{a'^t}
     \\ & \overset{*_2}{=} 
    \prod_{t=1}^T 
    \mathbb{I}(A'^t = A^t_{(\pi_\alpha)}) 
    \left(
    \frac{
    \mathbb{I}(A^t = 1) 
    }
    {
    \pi_{\beta}^t
    (A^t = 1| 
    \underline{X}^{t-1},
    \underline{A}^{t-1})
    }\right)^{A^t_{(\pi_\alpha)}}
\end{align*}

\noindent 
where we used in $*1)$ the above definition of $\pi'_{sim}$ and that $\mathbb{I}(A^t \geq 0)= 1 = \pi_\beta(
A^t \geq 0| \underline{X}^{t-1},\underline{A}^{t-1})$. 
In $*2)$, we see that the first term corresponds to whether $\pi_\alpha$ could be applied without running into censoring. It thus gives 0 weights to all datapoints where blocking occured (i.e. which are censored under the online RL + censoring viewpoint). The second term then corresponds to the same weights as $\rho^T_{Cen}$ which concludes the proof for Proposition \ref{proposition_equivalence}. 
%\hfill %\qedsymbol 
\end{proof}

\vspace{5pt}
We have demonstrated that, although the two estimators are derived from different concepts (online RL with censoring vs semi-offline RL), they are equal in this specific AFA setting of one action option per time-step, deterministic policies and under the maximal global positivity assumption. However, the key distinction lies in the generality of our estimator. Unlike $\hat{J}_\textit{IPW-Cen}$, which is limited to the described setting, we developed an IPW estimator that can be applied for multiple acquisition options (i.e. higher dimensional $A^t$), under the weaker global positivity assumption and in a modified version for static features settings as we show in our companion paper \cite{von_kleist_evaluation_2023-1}. %\{[cite companions paper]}. 
It can further be combined with a Q-model to build the DRL estimator.

\section{Proof of Theorem \ref{theorem_double_robustness}}
\label{app_theorem_double_robustness}

In this appendix, we proof Theorem \ref{theorem_double_robustness}, which we repeat here: 

\vspace{5pt}
\begin{thma}{\ref{theorem_double_robustness}}
(Double robustness of $\hat{J}_{\textit{DRL-Semi}}$).
The estimator $\hat{J}_{\textit{DRL-Semi}}$ is %RAL (regular and asymptotically linear) and
doubly robust, in the sense that it is consistent if either the Q-function $\hat{Q}_\textit{Semi}$ or the propensity score model $\hat{\pi}_\beta$ is correctly specified. 
\end{thma}

\vspace{10pt}
\begin{proof}
 To prove the double robustness property of the semi-offline RL version of the DRL estimator (i.e. Theorem \ref{theorem_double_robustness}), we decompose 
$\hat{J}_{\textit{DRL-Semi}}$ in two different ways:

\vspace{10pt}
\noindent 
\textit{Scenario 1:} If $\hat{\pi}_\beta$ is correctly specified, we find
\begin{align*}                  
\hat{J}_{\textit{DRL-Semi}} 
& = 
\hat{\mathbb{E}}_{n}
     \left[
     \hat{\mathbb{E}}_{n'}
    \left[
    \rho_\textit{Semi}^T C' + 
 \sum_{t=1}^{T} 
 \left(- \hat{\rho}_{\textit{Semi}}^{t}  
 \hat{Q}_\textit{Semi}^t
 %(\underline{X}'^{t-1},
 %\underline{A}'^{t}, 
 %\underline{X}^{t-1}, 
 %\underline{A}^{t-1}) 
 +
\rho_{\textit{Semi}}^{t-1}  
 \hat{V}_\textit{Semi}^{t-1}
 %(\underline{X}'^{t-1},
 %\underline{A}'^{t-1}, 
 %\underline{X}^{t-1}, 
 %\underline{A}^{t-1})
 \right) 
 \Big| 
 A,X,Y
 \right]
 \right]
\\ & = 
\underbrace{
\mathbb{E}[\rho_\textit{Semi}^T 
C'] }_{=J}  
+ 
\sum_{t=1}^{T} 
\underbrace{\mathbb{E}
\left[
-\rho_{\textit{Semi}}^{t}  
  \hat{Q}_{\textit{Semi}}^{t}
 +
 \rho_{\textit{Semi}}^{t-1}  
  \hat{V}_{\textit{Semi}}^{t-1}
 \right]}_{=0},  
\end{align*} 

\noindent 
where the first term is just the IPW estimator. As $\hat{\pi}_\beta =\pi_\beta $ is correctly specified, the IPW estimator consistently estimates $J$. 
The fact that the second term equals 0 is shown in the following: 
\begin{align*} 
    \mathbb{E}
    & 
    \left[ 
    -\rho_{\textit{Semi}}^{t}  
     \hat{Q}_{\textit{Semi}}^{t}
     +
     \rho_{\textit{Semi}}^{t-1}  
      \hat{V}_{\textit{Semi}}^{t-1}
     \right] 
%2
= 
\\ & 
\overset{*_1}{=} 
     \mathbb{E}
    \biggl[
    \rho_{\textit{Semi}}^{\boldsymbol{t-1}}
    \boldsymbol{ \biggl(
     - 
    \frac{
    \pi_\alpha^{t}(A'^{t}| \underline{X}'^{t-1}, \underline{A}'^{t-1}) 
    }
    {
    \pi'^{t}_{sim}(
    A'^{t}| \underline{X}'^{t-1}, \underline{A}'^{t-1},
    A^t
    )
    }
    \frac{
    \pi_{id}^{t}(
    A^{t} | 
    %\underline{X}'^{t-1},
    \underline{A}'^{t},
    \underline{X}^{t-1},
    \underline{A}^{t-1})
    }
    {
    \pi_\beta^{t}(
    A^t| 
    \underline{X}^{t-1},
    \underline{A}^{t-1}
    )
    }}
     \hat{Q}_{\textit{Semi}}^t
%2.2
 \\ & \quad \quad  \quad \quad \quad  \quad  +
     \boldsymbol{\sum_{A'^t}
      \pi_{\alpha}^{t}(A'^t|
     \underline{X}'^{t-1},
     \underline{A}'^{t-1})
     \hat{Q}_{\textbf{\textit{Semi}}}^t
     \biggl)}
     \biggl] 
%3
\\ & 
     \overset{*_2}{=} 
     \mathbb{E}
     \biggl[
    \rho_{\textit{Semi}}^{t-1}
      \biggl(
     - \boldsymbol{\sum_{A'^t,A^t} 
   \pi'^{t}_{sim}(
    A'^{t}| \underline{X}'^{t-1}, \underline{A}'^{t-1},
    A^t
    )
    \pi_\beta^{t}(
    A^t| 
    \underline{X}^{t-1},
    \underline{A}^{t-1}
    )
    }
    \frac{
    \pi_\alpha^{t}(A'^{t}| \underline{X}'^{t-1}, \underline{A}'^{t-1}) 
    }
    {
    \pi'^{t}_{sim}(
    A'^{t}| \underline{X}'^{t-1}, \underline{A}'^{t-1},
    A^t
    )
    }
%3.2
\cdot  \\ & \quad \quad \quad \quad \quad \quad \cdot
    %\\ & \cdot
    \frac{
    \pi_{id}^{t}(
    A^{t} | 
    %\underline{X}'^{t-1},
    \underline{A}'^{t},
    \underline{X}^{t-1},
    \underline{A}^{t-1})
    }
    {
    \pi_\beta^t(
    A^t| 
    \underline{X}^{t-1},
    \underline{A}^{t-1}
    )
    }
     \hat{Q}_{\textit{Semi}}^t +
    % + \\ & \quad \quad \quad \quad \quad \quad + 
     \sum_{A'^t}
      \pi_{\alpha}^{t}(A'^t|
     \underline{X}'^{t-1},
     \underline{A}'^{t-1})
     \hat{Q}_{\textit{Semi}}^t
     \biggl)
      \biggl] 
% 4
\\ & 
     \overset{*_3}{=} 
     \mathbb{E}
     \biggl[
    \rho_{\textit{Semi}}^{t-1}
      \biggl(
     - \boldsymbol{\sum_{A'^t} 
     \pi_{\alpha}^{t}(A'^t|
     \underline{X}'^{t-1},
     \underline{A}'^{t-1})}
     \hat{Q}_{\textit{Semi}}^t
      + 
     \sum_{A'^t}
      \pi_{\alpha}^{t}(A'^t|
     \underline{X}'^{t-1},
     \underline{A}'^{t-1})
     \hat{Q}_{\textit{Semi}}^t
     \biggl)
      \biggl] = 0
\end{align*} 
with the following explanations: 
\begin{itemize}
    \item $*1)$: We use the relationship  $ \hat{V}_{\textit{Semi}}^{t-1} = \mathbb{E}_{\pi_{\alpha}}[ \hat{Q}_{\textit{Semi}}^{t}]$ and the decomposition of $\rho_{\textit{Semi}}$.
    \item 
    $*2)$: We use the fact that one can pull the expected value with respect to \newline 
    $\pi'^t_{sim}(
    A'^{t}| \underline{X}'^{t-1}, \underline{A}'^{t-1},
    A^t
    )
    \pi_\beta^t(
    A^t| 
    \underline{X}^{t-1},
    \underline{A}^{t-1}
    )$ inside.
    \item $*3)$: We use the fact that $\hat{Q}^t_{Semi}$ is independent of $\pi_{id}^t(
    A^{t} | 
    %\underline{X}'^{t-1},
    \underline{A}'^{t},
    \underline{X}^{t-1},
    \underline{A}^{t-1})
    $ as long as it fulfills the positivity assumption. 
\end{itemize}

\vspace{10pt}
\noindent 
\textit{Scenario 2:} 
If $\hat{Q}_{\textit{Semi}}$ is correctly specified, we find
\begin{align*}                  
\hat{J}_{\textit{DRL-Semi}} & =              
\mathbb{E}[ 
V_{\textit{Semi}}^{0}]           
+                   
\mathbb{E}
 \left[ 
\hat{\rho}_{\textit{Semi}}^{T}
\left(C'-  
 Q_{\textit{Semi}}^{T}
 \right)\right]                   
 +                               
 \mathbb{E}
 \left[
 \sum_{t=1}^{T-1} 
\hat{\rho}_{\textit{Semi}}^{t}  
 \left(-Q_{\textit{Semi}}^{t}
 +  
 V_{\textit{Semi}}^{t}
 \right)]\right]   
\\ 
&
=                           
\underbrace{                
\mathbb{E}
\left[ 
V_{\textit{Semi}}^{0}
\right]
}_{=J}           
+                   
\mathbb{E}
\left[ 
\hat{\rho}_{\textit{Semi}}^{T}
\underbrace{
\left(
f_C(A', X', Y)
-  
 Q_{\textit{Semi}}^{T}
 \right)}_{=0}
 \right]
\\ &            
 +                                
 \mathbb{E}
 \biggl[
 \sum_{t=1}^{T-1} 
\hat{\rho}_{
\textit{Semi}}^{t}  
    %\frac{
    %\pi_\alpha(A'^{t}| \underline{X}'^{t-1}, \underline{A}'^{t-1}) 
    %}
    %{
    %\pi'_{sim}(
    %A'^{t}| \underline{X}'^{t-1}, \underline{A}'^{t-1},
    %A^t
   % )
    %    \pi_\beta(
    %A^t| 
    %\underline{X}^{t-1},
    %\underline{A}^{t-1}
    %)
   % }
   % \\ & \cdot 
 \underbrace{
 \biggl(
 -
 %\sum_{A'^t} 
 Q_{\textit{Semi}}^{t}
 %\pi_{id}
 %(A^t| 
%\underline{X}'^{t-1},
% \underline{A}'^{t},
% \underline{X}^{t-1},
 %\underline{A}^{t-1}
% )
+
 %\\ &  +  
 \sum_{
 X^t 
 }
 V_{\textit{Semi }}^{t}
 %g(X'^t|
 %   X^t, A'^t
 %) 
 p(X^t|
 \underline{X}^{t-1},
 \underline{A}^{t}, \Xi
 ) 
% \pi_{id}
% (A^t| 
%\underline{X}'^{t-1},
% \underline{A}'^{t},
% \underline{X}^{t-1},
% \underline{A}^{t-1}
% )
 \biggl)}_{=0}
 \biggl]   
 \end{align*} 

\noindent 
where $\mathbb{E} \left[  V_{\textit{Semi}}^{0}
\right]$ corresponds to the DM estimator which is consistent if $\hat{Q}_{\textit{Semi}}  = Q_{\textit{Semi}}$ is correctly specified. For the last term, we used that 
%$ Q_{\textit{Semi}}^{t} = \sum_{A^t}  Q_{\textit{Semi}}^{t}
%\pi_{id} (A^t|  \underline{X}'^{t-1}, \underline{A}'^{t}, \underline{X}^{t-1}, \underline{A}^{t-1} )$ and that 
we can pull in the expected value with respect to the conditional distributions of $X^t$. The resulting term equals the first part of the semi-offline RL version of Bellman's equation. 
This concludes the proof of Theorem \ref{theorem_double_robustness}. 
% \hfill % \qedsymbol 
\end{proof}

\section{Estimation of Other Target Parameters from the Semi-offline RL View}
\label{app_other_variants}

In this appendix, we extend the target parameter to include time-dependent costs $C^t$ (s.t. $C = \overline{C}^1$). The newly defined target parameter becomes $J = \mathbb{E}\left[\sum_{t=1}^T C^t_{(\pi_\alpha)}\right]$.  In particular, these costs may include acquisition costs or misclassifications costs for predictions at each time-step. The acquisition costs are given by the known deterministic $f_{C_a}^t(A^t)$. When considering misclassification costs, we assume a per-step label $Y^t$ to be available at each time step (s.t. $Y = \overline{Y}^1$). The per-step misclassification costs can be computed by: 
\begin{align*}
f^t_{C_\textit{mc}}(Y^{*t}, Y^t) = 
f^t_{C_\textit{mc}}(f_\textit{cl}(\underline{A}^t,\underline{X}^t), Y^t).
\end{align*}
We combine both costs such that the target parameter is redefined to be: 
\begin{align*}
    J & = \mathbb{E}\left[
    \sum_{t=1}^T 
    \left( C^t_{a,(\pi_\alpha)} 
    + 
    C^t_{mc,(\pi_\alpha)}\right) \right]
    = 
    \mathbb{E}\left[\sum_{t=1}^T 
    \left(f_{C_a}^t(A^t_{(\pi_\alpha)})
    +
    f^t_{C_\textit{mc}}(\underline{A}^{t}_{(\pi_\alpha)}, \underline{X}^{t}_{(\pi_\alpha)}, Y^{t})
    \right) \right]
    \\ & 
    \equiv 
    \mathbb{E}
    \left[
    \sum_{t=1}^T f_C^t(\underline{A}^{t}_{(\pi_\alpha)}, \underline{X}^{t}_{(\pi_\alpha)}, Y^{t})
    \right]
    = 
    \mathbb{E}\left[\sum_{t=1}^T C^t_{(\pi_\alpha)}\right].
\end{align*}

\noindent 
The reformulation, identification and estimation steps from the semi-offline RL view can be extended to per-step costs. We provide corollaries of the identification and estimation theorems from the main body for this setting. We do not provide additional proofs, as the extensions are straightforward. 

\subsection{Identification}

We start with a corollary that extends Theorem \ref{theorem_identification_semi_offline_RL} for the per-step costs setting.

\begin{corollary}
\label{corollary_perstep_identification_semi_offline_RL} 
%\begin{subequations}
(Identification of $J$ (for per-step costs) for the semi-offline RL view).
The reformulated AFAPE problem of estimating $J$ (for per-step costs) under the semi-offline RL view  (Eq. \ref{eq:AFAPE_objective_semi_offline_RL}) is under Assumption \ref{assump:measurement_noise} (no measurement noise),  Assumption \ref{assump:consistency} (consistency),  Assumption \ref{assump:interference} (no interference), Assumption \ref{assump:nde} (NDE), Assumption \ref{assump:nuc} (NUC) and Assumption \ref{assump:positivity_global_semi_offline_RL} (global positivity)
identified by
\begin{align}
\label{eq_identificiation_semi_offline_RL_1_perstep}
    J  
     = 
    \mathbb{E}_{p'}\left[\sum_{t=1}^T C'_{(\pi_\alpha)}\right]
     = 
    \sum_{A',A,G_A(X_{(1)}),Y} 
    \sum_{t=1}^T f_C^t(            
%G_{A'}(X_{(1)}),    
\underline{A}'^{t},
\underline{X}'^{t},         
Y^{t}) 
q'(            
A',         
A, 
X, 
%G_A(X_{(1)}), 
Y) 
\end{align}
\noindent
where $q'$ is given by Eq. \ref{eq_identificiation_semi_offline_RL_2}.
\end{corollary}

\noindent 
%Note that the decomposition in Eq. \ref{eq_identificiation_semi_offline_RL_1_perstep} does not lead to the most efficient estimators for this setting. There are multiple ways to change this identification theorem to improve the derived estimators. Firstly, one may leverage the correlation between $Y^t$/$C'^t$ between different time-steps. Secondly, the choice / support restrictions of $\pi_{id}$ can potentially be improved. The regional positivity assumption requires certain support restrictions for $\pi_{id}$. These support restrictions are, however, weaker for the target $\mathbb{E}[C'^{\tau}_{(\pi_\alpha)}]$ than for $\mathbb{E}[C'^{t}_{(\pi_\alpha)}]$ (for $\tau < t$). One may thus vary the choice for $\pi_{id}$ depending on the target. These potential improvements are, however, outside the scope of this paper. We provide here only identification and estimation results that can be easily derived from the theorems in the main body. 

\noindent 
Next, we continue with a corollary that extends Theorem \ref{theorem_Bellman_equation} for the per-step costs setting.
\begin{corollary}
\label{corollary_perstep_Bellman_equation} 
(Bellman equation for semi-offline RL (for per-step costs)). 
The semi-offline RL view admits under 
Assumption \ref{assump:measurement_noise} (no measurement noise),  Assumption \ref{assump:consistency} (consistency),  Assumption \ref{assump:interference} (no interference), Assumption \ref{assump:nde} (NDE), Assumption \ref{assump:nuc} (NUC) and the local positivity assumption at datapoint $\underline{x}^{t-1},\underline{a}^{t-1},a'^t$ (from Definition \ref{def_local_positivity}), the following semi-offline RL version of the Bellman equation  for per-step costs:     
\begin{align}
\label{eq_bellman_1_perstep}
Q_{\textit{Semi}}( &
\underline{A}'^{t},         
\underline{X}^{t-1},
%G_{\underline{A}^{t-1}}(X_{(1)}), 
\underline{A}^{t-1}, 
\Xi) 
 = 
 %\\ & =  
\sum_{\mathclap{X^{t}}}                
V_{\textit{Semi}}
(    
\underline{A}'^{t}, 
\underline{X}^{t},
% G_{\underline{A}^{t}}(X_{(1)}),
\underline{A}^{t-1}, 
A^{t}=a^t, 
\Xi) 
%\nonumber
%\\ & \cdot 
% g(X'^t|X^t, A'^t)
p(
X^{t}
% G_{A^{t}}(X_{(1)})
|                   
% G_{\underline{A}^{t-1}}(X_{(1)}),   
\underline{X}^{t-1},
\underline{A}^{t-1}, 
A^t = a^t, 
\Xi)
% \nonumber 
\\
& \text{ for any }
a^t 
\in \mathcal{A}^t_\textit{adm}
(          
\underline{X}^{t-1}, 
\underline{A}^{t-1},
A'^t )  
\nonumber 
%\\
\end{align}
\begin{align}
%\nonumber
 V_{\textit{Semi}}
(
\underline{A}'^{t},
\underline{X}^{t},
%G_{\underline{A}^{t}}(X_{(1)}),  
\underline{A}^{t}, 
\Xi)  
&=
% = \\
%& =   
\sum_{Y^t} f_C^t(\underline{X}'^t, \underline{A}'^t, Y^t) 
p(Y^t|\underline{X}^t, \underline{A}^t, \Xi) 
%\nonumber
%\\ & 
+
%\nonumber 
%\\ & +
\sum_{A'^{t+1}}  
Q_{\textit{Semi}}(
\underline{A}'^{t+1},    
\underline{X}^{t},
%G_{\underline{A}^{t}}(X_{(1)}), 
\underline{A}^{t}, 
\Xi
)
\pi_{\alpha}^{t+1}(
A'^{t+1}|
\underline{X}'^{t},
% G_{\underline{A}'^{t}}(X_{(1)}),
\underline{A}'^{t})
\label{eq_bellman_2_perstep}
\end{align}
with semi-offline RL versions of the state-action value function $Q_{\textit{Semi}}$ and state value function $V_{\textit{Semi}}$:
\begin{align*}
 Q_{\textit{Semi}}^t 
 & \equiv 
 Q_{\textit{Semi}}(
 %\underline{X}'^{t-1},
\underline{A}'^{t},
\underline{X}^{t-1},
%G_{\underline{A}^{t}}(X_{(1)}),
\underline{A}^{t-1}, 
\Xi) 
\equiv  
\mathbb{E}_{p'}\left[
\sum_{\tau = t}^T
C'^\tau_{(\overline{\pi}^{t+1}_\alpha)}
\Big|
%\underline{X}'^{t-1},
\underline{A}'^{t},
\underline{X}^{t-1},
%G_{\underline{A}^{t}}(X_{(1)}),
\underline{A}^{t-1}, 
\Xi\right]  
\\
V_{\textit{Semi}}^t
& \equiv 
V_{\textit{Semi}}(
%\underline{X}'^t,
\underline{A}'^{t},
\underline{X}^{t},
\underline{A}^{t}, 
\Xi) 
\equiv 
\mathbb{E}_{p'}
\left[
\sum_{\tau = t}^T
C'^\tau_{(\overline{\pi}^{t+1}_\alpha)}
\Big|
% \underline{X}'^t,
\underline{A}'^{t},
\underline{X}^t,
% G_{\underline{A}^{t}}(X_{(1)}),
\underline{A}^{t},
\Xi
\right].
\end{align*}

\noindent 
Furthermore, $Q_{\textit{Semi}}^t$ and $V_{\textit{Semi}}^t$ are identified if the regional positivity assumption holds at $%\underline{X}'^{t-1},
\underline{A}'^{t},
\underline{X}^{t-1},
\underline{A}^{t-1}$ and 
$a^t \in \mathcal{\tilde{A}}_\textit{adm}^t(
%\underline{X}'^{t-1},
\underline{A}'^{t},
\underline{X}^{t-1},
\underline{A}^{t-1})$.
\end{corollary}

\subsection{Estimation}

The estimation formulas can be extended to the per-step setting as follows: 

\vspace{10pt}
\noindent 
\textit{1) Inverse probability weighting (IPW):} 

\nopagebreak

\noindent 
The target cost (for per-step costs) that is estimated by the semi-offline IPW estimator is
\begin{align*} 
%\label{eq:semi-off-ipw_pi_id}
    \hat{J}_{\textit{IPW-Semi}}
    & = 
    \hat{\mathbb{E}}_n\left[
    \hat{\mathbb{E}}_{n'}
    \left[ 
    \sum_{t=1}^T
    \hat{\rho}_{
    \textit{Semi}}^t
    \text{ } 
    C'^t
    \big| 
    A, X,Y
    \right]
    \right],
\end{align*} 
with the same options for $\rho_\textit{Semi}^t$ as in the setting described in the main body.

\vspace{10pt}
\noindent
\textit{2) Direct method (DM):} 

\noindent
The target cost (for per-step costs) that is estimated by the semi-offline DM estimator is 
\begin{align*} 
% \label{eq:semi-off-dm}
    \hat{J}_{\textit{DM-Semi}} = 
    \hat{\mathbb{E}}_{n'}[\hat{V}_{\textit{Semi}}^0]
\end{align*}

\noindent 
with the adapted per-step cost version of $V_{\textit{Semi}}$ from Corollary \ref{corollary_perstep_Bellman_equation}.

\vspace{10pt}
\noindent 
\textit{3) Double reinforcement learning (DRL):} 

\noindent
The target cost (for per-step costs) that is estimated by the semi-offline DRL estimator is 
\begin{align*}
    \hat{J}_{\textit{DRL-Semi}} = %\mathbb{E}_{\mathcal{D}'}
     \hat{\mathbb{E}}_{n}
     \left[
     \hat{\mathbb{E}}_{n'}
    \left[
     \sum_{t=1}^{T} 
      \left(
    \hat{\rho}_\textit{Semi}^t C'^t  
- \hat{\rho}_{\textit{Semi}}^{t}  
 \hat{Q}_\textit{Semi}^t
 %(\underline{X}'^{t-1},
 %\underline{A}'^{t}, 
 %\underline{X}^{t-1}, 
 %\underline{A}^{t-1}) 
 +
 \hat{\rho}_{\textit{Semi}}^{t-1}  
 \hat{V}_\textit{Semi}^{t-1}
 %(\underline{X}'^{t-1},
 %\underline{A}'^{t-1}, 
 %\underline{X}^{t-1}, 
 %\underline{A}^{t-1})
 \right) 
 \Big| 
 A,X,Y
 \right]
 \right].
\end{align*}

\noindent 
with the adapted per-step cost version of $V_{\textit{Semi}}$ and $Q_{\textit{Semi}}$.

\section{Derivation of the Missing Data Semiparametric Theory Approach}
\label{app:derivation_missing_data_semiparametrics}

In this appendix, we show why Eq.~\ref{eq:h_missing_data} and Eq.~\ref{eq:lambda_2_missing_data} constitute valid choices for an element of the IPW space $\Lambda_\textit{IPW}$ and the augmentation space $\Lambda_2$, but rely on the positivity assumption for missing data (Assumption \ref{assump:positivity_missing_data}). This is a simplified derivation from \cite{tsiatis_semiparametric_2006}.
\newline

\noindent 
\textbf{IPW space:}

Clearly, $h_\textit{Miss} = \rho_\textit{Miss} \mathbb{E}[C_{(\pi_\alpha)}|X_{(1)},Y] - J$ is a function of the observed data, since $\rho_\textit{Miss}$ is non-zero only for complete cases. It thus remains to be shown that 
\begin{flalign*}
    \mathbb{E}
    \left[
    h_\textit{Semi}(A, G_A(X_{(1)}),Y) 
    | 
    X_{(1)},Y
    \right] 
    = 
    \varphi^F(X_{(1)},Y).
\end{flalign*}
This is shown in the following: 
\begin{flalign*}
    \mathbb{E}  
    [
    h_\textit{Semi}  (A, G_A(X_{(1)}),Y)  
    | 
    X_{(1)},Y
    ] 
    &= 
    \mathbb{E}
    \left[
    \rho_\textit{Miss} \mathbb{E}[C_{(\pi_\alpha)}|X_{(1)},Y]
    \Big|X_{(1)},Y
    \right]- J
\\ & = 
     \mathbb{E}
    \left[
    \prod_{t=1}^T
    \frac{\mathbb{I}(A^t= \vec{1}) }{
    \pi_{\beta}^t(
    A^t=\vec{1}| 
    \underline{X}_{(1)}^{t-1},
    \underline{A}^{t-1}= \vec{1})} 
    \mathbb{E}[C_{(\pi_\alpha)}|X_{(1)},Y]
    \Big|X_{(1)},Y
    \right]- J
\\ & = 
     \mathbb{E}
    \left[
    \prod_{t=1}^T
    \frac{
    \mathbb{E}
    [ \mathbb{I}(A^t= \vec{1})
    |\underline{X}_{(1)}^{t-1},
    \underline{A}^{t-1}= \vec{1})]
    }{
    %\prod_{t=1}^T 
    \pi_{\beta}^t(
    A^t=\vec{1}| 
    \underline{X}_{(1)}^{t-1},
    \underline{A}^{t-1}= \vec{1})} 
    \mathbb{E}[C_{(\pi_\alpha)}|X_{(1)},Y]
    \Big|X_{(1)},Y
    \right]- J
\\ & = 
    \mathbb{E}
    \left[
    1 \cdot 
    \mathbb{E}
    \left[
    C_{(\pi_\alpha)}
    \Big|
    X_{(1)},Y
    \right]     
    \Big|
    X_{(1)},Y
    \right] 
    - J
\\ & = 
    \mathbb{E}
    \left[
    C_{(\pi_\alpha)}
    \Big|
    X_{(1)},Y
    \right] 
    - J
    =\varphi^F(X_{(1)},Y).
\end{flalign*}

\noindent 
\textbf{Augmentation space $\Lambda_2$: }

To derive $\Lambda_2$ under the missing data view, one redefines any function $b(A, G_A(X_{(1)}), Y)$ using the fact that $A$ is a categorical variable:  
\begin{flalign}
\label{eq:factorization_b_missing_data}
b(A, G_A(X_{(1)}), Y) 
= 
\sum_{a \in \mathcal{A}} \mathbb{I}(A = a) b_a(G_a(X_{(1)}), Y)
\end{flalign}
where $b_a(G_a(X_{(1)}), Y)$ is any mean zero, finite variance function of $G_a(X_{(1)}), Y$. 
This allows the enforcement of the zero conditional mean condition that defines $\Lambda_2$:
\begin{flalign*}
\mathbb{E}[b(A, G_A(X_{(1)}), Y)|X_{(1)}, Y] = 
\sum_{a \in \mathcal{A}}
p(A=a|G_a(X_{(1)}),Y) b_a(G_a(X_{(1)}), Y) = 0.
\end{flalign*}
Under the missing data positivity assumption, one can now solve for $b_{\vec{1}}$:
\begin{flalign*}
b_{\vec{1}}(G_{\vec{1}}(X_{(1)}),Y) & = 
- \frac{1}
{p(A=\vec{1}|G_{\vec{1}}(X_{(1)}),Y)}
\sum_{a \in \mathcal{A} \backslash \vec{1}}
p(A=a|G_a(X_{(1)}),Y) b_a(G_a(X_{(1)}), Y)
\end{flalign*}

\noindent 
Substituting $b_{\vec{1}}(G_{\vec{1}}(X_{(1)}),Y)$ 
into Eq. \ref{eq:factorization_b_missing_data} and applying the known factorization of the propensity score model for our AFA setting gives the desired space $\Lambda_2$ consisting of all 
\begin{flalign*}
    b(A,  G_A&(X_{(1)}), Y) 
     =    
    \sum_{a \in \mathcal{A} \backslash \vec{1}}
    \left[
    \mathbb{I}(A=a)
    - 
    \frac{
    \mathbb{I}(A=\vec{1})
    p(A=a
    |
    G_{a}(X_{(1)}),Y) 
    }
    {
    p(A=\vec{1}
    |
    G_{\vec{1}}(X_{(1)}),Y
    )}
    \right] 
    b_a( 
    G_a(X_{(1)}), 
    Y) 
    \\
    & =    
    \sum_{a \in \mathcal{A} \backslash \vec{1}}
    \left[
    \mathbb{I}(A=a)
    - 
    \prod_{t=1}^T
    \frac{
    \mathbb{I}(A^t=\vec{1})
    \pi_\beta^t(A^t=a^t
    |
    G_{\underline{a}^{t-1}}(X_{(1)}),\underline{a}^{t-1} ) 
    }
    {
    \pi_\beta^t(A^t=\vec{1}
    |
    G_{\underline{a}^{t-1}}(X_{(1)}),\underline{a}^{t-1}
    )}
    \right] 
    b_a( 
    G_a(X_{(1)}), 
    Y). 
\end{flalign*}

\section{Proof of Lemma \ref{lemma_h_offline_RL} and Lemma \ref{lemma_liu_lambda_2}}
\label{app:derivation_offline_RL_semiparametrics}

In this appendix, we prove Lemma \ref{lemma_h_offline_RL} and Lemma \ref{lemma_liu_lambda_2} which establish the equivalence of the missing data and offline RL semiparametric theory approaches to AFAPE under the NUC and NDE assumptions. 
We start with Lemma \ref{lemma_h_offline_RL}. 

For ease of reference we repeat it here:
\vspace{5pt}
\begin{lma}{
\ref{lemma_h_offline_RL}}
    (Relating the offline RL IPW estimator to  the IPW space).
The functional $h_{\textit{Off}} \equiv \rho_{\textit{Off}} C-J$, based on the IPW estimator from the offline RL view, is a valid element of the IPW space: $h_{\textit{Off}} \in \Lambda_\textit{IPW}$.
\end{lma}
\vspace{5pt}

\vspace{5pt}
\begin{proof}
    We need to show that $h_\textit{Off} = h_\textit{Off}(A, G_A(X_{(1)}),Y) = \rho_{\textit{Off}}^T C - J \in \Lambda_\textit{IPW}$.
    Clearly, $h_\textit{Off}$ is a function of the observed data. 
    It thus remains to be shown that 
    \begin{flalign*}
    \mathbb{E}
    \left[
    h_\textit{Off}(A, G_A(X_{(1)}),Y) 
    | 
    X_{(1)},Y
    \right] 
    = 
    \varphi^F(X_{(1)},Y).
    \end{flalign*}
This is shown in the following:
\begin{flalign*}
    \mathbb{E} & \left[
    h_\textit{Off}(A,G_A(X_{(1)}), Y)
    | X_{(1)}, Y  \right] 
    =
    \mathbb{E}\left[ 
    \rho_{\textit{Off}}^T 
    C - J
    \Big| 
    X_{(1)}, Y  \right]  
    \\
    & 
    = 
    \mathbb{E}\left[ 
    \prod_{t=1}^T 
    \frac{\pi_\alpha^{t}(A^{t}|
    G_{\underline{A}^{{t-1}}}(X_{(1)}),
    \underline{A}^{{t-1}})}
    {
    \pi_\beta^{t}(A^{t}|
    G_{\underline{A}^ {{t-1}}}(X_{(1)}),
    \underline{A}^ {{t-1}})} 
    f_C(Y, G_{A}(X_{(1)}), A)
    \Big| 
    X_{(1)}, Y  \right] - J
    \\ 
    &=  
    \sum_{a \in \mathcal{A}}
    \prod_{t=1}^T
    \frac{
    \pi_\alpha^{t}(a^{t}|G_{\underline{a}^ {{t-1}}}(X_{(1)}),\underline{a}^ {{t-1}})}
    {
    \pi_\beta^{t}(a^{t}|G_{\underline{a}^{{t-1}}}(X_{(1)}), \underline{a}^{{t-1}})
    }
    f_C(Y, G_{a}(X_{(1)}), a)
    \pi_\beta^{t}(a^{t}|G_{\underline{a}^{{t-1}}}(X_{(1)}), \underline{a}^{{t-1}})
    - J
    \\ 
    &=  
    \sum_{a \in \mathcal{A}} 
    \prod_{t = 1}^T 
    \pi_\alpha^{t}(a^{t}|G_{\underline{a}^ {{t-1}}}(X_{(1)}),\underline{a}^ {{t-1}})
    f_C(Y, G_{a}(X_{(1)}), a) 
      - J 
    = \varphi^F(X_{(1)},Y)
\end{flalign*}
which completes the proof.
%\hfill %\qedsymbol 
\end{proof}
\vspace{5pt}

\noindent 
Next, we prove Lemma \ref{lemma_liu_lambda_2}. We also repeat it here: 

\vspace{5pt}
\begin{lma}{
\ref{lemma_liu_lambda_2}}
($\Lambda_*$ is equal to the augmentation space).
 The augmentation space $\Lambda_{2}$ is 
    %under the offline RL positivity assumption and 
    %for the setting of one acquisition option per time-step ($A^t \in \{0,1\}$) 
    equal to $\Lambda_*$. 
\end{lma}
\vspace{5pt}

\noindent 
We reuse properties about $\Lambda_*$ shown by Liu et al. \cite{liu_efficient_2021} and do not repeat the corresponding proofs for these properties as they involve cumbersome notation. 

\vspace{5pt}
\begin{proof}
To demonstrate that $\Lambda_* = \Lambda_2$, we first introduce a new space, denoted as $\Lambda_\textit{*}^{AF}$,  containing all functions of $b(A, X_{(1)},Y)$ for which $\mathbb{E}[b(A, X_{(1)},Y)|X_{(1)},Y] = 0$ holds:
\begin{flalign*}
\label{eq:efficient_if_liu}
 \Lambda_\textit{*}^{AF} 
 \equiv & 
 \Bigg\{
 b^{t}( 
 \underline{A}^{t-1}, 
 X_{(1)},  
 Y) 
 \left(
 \frac{A^t}{\pi^t_\beta}-1
 \right) 
 :
 b^{t}
 (  \underline{A}^{t-1}, X_{(1)},  Y) 
 \in 
 \mathcal{H}; t \in \{ 1, ..., T\} 
 \Bigg\}.
\end{flalign*}
\noindent 
The space $\Lambda_\textit{*}^{AF}$ indeed includes all functions $b(A, X_{(1)}, Y)$ with mean zero and finite variance for which $\mathbb{E}[b(A, X_{(1)},Y)|X_{(1)},Y] = 0$ holds. Specifically, the elements $b^{t}(\underline{A}^{t-1}, X_{(1)}, Y)$ can represent any function of $\underline{A}^{t-1}, X_{(1)}, Y$. 
Since $A^t$ is binary (taking values 0 or 1), adding a term $\left(\frac{A^t}{c(\underline{A}^{t-1}, X_{(1)},Y)} - 1 \right)$ (for some $c(A, X_{(1)},Y)$) generalizes the space to contain all functions of $A^t, X_{(1)}, Y$. The specific choice $c(\underline{A}^{t-1}, X_{(1)},Y) = \pi^t_\beta$ is enforced by the condition $\mathbb{E}[b(A, X_{(1)},Y)|X_{(1)},Y] = 0$:
\begin{flalign*}
    \mathbb{E}\biggl[b^t(\underline{A}^{t-1}, X_{(1)},Y) \left(\frac{A^t}{\pi^t_\beta} - 1 \right) |X_{(1)},Y\biggl] 
    % =\\ 
    & 
    = 
    \mathbb{E}\left[b^t(\underline{A}^{t-1}, X_{(1)},Y)\mathbb{E}\left[\left(\frac{A^t}{\pi^t_\beta} - 1 \right) |X_{(1)},Y, \underline{A}^{t-1}
    \right]|X_{(1)},Y\right]
    \\ & = 
    \mathbb{E}\left[b^t(\underline{A}^{t-1}, X_{(1)},Y) \underbrace{\left(\frac{\pi^t_\beta}{\pi^t_\beta} - 1 \right)}_{0} |X_{(1)},Y\right]
    = 0.
\end{flalign*} 
%Examining the final element (where $t = T$) confirms that $\Lambda_*^{AF}$ contains all observed data functions with mean zero and finite variance.

\noindent 
Now, in order to find $\Lambda_2$, we must find the subspace of $\Lambda_\textit{*}^{AF}$, that contains only the functions of the observed data. 
This is exactly $\Lambda_*$ as was shown by Liu et al. \cite{liu_efficient_2021} (in the proof of Remark 13). In fact, they showed that there aren't observed data elements that are in  $\Lambda_\textit{*}^{AF}$, but not in  $\Lambda_*$. They also showed that $\Lambda_*$ only contains observed data elements.
This concludes the proof. 
% \hfill % \qedsymbol 
\end{proof}

\section{Proof of Lemma \ref{lemma_h_semioffline_RL}}
\label{app:derivation_semioffline_RL_semiparametrics}

In this appendix, we prove Lemma \ref{lemma_h_semioffline_RL}. The proof follows a similar approach as for Lemma \ref{lemma_h_offline_RL} for the offline RL IPW estimator. 
We repeat the lemma here for ease of reference. 

\vspace{5pt}
\begin{lma}{
\ref{lemma_h_semioffline_RL}}
    $h_\textit{Semi} \equiv h_\textit{Semi}(A, G_A({X_{(1)}}),Y) = 
    \hat{\mathbb{E}}_{n'} [\rho_{\textit{Semi}}^T C'| A,G_A({X_{(1)}}),Y]  - J$ is an element of the IPW space $\Lambda_\textit{IPW}$. 
\end{lma}
\vspace{5pt}

\vspace{5pt}
\begin{proof}
    We need to show that $h_\textit{Semi} = h_{\textit{Semi}}(A, G_A(X_{(1)}),Y) = 
    \hat{\mathbb{E}}_{n'} [\rho_{\textit{Semi}}^T C'| A,G_A({X_{(1)}}),Y] - J 
    \in \Lambda_\textit{IPW}  
     $. 
    Clearly, $h_\textit{Semi}$ is a function of the observed data. 
    It thus remains to be shown that 
    \begin{flalign*}
    \mathbb{E}
    \left[
    h_\textit{Semi}(A, G_A(X_{(1)}),Y) 
    | 
    X_{(1)},Y
    \right] 
    = 
    \varphi^F(X_{(1)},Y).
    \end{flalign*}
This is shown in the following:
\begin{flalign*}
    \mathbb{E} & 
    \left[
    h_\textit{Semi}(A,G_A(X_{(1)}), Y)
    \Big| X_{(1)}, Y  \right] 
    =
    \mathbb{E}\left[
    \mathbb{E} 
    \left[
    \rho_{\textit{Semi}}^T 
    f_C(A', G_{A'}(X_{(1)}), Y)
    | A,G_A({X_{(1)}}),Y
    \right]
    - J
    \Big| 
    X_{(1)}, Y  \right]  
%\\
%& 
%= 
    %\mathbb{E}\left[ 
    %\mathbb{E}\left[ 
    %\rho_{\textit{Semi}}^T f_C(Y, G_{A'}(X_{(1)}), A')| Y, G_{A}(X_{(1)}), A \right]
    %- J
    %\Big| 
    %X_{(1)}, Y  \right]  
    \\
    & 
    = 
    \mathbb{E}\left[ 
    \prod_{t=1}^T 
    \frac{
    \pi_\alpha^t(A'^{t}| \underline{X}'^{t-1}, \underline{A}'^{t-1}) 
    }
    {
    \pi'^t_{\textit{sim}}(
    A'^{t}| 
    \underline{X}'^{t-1}, \underline{A}'^{t-1},
    A^t
    )
    }
    \frac{
    \pi_{id}^t(
    A^{t } | 
    %\underline{X}'^{t-1},
    \underline{A}'^{t},
    \underline{X}^{t-1},
    \underline{A}^{t-1})
    }
    {
    \pi_\beta^t(
    A^t| 
    \underline{X}^{t-1},
    \underline{A}^{t-1}
    )
    }
    f_C(A', G_{A'}(X_{(1)}), Y)
    \Big| 
    X_{(1)}, Y  \right] - J
    \\
    & 
    = 
    \sum_{a \in \mathcal{A}}
    \sum_{a' \in \mathcal{A}}
    \prod_{t=1}^T 
    \frac{
    \pi_\alpha^t(a'^{t}| \underline{X}'^{t-1}, \underline{a}'^{t-1}) 
    }
    {
    \pi'^t_{\textit{sim}}(
    a'^{t}| 
    \underline{X}'^{t-1}, \underline{a}'^{t-1},
    a^t
    )
    }
    \frac{
    \pi_{id}^t(
    a^{t } | 
   % \underline{X}'^{t-1},
    \underline{a}'^{t},
    \underline{X}^{t-1},
    \underline{a}^{t-1})
    }
    {
    \pi_\beta^t(
    a^t| 
    \underline{X}^{t-1},
    \underline{a}^{t-1}
    )
    }
    f_C(a', G_{a'}(X_{(1)}),Y)
    \\
    & \cdot 
    \pi'^t_{\textit{sim}}(
    a'^{t}| 
    \underline{X}'^{t-1}, \underline{a}'^{t-1},
    a^t
    )
    \pi_\beta^t(
    a^t| 
    \underline{X}^{t-1},
    \underline{a}^{t-1}
    ) 
     - J
         \\
    & 
    = 
    \sum_{a \in \mathcal{A}}
    \sum_{a' \in \mathcal{A}}
    \prod_{t=1}^T 
    \pi_\alpha^t(a'^{t}| \underline{X}'^{t-1}, \underline{a}'^{t-1}) 
    \pi_{id}^t(
    a^{t } | 
   % \underline{X}'^{t-1},
    \underline{a}'^{t},
    \underline{X}^{t-1},
    \underline{a}^{t-1})
    f_C(a', G_{a'}(X_{(1)}),Y)
     - J
              \\
    & 
    = 
    \sum_{a' \in \mathcal{A}}
    \prod_{t=1}^T 
    \pi_\alpha^t(a'^{t}| \underline{X}'^{t-1}, \underline{a}'^{t-1}) 
    f_C(a', G_{a'}(X_{(1)}),Y)
     - J
    \\ 
    &=  
    \varphi^F(X_{(1)},Y)
\end{flalign*}
which completes the proof.
%\hfill %\qedsymbol 
\end{proof}
\vspace{5pt}

\section{Derivation of the Influence Function for Semi-offline RL}
\label{app:semioffline_RL_projection}

In this Appendix, we provide the complete derivation of the projection of $h_\textit{Semi}$ onto $\Lambda_{2,\textit{Semi}}(\Xi)$, thereby completing the proof for the class of influence functions under the semi-offline RL view, as proposed in Theorem \ref{theorem_efficient_if}.

We now show all in-between steps of the following equalities shown in Section \ref{sec_semiparametrics}: 
\begin{flalign*}
    \varphi_{\textit{Semi}}(A, G_A{(X_{(1)}}),Y; \Xi)
   &  = 
    h_\textit{Semi}
    -  \Pi ( h_\textit{Semi} | \Lambda_{2,\textit{Semi}}(\Xi) )
% 1
\\ 
& 
= 
    h_\textit{Semi}  
    - 
    \sum_{t=1}^T
    \mathbb{E}
    \left[ 
    h_\textit{Semi}
    \Big| 
    A^t,  
    \underline{A}^{t-1}, 
    G_{\underline{A}^{t-1}}(X_{(1)}), 
    \Xi
    \right]
    + 
    \sum_{t=1}^T
    \mathbb{E}
    \left[ 
    h_\textit{Semi}
    \Big|    
    \underline{A}^{t-1}, 
    G_{\underline{A}^{t-1}}(X_{(1)}), 
    \Xi
    \right]
%2
\\ 
& 
 \overset{*_1}{=} 
    \mathbb{E}
    \left[ 
    \rho_{\textit{Semi}}^T f_C(A', G_{A'}(X_{(1)}), Y)| Y, G_{A}(X_{(1)}),A \right]
%2.2
\\ & -
    \sum_{t=1}^T
    \mathbb{E}
    \left[ 
    \rho_{\textit{Semi}}^t Q_\textit{Semi}^t 
    \Big|   
    A^t, 
    \underline{A}^{t-1}, 
    G_{\underline{A}^{t-1}}(X_{(1)}), 
    \Xi
    \right]
%2.3
\\ & +
    \sum_{t=1}^T
    \mathbb{E}
    \left[ 
    \rho_{\textit{Semi}}^{t-1} V_\textit{Semi}^{t-1}  
    \Big|   
    \underline{A}^{t-1}, 
    G_{\underline{A}^{t-1}}(X_{(1)}), 
    \Xi
    \right]
  - J
%3
\\ 
& 
=  
    \mathbb{E}
    \left[ 
     \rho_{\textit{Semi}}^T f_C(A', G_{A'}(X_{(1)}), Y)
     - 
     \sum_{t=1}^T
     \rho_{\textit{Semi}}^t Q_\textit{Semi}^t 
     +
     \sum_{t=1}^T
     \rho_{\textit{Semi}}^{t-1} V_\textit{Semi}^{t-1} 
     \Big|   
    A, 
    G_{A}(X_{(1)}), 
    Y
    \right]
  - J.
\end{flalign*}

\noindent 
We now go into more detail why $*1)$ holds. 
We begin with the term including $Q_\textit{Semi}$: 
\begin{flalign*}
    \mathbb{E}
    \bigg[ 
    h_\textit{Semi}
    \Big|   &
    A^t, 
    \underline{A}^{t-1}, 
    G_{\underline{A}^{t-1}}(X_{(1)}), 
    \Xi
    \bigg] 
    =
% 2
\\
& 
=  
    \mathbb{E} 
    \left[ 
    \mathbb{E}
    \left[ 
    \rho_{\textit{Semi}}^T f_C(A', G_{A'}(X_{(1)}), Y)| Y, G_{A}(X_{(1)}),A 
    \right]
    |
    A^t,  
    \underline{A}^{t-1}, 
    G_{\underline{A}^{t-1}}(X_{(1)}), 
    \Xi
    \right] - J 
% 3
\\
& 
= 
    \mathbb{E}
    \left[ 
    \rho_{\textit{Semi}}^T f_C(A', G_{A'}(X_{(1)}), Y)
    |
    A^t,  
    \underline{A}^{t-1}, 
    G_{\underline{A}^{t-1}}(X_{(1)}), 
    \Xi
    \right]- J 
% 4
\\ & = 
    \mathbb{E}
    \biggl[ 
    \prod_{\tau=1}^T 
    \frac{
    \pi_\alpha^\tau(A'^{\tau}| \underline{X}'^{\tau-1}, \underline{A}'^{\tau-1}) 
    }
    {
    \pi'^\tau_{\textit{sim}}(
    A'^{\tau}| 
    \underline{X}'^{\tau-1}, \underline{A}'^{\tau-1},
    A^\tau
    )
    }
    \frac{
    \pi_{id}^\tau(
    A^{\tau } | 
    %\underline{X}'^{t-1},
    \underline{A}'^{\tau},
    \underline{X}^{\tau-1},
    \underline{A}^{\tau-1})
    }
    {
    \pi_\beta^\tau(
    A^\tau| 
    \underline{X}^{\tau-1},
    \underline{A}^{\tau-1}
    )
    }
%4.2
\\ 
& 
\quad 
\cdot 
    f_C(A', G_{A'}(X_{(1)}), Y)
    \Big|
    A^t,  
    \underline{A}^{t-1}, 
    G_{\underline{A}^{t-1}}(X_{(1)}), 
    \Xi
    \biggl] - J 
% 5
\\ 
& 
= 
    \mathbb{E}
    \Big[ 
    \prod_{\tau = 1}^t
    \frac{
    \pi_\alpha^\tau(A'^{1}| \underline{X}'^{\tau-1}, \underline{A}'^{\tau-1}) 
    }
    {
    \pi'^\tau_{\textit{sim}}(
    A'^{\tau}| 
    \underline{X}'^{\tau-1}, \underline{A}'^{\tau-1},
    A^\tau
    )
    }
    \frac{
    \pi_{id}^\tau(
    A^{\tau } | 
    %\underline{X}'^{t-1},
    \underline{A}'^{\tau},
    \underline{X}^{\tau-1},
    \underline{A}^{\tau-1})
    }
    {
    \pi_\beta^\tau(
    A^\tau| 
    \underline{X}^{\tau-1},
    \underline{A}^{\tau-1}
    )
    }
%5.2
\\ 
&
\cdot
    \mathbb{E}
    \Big[ 
    \prod_{\tau=t+1}^T 
    \frac{
    \pi_\alpha^\tau(A'^{\tau}| \underline{X}'^{\tau-1}, \underline{A}'^{\tau-1}) 
    }
    {
    \pi'^\tau_{\textit{sim}}(
    A'^{\tau}| 
    \underline{X}'^{\tau-1}, \underline{A}'^{\tau-1},
    A^\tau
    )
    }
    \frac{
    \pi_{id}^\tau(
    A^{\tau } | 
    %\underline{X}'^{t-1},
    \underline{A}'^{\tau},
    \underline{X}^{\tau-1},
    \underline{A}^{\tau-1})
    }
    {
    \pi_\beta^\tau(
    A^\tau| 
    \underline{X}^{\tau-1},
    \underline{A}^{\tau-1}
    )
    }
%5.3
\\ 
&
\cdot
    f_C(A', G_{A'}(X_{(1)}), Y)
    \Big|
    A'^t,
    \underline{A}'^{t-1},
    A^t,  
    \underline{A}^{t-1}, 
    G_{\underline{A}^{t-1}}(X_{(1)}), 
    \Xi
     \Big]
%\\ 
%&
%\cdot
    \Big|
    A^t,  
    \underline{A}^{t-1}, 
    G_{\underline{A}^{t-1}}(X_{(1)}), 
    \Xi
    \Big] - J 
% 6
\\ 
& 
= 
    \mathbb{E}
    \Big[ 
    \rho^t_\textit{Semi}
    Q_\textit{Semi}(\underline{A}'^{t}, \underline{A}^{t-1}, G_{\underline{A}^{t-1}}(X_{(1)}), \Xi)
    \Big|
    A^t,  
    \underline{A}^{t-1}, 
    G_{\underline{A}^{t-1}}(X_{(1)}), 
    \Xi
    \Big]- J 
\\ 
& 
= 
    \mathbb{E}
    \Big[ 
    \rho^t_\textit{Semi}
    Q_\textit{Semi}^t
    \Big|
    A, 
    G_{A}(X_{(1)}), 
    Y
    \Big] - J 
\end{flalign*}

\noindent 
Similarly, we can show for the term including $V_\textit{Semi}$:
\begin{flalign*}
    \mathbb{E}
    \bigg[ 
    h_\textit{Semi}
    \Big|    &
    \underline{A}^{t-1}, 
    G_{\underline{A}^{t-1}}(X_{(1)}), 
    \Xi
    \bigg] = 
% 2
\\
& 
=  
    \mathbb{E} 
    \left[ 
    \mathbb{E}
    \left[ 
    \rho_{\textit{Semi}}^T f_C(A', G_{A'}(X_{(1)}), Y)| Y, G_{A}(X_{(1)}),A 
    \right]
    |
    \underline{A}^{t-1}, 
    G_{\underline{A}^{t-1}}(X_{(1)}), 
    \Xi
    \right]- J 
% 3
\\
& 
= 
    \mathbb{E}
    \left[ 
    \rho_{\textit{Semi}}^T f_C( A', G_{A'}(X_{(1)}), Y)
    |
    \underline{A}^{t-1}, 
    G_{\underline{A}^{t-1}}(X_{(1)}), 
    \Xi
    \right]- J 
% 4
\\ & = 
    \mathbb{E}
    \biggl[ 
    \prod_{\tau=1}^T 
    \frac{
    \pi_\alpha^\tau(A'^{\tau}| \underline{X}'^{\tau-1}, \underline{A}'^{\tau-1}) 
    }
    {
    \pi'^\tau_{\textit{sim}}(
    A'^{\tau}| 
    \underline{X}'^{\tau-1}, \underline{A}'^{\tau-1},
    A^\tau
    )
    }
    \frac{
    \pi_{id}^\tau(
    A^{\tau } | 
    %\underline{X}'^{t-1},
    \underline{A}'^{\tau},
    \underline{X}^{\tau-1},
    \underline{A}^{\tau-1})
    }
    {
    \pi_\beta^\tau(
    A^\tau| 
    \underline{X}^{\tau-1},
    \underline{A}^{\tau-1}
    )
    }
%4.2
\\ 
& 
\quad 
\cdot 
    f_C(A', G_{A'}(X_{(1)}), Y)
    \Big|
    \underline{A}^{t-1}, 
    G_{\underline{A}^{t-1}}(X_{(1)}), 
    \Xi
    \biggl]- J 
% 5
\\ 
& 
= 
    \mathbb{E}
    \Big[ 
    \prod_{\tau = 1}^{t-1}
    \frac{
    \pi_\alpha^\tau(A'^{1}| \underline{X}'^{\tau-1}, \underline{A}'^{\tau-1}) 
    }
    {
    \pi'^\tau_{\textit{sim}}(
    A'^{\tau}| 
    \underline{X}'^{\tau-1}, \underline{A}'^{\tau-1},
    A^\tau
    )
    }
    \frac{
    \pi_{id}^\tau(
    A^{\tau } | 
    %\underline{X}'^{t-1},
    \underline{A}'^{\tau},
    \underline{X}^{\tau-1},
    \underline{A}^{\tau-1})
    }
    {
    \pi_\beta^\tau(
    A^\tau| 
    \underline{X}^{\tau-1},
    \underline{A}^{\tau-1}
    )
    }
%5.2
\\ 
&
\cdot
    \mathbb{E}
    \Big[ 
    \prod_{\tau=t}^T 
    \frac{
    \pi_\alpha^\tau(A'^{\tau}| \underline{X}'^{\tau-1}, \underline{A}'^{\tau-1}) 
    }
    {
    \pi'^\tau_{\textit{sim}}(
    A'^{\tau}| 
    \underline{X}'^{\tau-1}, \underline{A}'^{\tau-1},
    A^\tau
    )
    }
    \frac{
    \pi_{id}^\tau(
    A^{\tau } | 
    %\underline{X}'^{t-1},
    \underline{A}'^{\tau},
    \underline{X}^{\tau-1},
    \underline{A}^{\tau-1})
    }
    {
    \pi_\beta^\tau(
    A^\tau| 
    \underline{X}^{\tau-1},
    \underline{A}^{\tau-1}
    )
    }
%5.3
\\ 
&
\cdot
    f_C(A', G_{A'}(X_{(1)}), Y)
    \Big|
    \underline{A}'^{t-1},
    \underline{A}^{t-1}, 
    G_{\underline{A}^{t-1}}(X_{(1)}), 
    \Xi
     \Big]
%\\ 
%&
%\cdot
    \Big|
    \underline{A}^{t-1}, 
    G_{\underline{A}^{t-1}}(X_{(1)}), 
    \Xi
    \Big] - J 
% 6
\\ 
& 
= 
    \mathbb{E}
    \Big[ 
    \rho^{t-1}_\textit{Semi}
    V_\textit{Semi}(\underline{A}'^{t-1}, \underline{A}^{t-1}, G_{\underline{A}^{t-1}}(X_{(1)}), \Xi)
    \Big| 
    \underline{A}^{t-1}, 
    G_{\underline{A}^{t-1}}(X_{(1)}), 
    \Xi
    \Big]- J 
\\ 
& 
= 
    \mathbb{E}
    \Big[ 
    \rho^{t-1}_\textit{Semi}
    V_\textit{Semi}^{t-1}
    \Big|
    A, 
    G_{A}(X_{(1)}), 
    Y
    \Big]- J .
\end{flalign*}

\section{Experiment Details}
\label{Appendix_experiments}

In this section, we describe the experiment setup in more detail. 
We also provide a detailed list of the parameters and configurations for each experiment in Tables \ref{tab:details synthetic} and \ref{tab:details synthetic missingness}.

\subsection{Data, Costs and Missingness Mechanisms}
For the experiments, we defined a "superfeature" as a feature that comprises multiple subfeatures, which are acquired jointly and which have a single cost. 
Furthermore, we assumed a subset of features is available at no cost (free features) and set fixed acquisition costs $c_{acq}$ 
for the remaining features. A prediction was to be performed at each time-step, which corresponds to the setting described in Appendix \ref{app_other_variants}. We chose misclassification costs such that good policies must find a balance between feature acquisition cost and predictive value of the features.

We evaluated and compared the described methods on synthetic datasets with and without violation of either the NDE or NUC assumption. 
In experiments where the NDE assumption holds, the features are distributed according to: 
\begin{align*}
    X_{(1),i}^t = 
    \begin{cases}
    \gamma_{i} X_{(1),i}^{t-1} + (1- 
    \gamma_{i}) \epsilon_i
    , & \text{if } t>0 \\
     \epsilon_i , & \text{if } t=0 .
\end{cases}
\end{align*}
where 
$\epsilon_i \sim \mathcal{N}(0, \sigma)$.
In experiments with a violation of the NDE assumption, the unobserved variables $U$ were distributed according to: 
\begin{align*}
    U_{i}^t = 
    \begin{cases}
    \gamma_{i} U_i^{t-1} + (1- 
    \gamma_{i}) \epsilon_i  + 0.5 \sum_i A_i^{t-1} 
     & \text{if } t>1 \\
     \gamma_{i} U_i^{t-1} + (1- 
    \gamma_{i}) \epsilon_i  , & \text{if } t=1 \\
     \epsilon_i , & \text{if } t=0 .
\end{cases}
\end{align*}

\noindent
The labels are distributed according to 
\begin{align*}
    p(Y^t=1) = 
    \begin{cases}
    1
    , & \text{if } \zeta_{1} \sum_i W_i X_{(1),i}^t +  \zeta_{2} \sum_i  W_i X_{(1),i}^{t-1} > 0  \\
   0.3 , & \text{otherwise}.
\end{cases}
\end{align*}
This choice for $Y$ simulates a scenario where not all data points are equally easy to classify. 

%We induced synthetic missingness with MAR and MNAR missingness mechanisms similar to the scenarios described in Appendices \ref{app_missing_data}. 
The retrospective policy $\pi_\beta$ follows different logistic models depending on whether a MAR assumption (NUC holds) or MNAR assumption (NUC is violated) is assumed, as specified in Table \ref{tab:details synthetic}. 
To evaluate the convergence of different estimators when the NDE assumption holds, we consider the average cost of running the AFA agent on the dataset over all data points in the ground truth test set (without missingness) as the true expected cost $J$. 
%This corresponds to estimating $J$ using Eq. \ref{eq:AFAPE_objective_miss} with a Monte Carlo estimate for $\hat{\mathbb{E}}\left[ C_{(\pi_{\alpha})}|X_{(1)},Y \right]$ and the ground truth data without missingness (i.e. samples from $p(X_{(1)},Y)$). 
When NDE is violated, we sample the ground truth data generating process while running the agent and do so the same number of times as there are data points in the test set. 

We performed 5 different experiments: 
\begin{itemize}
    \item \textbf{Experiment 1:} 
    Standard experiment where NUC, NDE and all three positivity assumptions (Assumptions \ref{assump:positivity_offline_RL}, 
    \ref{assump:positivity_missing_data}, and \ref{assump:max_positivity_global_semi_offline_RL}) hold. 
    \item \textbf{Experiment 2:} 
    The NUC and NDE assumptions hold, but the missing data positivity assumption (Assumption \ref{assump:positivity_missing_data}) is violated. This is achieved by reducing the number of complete cases to $0.007\%$. 
    \item \textbf{Experiment 3:} 
    The NUC and NDE assumptions hold, but the offline RL (Assumption \ref{assump:positivity_offline_RL}) positivity assumption is violated. This is achieved by letting $A_2$ be always 1 under $\pi_\beta.$
    \item \textbf{Experiment 4:} 
    The NDE assumption holds, but the NUC assumption is violated. This corresponds to an MNAR missing data scenario.
    \item \textbf{Experiment 5:} 
    The NUC assumption holds, but the NDE assumption is violated. 

\end{itemize}

\noindent 
For full experiment configurations for the acquisition processes, please see Table     \ref{tab:details synthetic missingness}.

%\vspace{10 pt}
%\noindent 
\subsection{Training}
We used an impute-then-regress classifier \cite{le_morvan_whats_2021} with unconditional mean imputation and a logistic regression classifier for the classification task and trained it on the available and further randomly subsampled data (where $p(A_i^t = 1) =  0.5$). We tested random and fixed acquisition policies that acquire each costly feature with a 50\% or 100\% probability. Furthermore, we evaluated a proximal policy optimization (PPO) RL agent \cite{schulman_proximal_2017} which was trained on the semi-offline sampling distribution $p'$ using $\pi_{\alpha}$ as the semi-offline sampling policy, but without adjustment for the blocking of actions. 
The datasets were split into training set (for the training of the agent and the classifier), nuisance function training set, and test set, where the estimators were evaluated.  
The splitting of the dataset in a nuisance function training set and a test set is necessary due to the complexity of the used nuisance model functions classes \cite{kennedy_semiparametric_2022}. The resulting loss of efficiency may, however, be avoided using a cross-fitting approach \cite{kennedy_semiparametric_2022}.

\begingroup
\setlength{\tabcolsep}{10pt} % Default value: 6pt
\renewcommand{\arraystretch}{2} % Default value: 1
\begin{center}
\begin{table}
\begin{tabular}{| p{0.25\textwidth} | p{0.75\textwidth} |}
    \hline
    \multicolumn{2}{| c |}{\textbf{Data and environment}} \\
    \hline\hline
    Sample size $n_{D}$ & $100'000$ divided into $30\%$ training set (for agent and classifier), $30\%$ nuisance function training set, and $40\%$ test set. \\
    \hline
    Superfeatures & super$X_0$: $[X_0]$, super$X_1$: $[X_1]$, super$X_2$: $[X_2, X_3]$\\
    \hline
    Label & $Y^t \in \{0,1\}$ and for $t \leq T= 3$. %(class 0: 50\%,  class 1: 50\%)
    \\
    \hline
    Data generation parameters & 
    $\gamma_i = 0.2$ $\forall i$, $\sigma = 1$, $\zeta_1 = 1$, $\zeta_2 = 0.3$, $W = [1,1,2,2]/6$ \\
    \hline
    Feature acquisition cost & $c_{acq}=[0 , 1,  1]$ \\
    \hline
    Misclassification cost & $c_{mc} = 12$ \\
    \hline\hline
    \multicolumn{2}{| c |}{\textbf{Models}} \\
    \hline\hline
    Classifier & Logistic regression \\
    \hline
    Agents & \makecell*[l]{Random $50\%$, Fixed $100\%$, \\
   PPO (learning rate: 0.0001, number of layers: 2, \\
   \quad\quad\quad hidden layer neurons per layer: 64, \\
   \quad\quad\quad
   hidden layer activation function: tanh) } \\
    \hline
    Nuisance functions & \makecell*[l]{$\hat{\pi}_\beta$ (logistic regression),   \\
    $\hat{Q}_{Semi}$ ($\Xi = \emptyset$, learning rate: 0.001, number of layers: 2, 
    \\ \quad\quad\quad
    hidden layer neurons per layer: 16, \\ 
    \quad\quad\quad hidden layer activation function: ReLU) } \\
    \hline
\end{tabular}
\vspace{5pt}
    \caption{Full experiment details except for the acquisition process}
    \label{tab:details synthetic}
\end{table}
\end{center}

\begingroup
\setlength{\tabcolsep}{10pt} % Default value: 6pt
\renewcommand{\arraystretch}{2} % Default value: 1
\begin{center}
\begin{table}
\begin{tabular}{| p{0.25\textwidth} | p{0.75\textwidth} |}
   \hline
    \multicolumn{2}{| c |}{\textbf{Missingness mechanisms}} \\
    \hline\hline
    Exp 1 & \makecell*[l]{
        $p(A_0^t = 1) = 1.0,$ \\
        $p(A_1^t = 1) = \sigma(0.8 - 3.0 X_{0}^{t-1} + 0.02 X_{1}^{t-1}- 0.02 X_{2}^{t-1}),$ \\
        $p(A_2^t = 1) = \sigma(0.8 - 3.0 X_{0}^{t-1} + 0.02 X_{1}^{t-1}- 0.02 X_{2}^{t-1})$ \\
        Complete cases ratio: $p(A = \vec{1}) = 11.71\%$} \\
        \hline
    Exp 2 & \makecell*[l]{
        $p(A_0^t = 1) = 1.0,$ \\
        $p(A_1^t = 1) = 0.2,$ \\
        $p(A_2^t = 1) = 0.2$ \\
        Complete cases ratio: $p(A= \vec{1}) = 0.007\%$} \\
        \hline
    Exp 3 & \makecell*[l]{
        $p(A_0^t = 1) = 1.0,$ \\
        $p(A_1^t = 1) = 1.0,$ \\
        $p(A_2^t = 1) = \sigma(-0.5 - 2.0 X_{0}^{t-1} - 0.1 X_{1}^{t-1} - 0.1 X_{2}^{t-1})$ \\
        Complete cases ratio: $p(A = \vec{1}) = 7.09\%$} \\
        \hline
    Exp 4 & \makecell*[l]{
        $p(A_0^t = 1) = 1.0,$ \\
        $p(A_1^t = 1) = 1.0,$ \\
        $p(A_2^t = 1) = \sigma(-0.6 - 1.5 X_{(1),2}^{t-1} - 1.5 X_{(1),3}^{t-1})$ \\
        Complete cases ratio: $p(A = \vec{1}) = 9.63\%$} \\
      \hline
    Exp 5 & \makecell*[l]{
        $p(A_0^t = 1) = 1.0,$ \\
        $p(A_1^t = 1) = \sigma(0.8 - 0.2 X_{0}^{t-1} - 0.1 X_{1}^{t-1}+ 0.5 X_{2}^{t-1}),$ \\
        $p(A_2^t = 1) = \sigma(0.8 - 0.2 X_{0}^{t-1} - 0.1 X_{1}^{t-1}+ 0.5 X_{2}^{t-1})$ \\
        Complete cases ratio: $p(A = \vec{1}) = 11.71\%$} \\
    \hline
\end{tabular}
\vspace{5pt}
    \caption{Acquisition process details for all five experiments}
    \label{tab:details synthetic missingness}
\end{table}
\end{center}

\clearpage
\newpage

\bibliographystyle{plain}
\bibliography{references}  

%\bibliography{references}

\end{document}